%% file: main.tex
\documentclass[10pt]{wlscirep}
\usepackage[utf8]{inputenc}
\usepackage{float}
\usepackage[T1]{fontenc}
\usepackage[draft]{changes}

\usepackage{lineno} % 引入 lineno 宏包
\usepackage{csquotes}

\usepackage{amsmath}
\usepackage{xspace}
\usepackage{bm}
\usepackage{amssymb}
\usepackage[ruled,vlined]{algorithm2e}
\usepackage{todonotes}
\usepackage{amsthm}
\usepackage{multirow}
\usepackage[numbers,square]{natbib}
\usepackage{bbm}

\usepackage{titletoc}
\usepackage{pifont}
\newtheorem{proposition}{\textcolor{black}{Proposition}}
\newtheorem{lemma}{Lemma}
\newtheorem{theorem}{Theorem}
\newtheorem{definition}{Definition}
\newtheorem{corollary}{Corollary}
\newtheorem*{remark}{Remark}

\title{Shared Spatial Memory Through Predictive Coding}

\author[1, 4, $\star$]{Zhengru~Fang}
\author[1, $\star$]{Yu~Guo}
\author[2]{Yuang~Zhang}
\author[1]{Haonan~An}
\author[3]{Wenbo~Ding}
\author[1]{Yuguang~Fang}

\affil[1]{Department of Computer Science, City University of Hong Kong, Hong Kong}
\affil[2]{Department of Automation, Tsinghua University, Beijing, China}
\affil[3]{Shenzhen International Graduate School, Tsinghua University, Shenzhen, China}
\affil[$\star$]{indicates equal contribution}

\begin{abstract}
The mammalian hippocampus harbors specialized neural populations, often termed social place cells, that encode the locations of conspecifics and support collective navigation. Although such representations are believed to arise from the same 
predictive machinery supporting individual spatial cognition, a computational principle that can reproduce this capacity in 
artificial multi-agent systems remains unclear. Here, we introduce a multi-agent predictive coding framework that 
casts coordination as the minimization of mutual predictive uncertainty, namely each agent reduces uncertainty about other 
agents' future states under limited communication. This objective induces an information bottleneck that adaptively 
determines \emph{who} communicates \emph{when} and \emph{what} to transmit, while self-supervised motion prediction promotes the emergence of 
grid-cell-like spatial codes. Notably, the same learning dynamics give rise to specialized neural units that encode 
partners' locations, serving as artificial analogues of hippocampal social place cells (SPCs). In-silico lesion experiments 
further show that these social representations are causally required for coordination. On the Memory-Maze benchmark, our 
approach degrades gracefully from 73.5\% to 64.4\% success as bandwidth shrinks from 128 to 4 bits/step, whereas a 
full-broadcast baseline collapses from 67.6\% to 28.6\%. Together, these results identify predictive uncertainty 
minimization as a biologically grounded blueprint for bandwidth-efficient multi-agent coordination.
\end{abstract}

\begin{document}

\flushbottom
\maketitle

\thispagestyle{empty}

\input{sections/introduction}\label{sec1}

\input{sections/results}\label{sec2}

% Sample body text. Sample body text. Sample body text. Sample body text. Sample body text. Sample body text. Sample body text. Sample body text.

\input{sections/discussion}\label{sec3}

\input{sections/methods}\label{sec4}

% \section{Conclusion}\label{sec13}

% Conclusions may be used to restate your hypothesis or research question, restate your major findings, explain the relevance and the added value of your work, highlight any limitations of your study, describe future directions for research and recommendations. 

% In some disciplines use of Discussion or 'Conclusion' is interchangeable. It is not mandatory to use both. Please refer to Journal-level guidance for any specific requirements. 

\section*{Data Availability}\label{sec5}

The datasets used in this study are based on procedurally generated environments within the Memory-Maze benchmark, which has been made publicly available at \url{https://github.com/jurgisp/memory-maze}. All environment configurations, agent spawn locations, and goal placements used in our experiments are provided in the supplementary repository. Additional raw data (e.g., agent trajectories, BEV reconstructions, and communication transcripts) are available from the corresponding author upon reasonable request.

\section*{Code Availability}\label{sec6}

The full source code for the predictive coding framework, including training scripts, network architectures, and experiment configurations, will be released at \url{https://github.com/fangzr/SSM-PC} upon publication. To facilitate reproducibility, the repository also includes pretrained models, instructions for reproducing the Memory-Maze benchmark results, and detailed documentation. A permanent versioned archive will be deposited in Zenodo prior to final acceptance.

\section*{Author Contributions}
Z.F. and Y.G. conceived the predictive coding framework. Z.F., Y.G., Y.Z., and H.A. performed all experiments and data analysis. Y.G. and H.A. contributed to the conception of the HRL-ICM framework and data collection. W.D. and Y.F. supervised the entire project. W.D. and Y.F. wrote and reviewed the manuscript. All authors discussed the results and provided critical feedback on the manuscript. Y.F. provided funding.

\bibliographystyle{unsrtnat}
\bibliography{ref}

\newpage

\input{sections/suppl}

\end{document}

%% file: sections/introduction.tex
\section*{Introduction}\label{sec:introduction}
\begin{figure*}[t!]
\centering
\includegraphics[width=\textwidth]{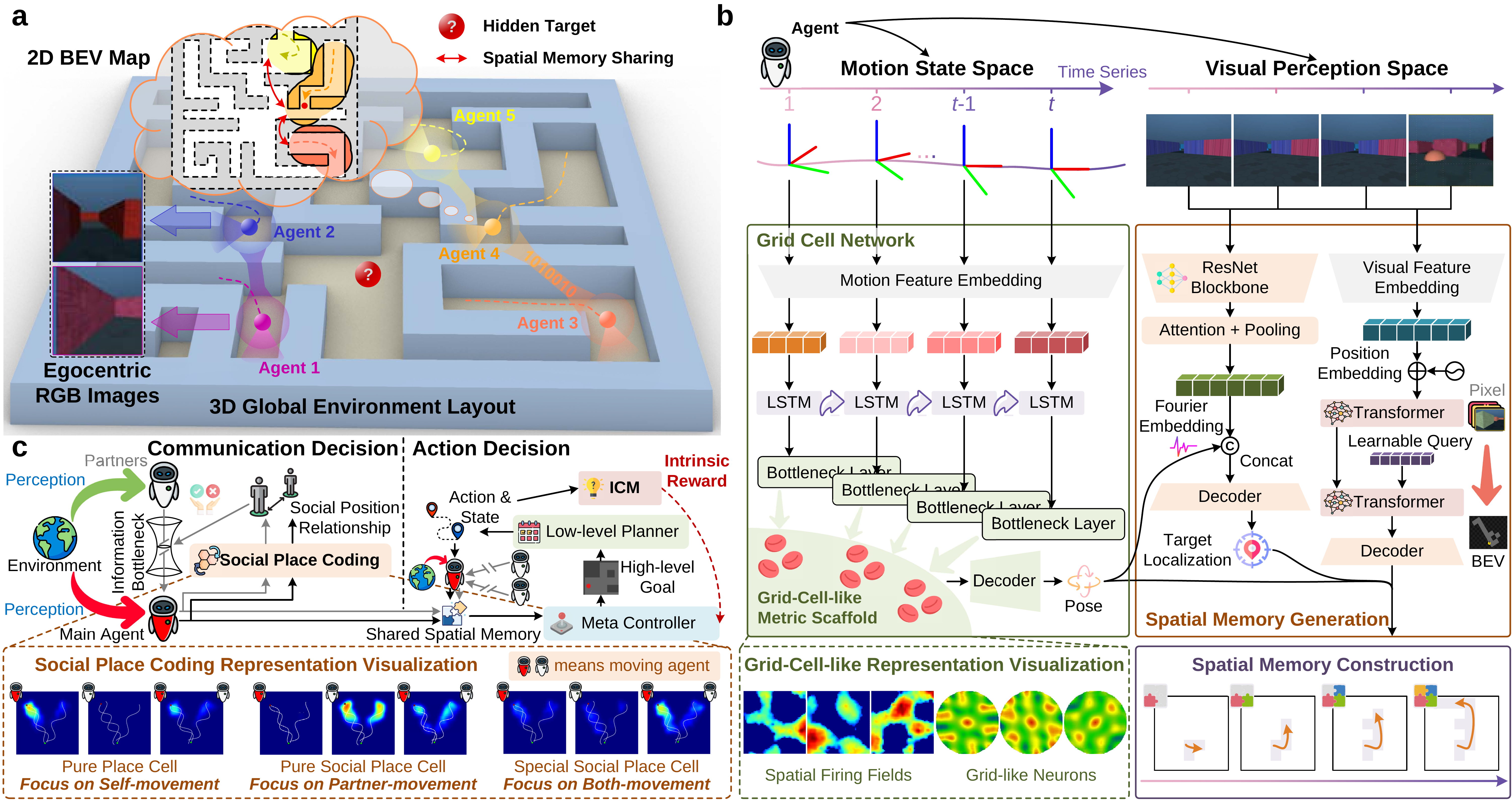}
\caption{\textbf{Overview of the predictive coding framework for sharing spatial memory.}
\textbf{a}, The multi-agent cooperative navigation task. Multiple agents, each with egocentric vision input, explore a 3D environment to find a hidden target. They coordinate by building and sharing a 2D bird's-eye-view (BEV) map via learned, emergent symbols. \textbf{b}, The single-agent spatial memory module. This module consists of two streams. The left stream, a Grid Cell Network, functions as a LSTM-based path integrator that processes the agent's motion state to estimate its pose. Its bottleneck layer spontaneously develops hexagonal activation patterns, mimicking biological grid cells. The right stream uses a Transformer-based network to generate a BEV map from visual inputs. The pose information from the path integrator is then used to accurately register the BEV map, constructing a coherent spatial memory of the maze layout. \textbf{c}, The agent's decision-making process via shared spatial memory. This process is divided into communication and action decisions. The communication decision is managed by an information bottleneck that adaptively adjusts data compression. Crucially, as this process must account for social peers, the network architecture gives rise to emergent social place cell-like activations. The action decision is handled by a hierarchical framework where a meta controller, trained with a multi-agent proximal policy optimization (MAPPO) algorithm and guided by an enhanced intrinsic curiosity module (ICM), directs a low-level planner to navigate toward regions that reduce the uncertainty of spatial memory.}
\label{fig:framework_overview}
\end{figure*}
Collective intelligence enables groups to solve complex tasks exceeding individual capabilities~\cite{soria2021predictive, yang2022autonomous, bai2025swarm, couzin2009collective, berdahl2018collective,duenez2023social}. From ant colonies to wolf packs, effective coordination depends on forming and acting on shared environmental understanding~\cite{basu2021orbitofrontal, wen2024one, farzanfar2023cognitive}. Specifically, in spatial navigation, animals align internal representations through sparse exchange of high-level cues, suggesting the presence of shared cognitive maps: distributed representations that encode resources, hazards, and conspecifics~\cite{whittington2022build,villafranca2021integrating,gornet2024automated,omer2018social,zhang2024multiplexed,Bray2018other_placecell}.

Neural algorithms underlying shared representations extend innate mechanisms supporting individual cognition. In mammals, the hippocampal-entorhinal system provides the substrate for cognitive maps, with place cells encoding locations and grid cells providing a metric scaffold~\cite{hafting2005microstructure,whittington2020tolman,ginosar2021locally,gardner2022toroidal,wagner2023entorhinal}. ``Social place cells''---neurons firing when a partner occupies particular locations---offer compelling evidence for neural substrates integrating self with peers~\cite{zhou2024vector,xu2022grid,liang2024distance}. These findings indicate that the brain possesses sophisticated machinery not only for building a model of its own world but also for representing the world of others, a prerequisite for any meaningful social coordination. However, a critical distinction exists between representing others within single brains versus coordinating across brains through severely bandwidth-constrained communication. This raises a fundamental question: What computational principle, grounded in how biological brains organize social cognition, would be sufficient for artificial agents to reproduce this capacity under severely bandwidth-constrained communication?

Existing multi-agent communication methods offer only partial answers because none treats communication as sending what a partner will need next~\cite{10438074}. Learned recipient-selection approaches~\cite{das2019tarmac,ding2020learning,du2021learning} determine \emph{who} receives a message, but do not optimize message content to reduce the receiver's \emph{future} uncertainty. Bandwidth-aware scheduling methods~\cite{9597491,sun2024dynamic,zhang2020succinct} determine \emph{when} to transmit based on local triggers or heuristics, without considering the receiver's predictive needs. Information-bottleneck approaches~\cite{wang2020learning,ding2023robust,Wang2020DecompComm} compress messages to remove redundancy, but optimize for immediate relevance rather than downstream predictive utility, and performance degrades sharply under tight bandwidth constraints~\cite{fang2025ton}. What all three families lack is an explicit model of the partner: without predicting what a partner will observe next, an agent cannot identify which information would most help that partner, which is precisely the strategy that makes biological coordination efficient.

We demonstrate that the optimal shared spatial memory emerges from minimizing mutual predictive uncertainty between agents. Our multi-agent predictive coding framework implements this through: \textbf{Level 1}, grid-cell-like spatial metrics for self-localization; \textbf{Level 2}, bandwidth-efficient communication and emergent social place cells; \textbf{Level 3}, hierarchical policies reducing collective uncertainty through coordinated exploration.

We ground our approach in predictive coding and information theory, constructing a shared spatial memory system operating under strict communication constraints. Agents with egocentric vision explore environments and communicate via emergent symbols (\textbf{Fig.~\ref{fig:framework_overview}a}). Each agent builds predictive bird's-eye-view (BEV) maps from observations (\textbf{Fig.~\ref{fig:framework_overview}b}), then exchanges insights through an information bottleneck minimizing partners' uncertainty and constructing efficient shared representations. Visualized in \textbf{Fig.~\ref{fig:framework_overview}c}, networks develop neurons tuned to self-location, partner location, or both, mirroring the mammalian hippocampus (\textbf{Fig.~\ref{fig:social_place_code}b}). The framework integrates two mechanisms. First, agents build predictive models generating BEV maps from vision, scaffolded by internal path integrators spontaneously producing grid-cell-like representations without supervision (\textbf{Fig.~\ref{fig:scaffold_and_mapping}}). Second, agents develop communication mechanisms (\textbf{Fig.~\ref{fig:communication}}), transmitting compressed symbols that reduce partners' uncertainty. This emerges from variational information bottleneck (VIB) objectives balancing communication cost and predictive utility~\cite{kawaguchi2023does,fang2025ton,taniguchi2024generative,taniguchi2023emergent}, embedded in hierarchical reinforcement learning with intrinsic curiosity (HRL-ICM) for strategic uncertainty reduction~\cite{yu2022surprising,pathak2017curiosity}.

Our HRL-ICM framework substantially exceeds baselines (\textit{No Communication}, \textit{Periodic Broadcast}, \textit{Full Broadcast}). On Memory-Maze~\cite{Pasukonis2023MemoryMaze}, it exhibits remarkable resilience: while \textit{Full Broadcast} collapses from 67.6\% to 28.6\% (58\% relative decline) when bandwidth is reduced from 128 to 4 bits/step, our method degrades minimally from 73.5\% to 64.4\% (12\% decline), maintaining superior performance under extreme constraints (\textbf{Fig.~\ref{fig:navigation}g}). Predicting partners' states drives emergence of specialized social cognition substrates. Neurons spontaneously segregate into populations selectively encoding teammates' locations—artificial social place cells mirroring mammalian hippocampus (\textbf{Fig.~\ref{fig:social_place_code}})~\cite{jaderberg2019human}. Causal analyses confirm these representations are functionally critical (\textbf{Fig.~\ref{fig:Comprehensive_performance}b}, \textbf{Supplementary Video 3}).

Unifying predictive coding and information bottleneck theory~\cite{whittington2022build,caucheteux2023evidence}, our work provides a theoretically grounded, biologically plausible basis for shared spatial memory. Our main contributions are as follows: \ding{182} A computational model showing that social place cell-like representations emerge as a necessary consequence of a social predictive coding objective, offering a mechanistic account of their biological origins. \ding{183} A multi-agent framework enabling efficient, semantically rich communication from first principles. \ding{184} Validation demonstrating state-of-the-art performance, scalability, and bandwidth robustness in embodied cooperative navigation. Together, these results establish predictive uncertainty minimization as a bridge between individual spatial cognition and collective intelligence.

%% file: sections/results.tex
\begin{figure*}[t!]
    \centering
    \includegraphics[width=\textwidth]{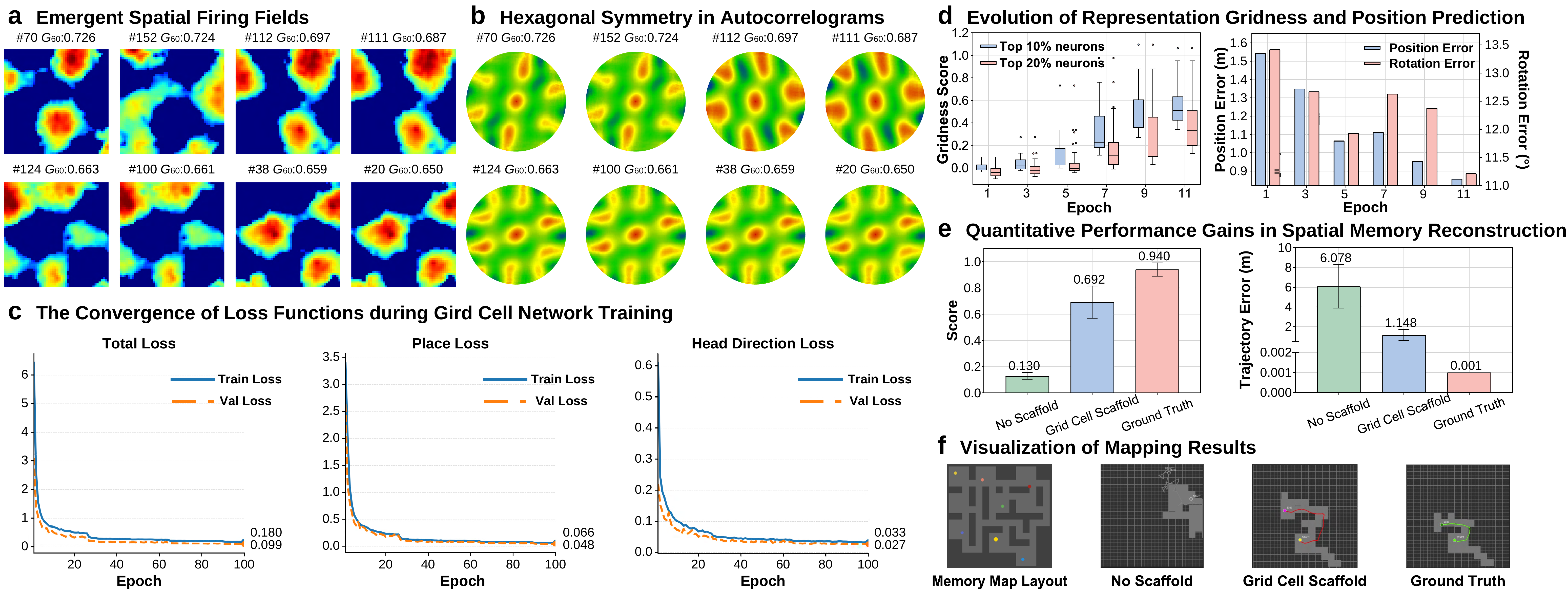}
    \caption{\textbf{Grid-cell-like representations enhance robust global BEV mapping.} 
    % \textbf{a}, The perception module integrates an LSTM-based path integrator with visual inputs to produce the BEV map $M$.
    \textbf{a}, 2D spatial firing rate maps of learned grid-like representations. 
    Each panel shows activity of a single unit in the path integrator LSTM network, 
    labeled by artificial neuron index (\#) and gridness score ($G_{60}$, range: -2 to +2; 
    higher values indicate stronger hexagonal symmetry). Heat maps represent 
    normalized firing rates across the 2D environment.
    \textbf{b}, Spatial autocorrelograms (SACs) of representative grid-like units 
    reveal hexagonal symmetry. Each circular plot shows the autocorrelation structure 
    of the corresponding unit's firing pattern from panel \textbf{a}, with the six-fold 
    rotational symmetry characteristic of biological grid cells. Unit labels and 
    $G_{60}$ scores match panel \textbf{a}.
    \textbf{c}, The convergence of loss functions during grid cell network training.
    \textbf{d}, As training proceeds, top unit gridness ($G_{60}$) increases while 
    path integration error decreases, demonstrating the co-emergence of structured 
    representations and predictive accuracy.
    \textbf{e}, Ablation study shows the full model achieves lower trajectory error and higher prediction confidence than the variant without the grid scaffold.
    \textbf{f}, Comparison of BEV map reconstructed by different configurations.}
    \label{fig:scaffold_and_mapping}
\end{figure*}

\section*{Results}\label{sec:results}

% Subsection 1: Integrated Predictive Visual Mapping with Path Integration Support
% This section now details how predictive visual mapping is achieved and enhanced by the integrated path integrator.
% Figure 2 will be a dense figure supporting this comprehensive subsection.
\subsection*{Grid-cell-like metric scaffold emerges spontaneously from self-supervised motion prediction}\label{subsec:grid_cell_emergence}

A fundamental prerequisite for our predictive coding framework is an agent's ability to form a stable internal model of its own state and surroundings. As outlined in \textbf{Fig.~\ref{fig:framework_overview}b}, this is achieved by solving two coupled prediction problems: predicting the visual appearance of the world (BEV mapping) and predicting the agent's own trajectory through it (path integration). While previous work has demonstrated grid-cell-like representations can emerge in navigation tasks~\cite{banino2018vector,whittington2020tolman}, these studies primarily focused on supervised learning with explicit spatial labels. Here, we demonstrate that self-supervised self-motion prediction alone—without any explicit spatial supervision—naturally develops a grid-cell-like spatial coding scheme that provides an essential metric scaffold for robust visual perception.

The path integration module uses an LSTM network to predict its future pose based solely on past velocity commands. Under this predictive constraint, the network's hidden units spontaneously form periodic spatial firing patterns. Many individual units developed highly structured spatial firing fields arranged in triangular lattices (\textbf{Fig.~\ref{fig:scaffold_and_mapping}a}), labeled by neuron index and gridness score ($G_{60}$)—a quantitative measure of hexagonal symmetry ranging from -2 to +2. To verify this grid-like organization quantitatively, we computed spatial autocorrelograms (SACs) for each neuron's activation map. This analysis revealed clear hexagonal symmetry in firing patterns (\textbf{Fig.~\ref{fig:scaffold_and_mapping}b}), the defining physiological signature of grid cells found in mammalian entorhinal cortex~\cite{hafting2005microstructure}.

Importantly, this grid-like representation is not merely a byproduct but represents a computational solution that the network actively converges to when optimizing for stable self-motion prediction. To demonstrate this, we tracked the co-evolution of this neural code and predictive performance throughout training. As the network learned, gridness scores of the most prominent units steadily increased, while path integration error concurrently and dramatically decreased (\textbf{Fig.~\ref{fig:scaffold_and_mapping}c,d} and \textbf{Supplementary Fig.~\ref{fig:S1}}). This tight correlation demonstrates that the formation of stable, periodic neural coding is the very mechanism through which the network masters motion prediction, enabling it to accurately integrate movements over long distances and maintain coherent beliefs about its location.

Beyond characterizing the emergence of grid-like codes, we establish their functional necessity for visual prediction. By conditioning BEV generation on the latent state provided by the path integrator, agents can correctly register and align transient visual inputs into coherent allocentric frames. To isolate and quantify this functional role, we performed ablation studies comparing the full predictive model against variants where the grid-cell scaffold was disabled. The scaffold's contribution is clear: the full model achieved significantly lower trajectory error and higher prediction confidence (\textbf{Fig.~\ref{fig:scaffold_and_mapping}e}). This improvement in self-localization directly translates into superior visual prediction. BEV maps from scaffolded models are more complete and geometrically accurate, while no-scaffold baselines produce fragmented and distorted maps (\textbf{Fig.~\ref{fig:scaffold_and_mapping}f}). Thus, these findings demonstrate that the emergence of grid-cell-like coding in our predictive model serves a necessary functional role for constructing stable spatial memory, providing a computational foundation for coherent world modeling.

\subsection*{Structured communication mechanism emerges from the social predictive objective}\label{subsec:social_emergent_communication}

\begin{figure*}[t!]
    \centering
    \includegraphics[width=\textwidth]{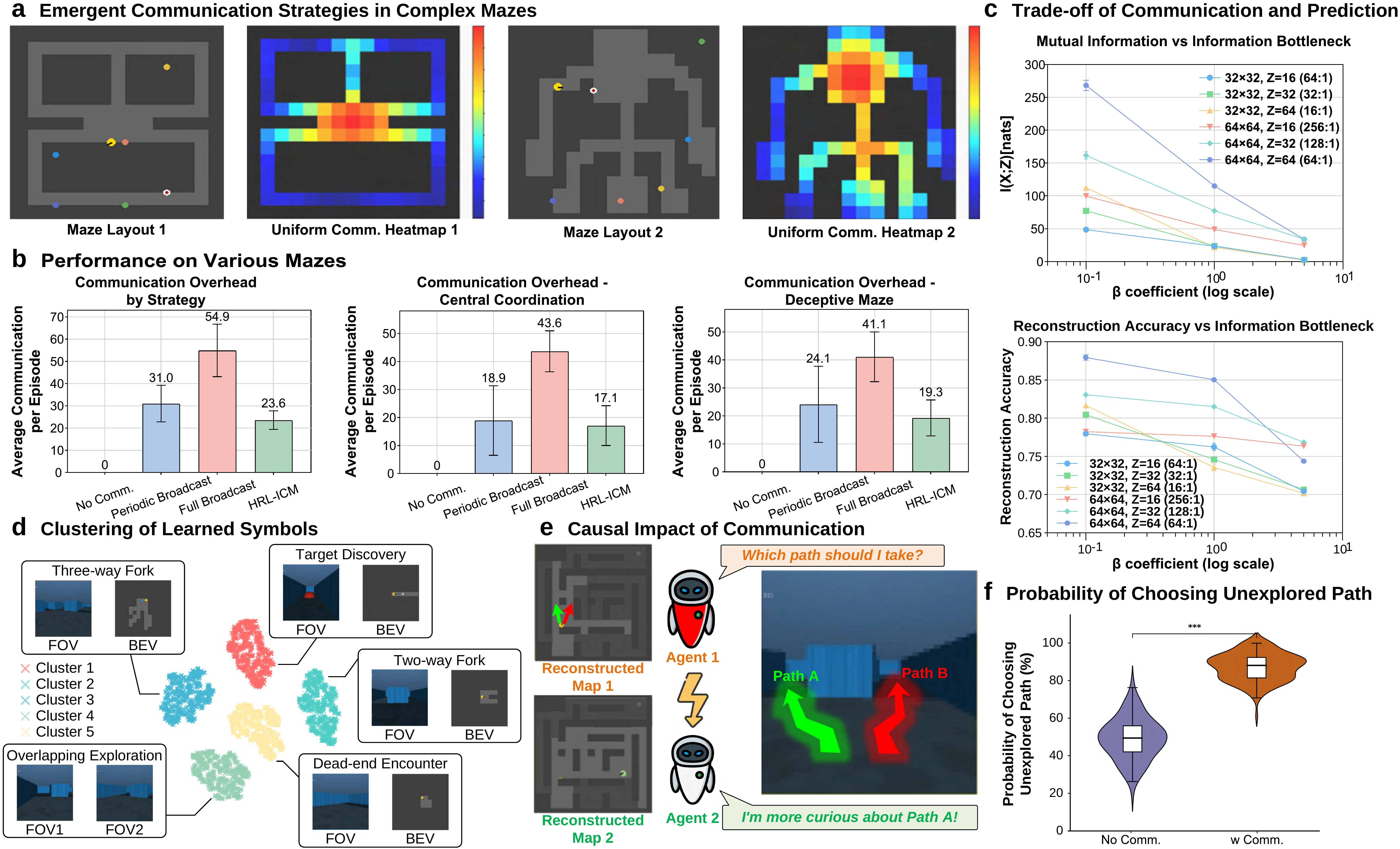}
\caption{\textbf{An efficient, structured, and intelligent communication mechanism emerges from a predictive objective.}
\textbf{a.} Intelligent communication strategies emerge, with message frequency (heatmaps) concentrated at critical decision points like coordination hubs or dead ends, demonstrating strategic triggering.
\textbf{b.} The emergent communication mechanism is highly bandwidth-efficient, consistently requiring the lowest communication overhead across diverse maze types when compared to full and periodic broadcast baselines.
\textbf{c.} The communication mechanism is theoretically controllable via the information bottleneck's $\beta$ coefficient, which enables a principled trade-off between message compression (compression ratio) and predictive utility (reconstruction accuracy).
\textbf{d.} An emergent symbolic vocabulary is grounded in strategic contexts. A t-SNE visualization reveals distinct symbol clusters corresponding to high-level navigational situations, such as encountering a ``Three-way Fork'' or discovering the ``Target''.
\textbf{e.} Communication causally influences decision-making. In a controlled scenario where one agent faces a choice between an unexplored path (A) and a known one (B), communication from its partner allows it to identify Path A as the more informative route.
\textbf{f.} The behavioral impact is statistically significant. Violin plots quantifying choices at two-way forks show a significantly higher probability of selecting the unexplored path with communication (``w Comm.'') compared to the no-communication baseline (``No Comm.''). Internal box plots indicate the median (center line) and interquartile range (white box); whiskers denote 1.5$\times$IQR. ($^{***}$ denotes $p < 0.001$; two-sided $t$-test, $n = 150$ per condition).
}
\label{fig:communication}
\end{figure*}

Having established that individual agents can build robust predictive models of their environment through grid-cell-like scaffolds, we next investigate how the framework extends this predictive principle to the social domain. A central challenge in multi-agent coordination is learning efficient communication under bandwidth constraints. 
% While existing approaches rely on handcrafted communication protocols~\cite{das2019tarmac,9597491} or learn to broadcast raw sensory data~\cite{10438074,wang2020learning}, these methods struggle with bandwidth limitations and fail to discover semantically meaningful communication strategies. 
Unlike frameworks exploring asymmetric, teacher–student knowledge 
transfer~\cite{wieczorek2024framework}, our setting addresses symmetric, 
peer-to-peer coordination. We compare our HRL-ICM framework against 
three baselines: \textit{No Communication}, where agents search 
independently; \textit{Periodic Broadcast}, where agents communicate 
at fixed intervals under bandwidth constraints; and \textit{Full 
Broadcast}, where agents exchange information without bandwidth 
limitations. Framing communication through the information bottleneck 
principle, we demonstrate that agents can learn to cooperate by developing a communication mechanism guided by a singular objective: transmit only information that maximally reduces a partner's future uncertainty. This simple requirement compels agents to collaboratively discover an emergent communication mechanism that is not only sparse and bandwidth-efficient, but also intelligent and semantically structured.

A key property of the emergent mechanism is its context-aware, task-oriented transmission strategy. Unlike prior work where agents communicate periodically or in response to predefined triggers, our agents learn to communicate strategically at particular moments and locations where a partner's internal model is most likely to be inaccurate. We visualize the spatial distribution of communication events and find that agents concentrate their transmissions at points of high predictive uncertainty (\textbf{Fig.~\ref{fig:communication}a}). For example, in mazes with a central hub, communication peaks in this critical coordination area where an agent's next move is most ambiguous. Conversely, in mazes with long, deceptive dead ends, agents learn to communicate most frequently from deep within these traps. This behavior directly solves the social prediction problem: a message from a dead end serves as a powerful ``prediction error'' signal to teammates, effectively correcting their erroneous implicit prediction that the path might be fruitful. This strategic triggering, consistently observed across diverse maze topologies (\textbf{Supplementary Fig.~\ref{fig:S3}}), demonstrates that agents learn an implicit model of their partners' beliefs, sharing information precisely when it can best resolve uncertainty and prevent predictive mistakes.

Building upon this context-aware triggering mechanism, the strategy of transmitting only the most surprising, uncertainty-reducing information naturally gives rise to a highly bandwidth-efficient communication mechanism. Across thousands of randomly generated layouts, our framework consistently operates with a fraction of the bandwidth required by periodic or full-broadcast approaches (\textbf{Fig.~\ref{fig:communication}b}). Furthermore, the mechanism is fundamentally structured and controllable through information bottleneck theory. The $\beta$ coefficient enables principled tuning between source compression and predictive utility. Increasing $\beta$ forces more compressed, abstract symbolic representations (\textbf{Fig.~\ref{fig:communication}c}), confirming that our framework provides theoretical control over communication effectiveness.

Most importantly, the uncertainty-driven compression scheme culminates in a meaningful symbolic vocabulary. A t-SNE visualization of the message latent space reveals distinct clusters corresponding to high-level strategic contexts: ``Three-way Fork'', ``Dead-end'', ``Target'' (\textbf{Fig.~\ref{fig:communication}d}). To test for a causal link between these symbols and agent behavior, we conduct controlled experiments where an agent at a fork could resolve its uncertainty only after receiving a message from its partner (\textbf{Fig.~\ref{fig:communication}e}). Quantifying choices across thousands of trials, we find that communication significantly guides agents to select the more informative, unexplored path over a known one (\textbf{Fig.~\ref{fig:communication}f}), providing compelling evidence that agents learn to transmit compressed, symbolic representations of prediction errors to collaboratively refine their shared world model and reduce collective uncertainty.

\subsection*{Predicting partner states forges an emergent social place code}

\begin{figure*}[t!]
    \centering
    \includegraphics[width=\textwidth]{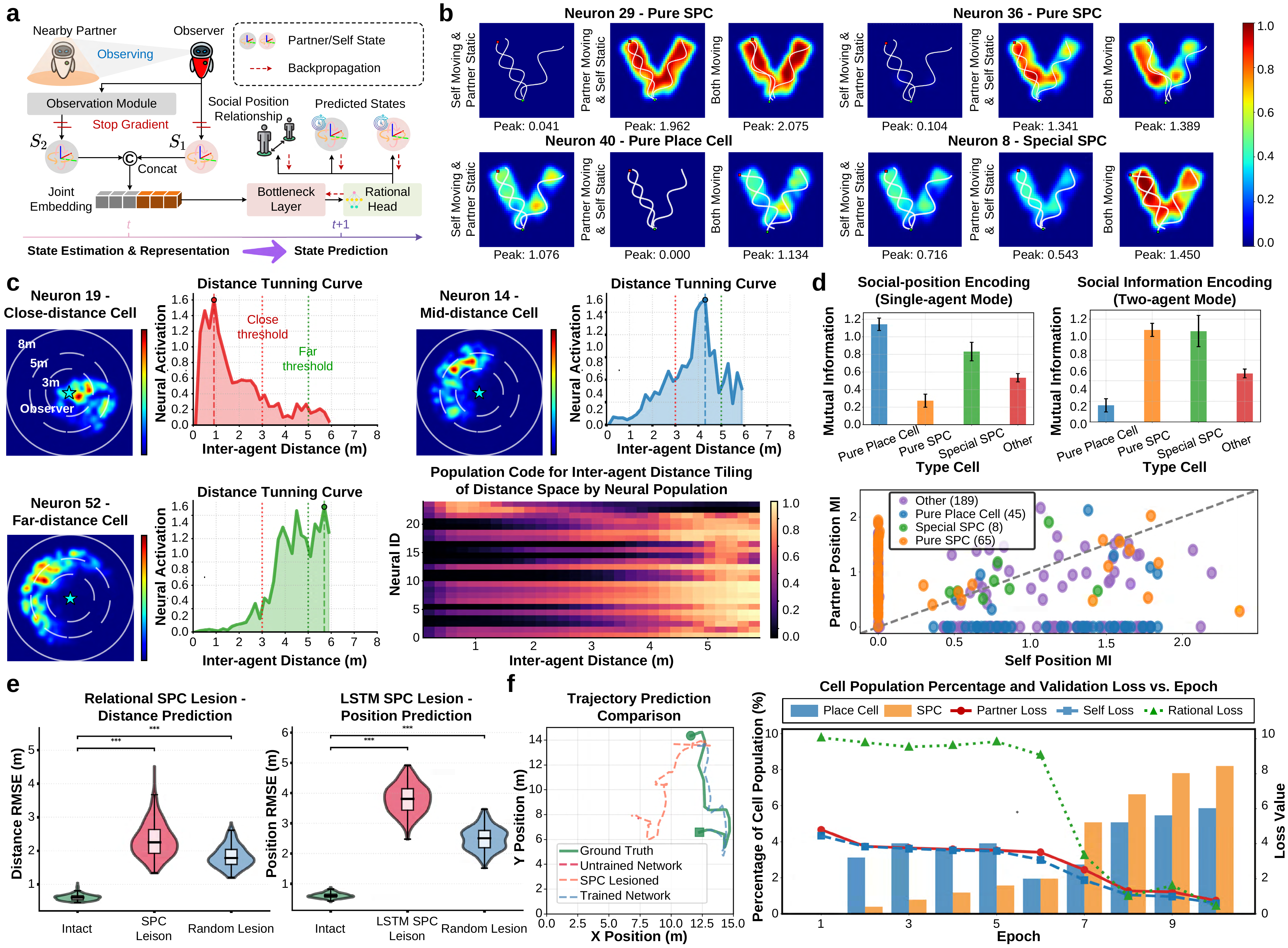}
    \caption{\textbf{Predictive learning forges a functionally specialized social place code.}
    \textbf{a}, Model architecture. Observer ($S_1$) and partner ($S_2$) states are processed through a bottleneck layer and relational head. The network is trained by back-propagating predictive error from self, partner, and social outputs.
    \textbf{b}, Functionally distinct neuron types. Three artificial neuron types in the relational head: Pure Place Cells (self-location encoding, Unit 40), Pure SPCs (partner-location encoding, Units 29, 36), and Special SPCs (mixed selectivity, Unit 8). Heatmaps show normalized firing rates. 
    \textbf{c}, Population code for inter-agent distance. Top panels show 2D maps and 1D tuning curves for neurons selective for close-, mid-, and far-distances. Bottom right, a heatmap of all distance-tuned neurons reveals a ``tiling'' of the distance space.
    \textbf{d}, Quantitative functional dissociation. Top, bar plots show high mutual information (MI) with self-position for Place Cells and with partner-position for SPCs. Bottom, a scatter plot of self MI vs. partner MI reveals specialized cell clusters.
    \textbf{e}, Causal necessity of SPCs demonstrated via in-silico lesioning. 
Left, targeted lesioning of relational-layer distance-selective SPCs 
($n=20$ units) specifically impairs distance prediction compared to 
intact model and random lesion controls (two-tailed unpaired $t$-test; 
$n=50$ test episodes per condition). Right, targeted lesioning of 
LSTM-layer SPCs impairs position prediction. Violin plots show full 
data distribution; *$p<0.05$, **$p<0.01$, ***$p<0.001$.
    \textbf{f}, Co-evolution of performance and specialization. Left, trajectory predictions are accurate for the trained network but poor for untrained or SPC-lesioned networks. Right, validation loss decreases over training epochs as the proportion of specialized Place Cells and SPCs increases.
    }
    \vspace{-1mm}
    \label{fig:social_place_code}
\end{figure*}

In the previous section, we demonstrated that agents learn efficient communication through predictive coding. However, a more fundamental question remains: What internal representations must agents develop to effectively model and predict the states of their partners? While prior work has identified place cells encoding self-location in artificial agents~\cite{banino2018vector}, the neural substrates for representing social partners in spatial contexts remain largely unexplored. Here, we show that a central challenge in collective intelligence—forming and maintaining representations of others—can be addressed through predictive learning.

We design a social processing module integrating self-state ($S_1$) with partner-state ($S_2$) to predict future outcomes for both agents (\textbf{Fig.~\ref{fig:social_place_code}a}). Single-unit analyses indicate spontaneous segregation into distinct, interpretable cell types (\textbf{Fig.~\ref{fig:social_place_code}b}). One substantial population behaved like classical place cells, exhibiting sharp and stable firing fields tuned exclusively to the agent's own location. These units are largely invariant to the partner's motion or position, thereby providing a stable allocentric representation of self. In contrast, a second major population fired as a function of the partner's location: these artificial social place cells (SPCs) showed strong spatial tuning to the partner's location within the observer's reference frame (\textbf{Fig.~\ref{fig:social_place_code}b}, Neurons 29 and 36), providing a substrate for tracking others. We also observe a population of mixed-selectivity units that conjunctively encode self- and partner-locations (\textbf{Fig.~\ref{fig:social_place_code}b}, Neuron 8). Representative galleries across four conditions (self-moving, partner-moving, both-moving, both-static) indicate that this specialization is expressed across units and task contingencies (\textbf{Supplementary Figs.~\ref{fig:S4}-\ref{fig:S6}}).

Beyond single-neuron effects, the specialized units form a population code for higher-order relational variables, most notably inter-agent distance. We identify subpopulations selectively tuned to near, mid-range, and far separations (\textbf{Fig.~\ref{fig:social_place_code}c}). Their graded tuning curves span the full range of separations encountered during exploration, yielding a tiling-like coverage of relational distance space (\textbf{Fig.~\ref{fig:social_place_code}c}, bottom right). This organization is consistent with the view that predictive objectives shape compact, task-relevant embeddings of spatial relations~\cite{zhou2024vector,liang2024distance}.

To move beyond descriptive characterization and establish functional relevance, we quantitatively dissociate these populations using mutual information (MI) between firing rate and spatial variables. Place cells carry high MI about self-position but negligible MI about partner-position, whereas SPCs show the inverse profile (\textbf{Fig.~\ref{fig:social_place_code}d}, top). A scatter of Self MI versus Partner MI reveals clearly separable clusters corresponding to self-tuned, partner-tuned, and mixed-selectivity units (\textbf{Fig.~\ref{fig:social_place_code}d}, bottom). These results indicate that specialization reflects a broader division of labor induced by the social predictive objective.

Moreover, to assess causal involvement, we perform in-silico lesions. Targeted lesioning of distance-tuned SPCs produced a marked, selective impairment in inter-agent distance prediction relative to pre-lesion performance and to size-matched random lesions (\textbf{Fig.~\ref{fig:social_place_code}e}). A broader lesion of the peer-processing module degrades general position prediction. These patterns indicate that SPCs are causally important for computing relational social geometry, whereas self-tuned units primarily support self-localization. Control ablations that preserve overall parameter count but disrupt SPC-selective pathways yielded similar deficits in distance estimation.

Finally, linking representation to learning dynamics, the fully trained network's ability to predict a partner's future trajectory depends on the integrity of SPCs: lesioning these units reduces predictive accuracy toward untrained levels (\textbf{Fig.~\ref{fig:social_place_code}f}, left). During training, validation loss decreases as the fraction of specialized units increases, with self-tuned and partner-tuned populations emerging in parallel (\textbf{Fig.~\ref{fig:social_place_code}f}, right). This co-evolution suggests that a functionally segregated social place code is not incidental but emerges as a principal mechanism by which the model solves partner-state prediction and coordination. Combined with prior evidence that predictive pressures can organize grid- and distance-related codes in artificial agents~\cite{zhou2024vector,xu2022grid,liang2024distance}, these findings support that minimizing mutual predictive uncertainty induces specialized representations for self and partners that jointly enable robust coordination under bandwidth and observability constraints.

% Subsection 4: 对应您原计划的 Figure 6 (现在可能是 Fig 5 或 Fig 6，取决于上面的合并情况)
\subsection*{Integrated framework achieves superior cooperative navigation performance}\label{subsec:integrated_performance}

\begin{figure*}[t!]
    \centering
    \includegraphics[width=\textwidth]{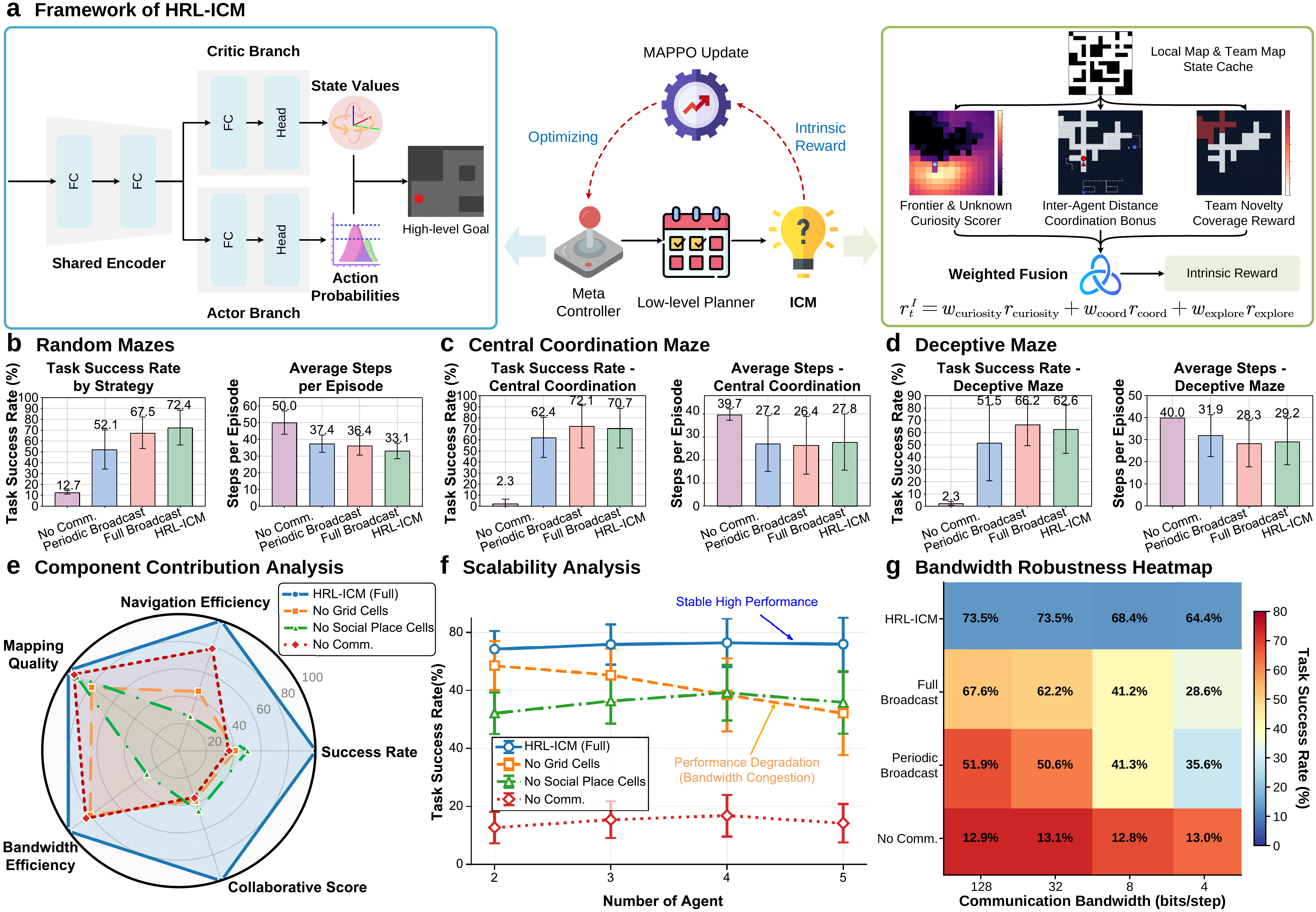}
    \caption{\textbf{HRL-ICM framework achieves superior and robust cooperative performance.}
    \textbf{a}, Architecture of the hierarchical reinforcement learning with intrinsic curiosity module (HRL-ICM). The ICM embodies the Level 3 predictive objective: it generates an intrinsic reward based on the agent's inability to predict the consequences of its actions. This ``prediction error'' signal guides the high-level Meta Controller to select goals that maximally reduce uncertainty, which are then executed by a Low-level Planner.
    \textbf{b}, Superior success rates and efficiency across 10,000 random mazes.
    \textbf{c}, High performance maintained in a central coordination maze.
    \textbf{d}, Robustness demonstrated in a deceptive maze with numerous dead ends.
    \textbf{e}, Ablation analysis confirms that each predictive component (grid-cells, social cells, communication) is critical for performance, with communication being indispensable.
    \textbf{f}, The framework scales effectively as agent count increases, outperforming baseline strategies that suffer from performance degradation.
    \textbf{g}, Exceptional bandwidth robustness is shown as our method's success rate degrades slightly when bandwidth shrinks, while the ``Full Broadcast'' baseline's performance collapses.
    }
    \label{fig:navigation}
\end{figure*}

The preceding sections demonstrate how a unified predictive objective forges specialized components for perception, communication, and social representation. Here, we assess whether these components synergistically combine to achieve superior cooperative performance. Unlike prior work that addresses individual components in isolation~\cite{banino2018vector,das2019tarmac}, we evaluate the fully integrated system across diverse navigation challenges. We assess the hierarchical reinforcement learning with intrinsic curiosity module (HRL-ICM) against strong baselines in Memory-Maze (\textbf{Fig.~\ref{fig:navigation}a}), implementing our three-level predictive hierarchy with policies guided by prediction-error-based intrinsic rewards.

We first assess our framework's effectiveness across a wide range of navigation problems. Over 10,000 procedurally generated mazes, HRL-ICM achieves highest success rate (72.4\%) with fewest steps (\textbf{Fig.~\ref{fig:navigation}b}). Moreover, performance holds in targeted scenarios: 72.0\% in central coordination mazes (\textbf{Fig.~\ref{fig:navigation}c}) and 66.0\% in deceptive mazes with long dead-ends (\textbf{Fig.~\ref{fig:navigation}d}), highlighting emergent communication's value for sharing negative information. Additionally, \textbf{Supplementary Video 1} directly visualizes the performance comparison in the memory maze.This resilience highlights the functional value of the emergent communication mechanism, which enables agents to share high-value negative information, a critical capability for efficient exploration.

{To quantify the contribution of each core component to this performance,} we conduct a systematic ablation study \mbox{(\textbf{Fig. \ref{fig:navigation}e})}. The complete framework (HRL-ICM Full) outperforms all ablated variants across every metric, including success rate, navigation efficiency, and mapping quality. The removal of any single component resulted in a significant and interpretable performance degradation. Disabling the communication channel led to a collapse in collaborative score and success rate, confirming that cooperation is indispensable. Excising the social place cell module specifically degrades the agents' collaborative score and efficiency, providing further causal evidence that the emergent social representations are functionally critical for effective coordination. Similarly, removing the grid-cell scaffold compromises mapping quality and navigation efficiency, underscoring the foundational importance of a stable internal metric. This analysis demonstrates that the framework's success is attributable to the synergistic integration of its predictive components. {The system's robustness is further evaluated against communication noise and environmental complexity (\textbf{Supplementary Fig. \ref{fig:S-noise}} and \textbf{Video 4}).}

{A critical test for any multi-agent system is scalability.} We evaluate how the framework performs as the number of agents increases from two to five (\textbf{Fig. \ref{fig:navigation}f}). Baseline strategies that rely on simple broadcasting (Full and Periodic) suffer a clear performance decline as more agents are added. {Their inability to handle the exponential growth in information flow leads to channel congestion and degraded coordination.} The learned, predictive communication mechanism from our framework mitigates this issue by ensuring that only the most vital, non-redundant information is transmitted. Besides, we test bandwidth robustness (\textbf{Fig.~\ref{fig:navigation}g}). Reducing bandwidth from 128 to 4 bits/step, HRL-ICM success degrades moderately (12\% relative decline) while Full Broadcast collapses (58\% decline). By communicating only essential prediction errors, agents sustain coordinated action in austere communication environments where conventional methods fail.

\subsection*{Framework analysis: Convergence, causality, and generalization}\label{subsec:interpretability}

\begin{figure*}[t!]
    \centering
    \includegraphics[width=\textwidth]{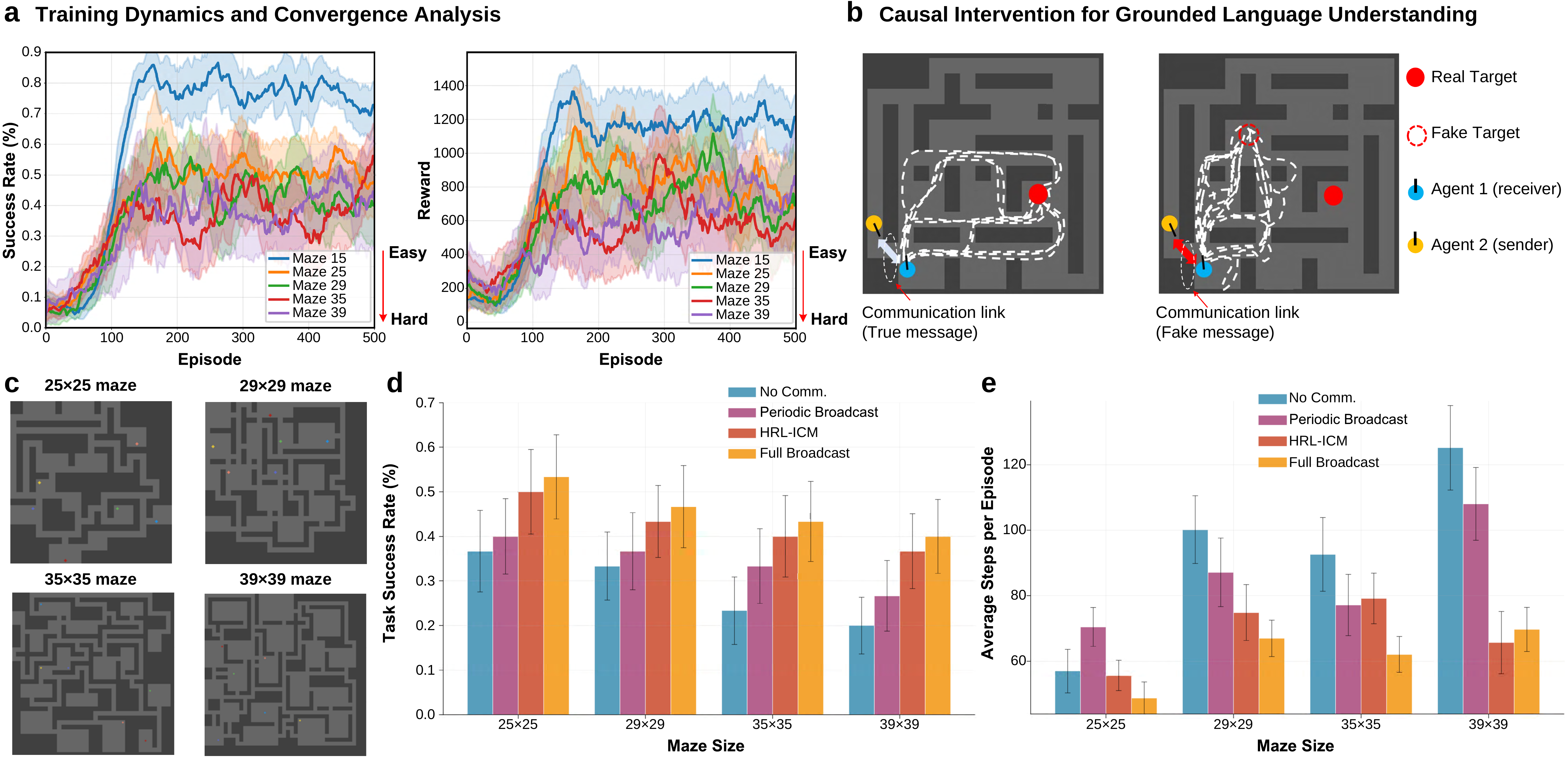}
\caption{
\textbf{Comprehensive performance analysis of the HRL-ICM framework, demonstrating stable learning, causal language understanding, and robust generalization.}
\textbf{a,} Training dynamics and convergence across mazes of varying difficulty. The plots show steady improvement and convergence for both average reward and task success rate, indicating stable end-to-end learning.
\textbf{b,} Causal intervention analysis reveals grounded language understanding. A receiver's trajectory is predictably altered by a manipulated ``fake'' message (right) versus a true message (left), confirming the emergent symbols causally drive behavior.
\textbf{c,} Example maze layouts of increasing complexity (from $25\times 25$ to $39\times 39$) used to test generalization.
\textbf{d,} Task success rate across different maze sizes. HRL-ICM consistently outperforms baselines, showing superior performance.
\textbf{e,} Average steps to completion. HRL-ICM demonstrates higher efficiency (fewer steps) than baselines, with its advantage growing in more complex environments.
}
\label{fig:Comprehensive_performance}
\end{figure*}
The results presented thus far establish the emergent properties and cooperative performance of our framework. To complete characterization, we address three critical questions: Does learning converge reliably? Do emergent symbols possess genuine semantic meaning? Can learned strategies generalize beyond training conditions?

First, we demonstrate convergence by tracking training dynamics across varying difficulty (\textbf{Fig.~\ref{fig:Comprehensive_performance}a}). From $15\times 15$ to $39\times 39$ configurations, success rates and rewards exhibit steady increases, converging to stable plateaus. This validates end-to-end training and confirms minimizing predictive uncertainty provides consistent learning signals for visuomotor control and inter-agent communication.

Second, to test semantic grounding, we designed causal intervention experiments (\textbf{Fig.~\ref{fig:Comprehensive_performance}b}). Senders who accurately identify recipients transmit messages that empower receivers. Replacing true messages with ``fake'' symbols previously associated with distractors causes receivers to predictably alter trajectories toward fake targets. This behavioral change demonstrates learned symbols function as causal drivers, not mere correlations, confirming agents develop shared, grounded understanding. Moreover, the causal intervention experiment is visualized in \textbf{Supplementary Video 3}.

Third, for generalization, we compare HRL-ICM against baselines across four complexity scales (\textbf{Fig.~\ref{fig:Comprehensive_performance}c-e}). Our framework achieves significantly higher success rates and requires fewer steps than baselines. Critically, as complexity increases, baseline performance degrades sharply while our framework exhibits graceful decline. This superior scalability stems from core principles: rather than brute-force information sharing, agents leverage learned internal models—grid-cell metric scaffolds and social place codes—forming robust shared memories. The emergent communication mechanism shares only critical, uncertainty-reducing information, an efficient strategy less susceptible to state-space combinatorial explosion. These findings validate our theoretically-grounded approach resolves bandwidth-coordination trade-offs, yielding generalizable, scalable multi-agent intelligence solutions.

%% file: sections/discussion.tex
\section*{Discussion}

In this article, we have shown that a single computational objective—the minimization of mutual predictive uncertainty—can serve as a first principle for the emergence of collective intelligence. We demonstrated that this unified predictive drive gives rise to a three-tiered hierarchy of phenomena: a stable, grid-cell-like spatial metric at the individual level; an efficient communication mechanism at the inter-actional level; and a specialized neural substrate for social cognition, analogous to hippocampal social place cells, at the representational level. Our findings suggest that complex neural architectures supporting social navigation may not be distinct adaptations but are the computational consequence of this predictive learning framework, which formulates social intelligence as an extension of individual cognition.

Our work is grounded in established principles of individual cognition. We first confirm that a stable personal spatial representation, scaffolded by an emergent, grid-cell-like metric critical for robust navigation~\cite{banino2018vector}, is a necessary foundational component (\textbf{Fig.~\ref{fig:scaffold_and_mapping}}). While architectures such as the Tolman–Eichenbaum Machine provide a unifying account of how an \textit{individual} agent learns and generalizes spatial and relational knowledge~\cite{whittington2020tolman}, our work addresses the subsequent challenge: how multiple agents, each possessing such internal models, achieve collective intelligence. We posit that the drive to minimize mutual uncertainty provides a computational bridge from individual cognition to shared understanding.

From this foundation, we address emergent communication. Unlike frameworks exploring asymmetric, teacher–student knowledge transfer~\cite{wieczorek2024framework}, our setting resolves symmetric, peer-to-peer coordination. By framing communication through the information bottleneck principle, we show that agents can autonomously learn a communication mechanism that balances communication cost and predictive utility (\textbf{Fig.~\ref{fig:communication}c}), surpassing standard multi-agent reinforcement learning baselines, such as PPO and Dreamer, considered here~\cite{medany2024model}. A key consequence of this social predictive objective is the spontaneous formation of a functionally specialized neural substrate for social cognition. Inspired by the discovery of hippocampal social place cells in bats~\cite{omer2018social}, we provide a computational account consistent with their emergence under a social predictive objective. Whereas biological studies observe these neurons, our framework posits that the imperative to predict a partner’s future state encourages the development of functionally specialized units tuned to the location of others. Our \textit{in-silico} lesion experiments (\textbf{Fig.~\ref{fig:social_place_code}e, f}) indicate that these emergent ``social place cells'' are causally important for social prediction, offering a functional interpretation for the ``shared neural subspaces'' reported in interacting systems~\cite{zhang2025inter}.

The same principle mirrors the ``next-token prediction'' objective that underpins large language models (LLMs): while LLMs build a world model by predicting sequences of texts, our agents build a shared spatial model by predicting sequences of each other's sensory states. This parallel suggests that predictive learning is a general mechanism for constructing both individual and shared world models. The connection between understanding and effective compression, highlighted in recent work on lossless data compression~\cite{li2025lossless}, provides a theoretical link to the communication efficiency observed in our system. For multi-robot systems, this work offers a paradigm for designing communication-efficient swarms that learn to coordinate from first principles~\cite{medany2024model}.

While our framework establishes these principles, it opens several interesting avenues. The current model utilizes a discrete communication channel; future work could explore the emergence of more continuous or compositional communication structures~\cite{wieczorek2024framework}. Deploying our model-based learning approach on physical multi-robot systems to bridge the simulation-to-reality gap is a key next step~\cite{medany2024model}. More broadly, the phenomena of shared memory and social representation may arise not from intricate, pre-programmed rules, but from a single computational objective: the drive to predict the world of another.

%% file: sections/methods.tex
% \section*{Shared Spatial Memory Through Predictive Coding}\label{sec_methods}
\section*{Methods}\label{sec_methods}

\subsection*{Memory-Maze benchmark}\label{meth_benchmark}
All experiments are conducted within the Memory-Maze benchmark, a simulation environment specifically designed to rigorously evaluate agents' long-term spatial memory and cooperative capabilities under partial observability~\cite{Pasukonis2023MemoryMaze}. The framework procedurally generates three-dimensional mazes with randomized layouts for each episode, which prevents task-specific overfitting and ensures that learned policies generalize across a wide distribution of environments. To systematically assess the scalability and robustness of our proposed model, we utilize a range of maze complexities, configuring the environment with progressively larger layouts, specifically $15 \times 15$, $25 \times 25$, $29 \times 29$, $35 \times 35$, and $39 \times 39$ grids. This setup allows for a thorough investigation of model performance as the state space and the demand on spatial memory and communication grow significantly.

Each agent in the multi-agent system operates based on egocentric sensory inputs and is subject to realistic physical and communicational constraints. An agent's perception is primarily driven by a forward-facing RGB camera providing a 75-degree field of view (FOV), which serves as its visual input. In addition to vision, each agent has access to its own proprioceptive information, namely its linear and angular velocities. The action space is discrete, consisting of move-forward, turn-left, turn-right, and stay, enabling navigation through the maze corridors. A critical component of our experimental design is the constraint on inter-agent communication. Agents can only exchange information when they are within a pre-defined communication range, and the channel is subject to a configurable bandwidth bottleneck (e.g., 4 to 128 bits/step). This limitation mirrors real-world robotic applications and creates a strong selective pressure for the emergence of an efficient and targeted communication mechanism, which is a central focus of our work. The collective objective for the team of $N$ agents (where $N$ is typically 2-4 in our experiments) is to collaboratively explore the unknown maze to locate a single, hidden target.

\subsection*{Unified framework for predictive coding}\label{Unified_framework}
Throughout this paper, we use neuroscience-inspired terminology to describe artificial network components that exhibit functional properties analogous to biological neurons. Specifically, ``place cells'' refer to network units encoding spatial position with Gaussian receptive fields, ``head-direction cells'' encode heading orientation with von Mises tuning, and ``grid cells'' exhibit periodic hexagonal firing patterns. ``Social place cells'' (SPCs) are distinct units that encode partner locations rather than self-locations. These terms facilitate comparison with biological findings while denoting artificial neural network parameters that represent position and orientation information. Detailed mathematical definitions are provided in Supplementary Methods.

The cornerstone of our framework is the principle that intelligent agents build and share world models by continuously minimizing prediction error. We formalize the challenge of coordinated navigation as a multi-level prediction problem, where each agent's objective is to build a generative model that predicts its own sensory inputs, the states of its partners, and the global state of the environment. This unified predictive objective is realized through three synergistic mechanisms: \ding{182} an individual predictive model for robust perception and self-localization, \ding{183} a social predictive coding model for emergent communication, and \ding{184} a hierarchical policy guided by predictive uncertainty for strategic exploration.

\subsubsection*{Level 1: Individual predictive model for perception and self-localization}
An agent must first form a coherent internal model of its own state and surroundings to serve as a foundation for any higher-level reasoning or social interaction. This is achieved not through a single monolithic network, but by the synergistic interplay of two specialized yet deeply integrated sub-systems that solve simultaneous, coupled prediction tasks. Meanwhile, these components function analogously to a visual simultaneous localization and mapping (SLAM) system, where the agent concurrently builds a map of its environment (perception) while predicting its location within it (localization). This dual predictive process is essential because each task validates the other: a stable location estimate prevents the map from becoming distorted, while a coherent map provides the landmarks necessary for accurate localization. This reciprocal relationship enables the construction of a reliable individual world model, which is the prerequisite for forming a shared spatial memory.

\paragraph{Visual predictive coding for BEV mapping.}
As shown in \textbf{Fig. \ref{fig:framework_overview}b}, the primary perceptual task is framed as a visual prediction problem that addresses the challenge of inferring a stable world representation from a fleeting, ambiguous sensory stream. Specifically, the agent must learn a generative model of the physical world's local geometry and appearance. The agent does so by learning to translate a high-dimensional, egocentric image, \(O_{\text{Ego},t}\in\mathbb{R}^{3\times H_\text{in}\times W_\text{in}}\) at time $t$, into a structured, allocentric BEV map, \(\hat{O}_{\text{BEV},t}\in\mathbb{R}^{4\times H_\text{out}\times W_\text{out}}\). This is a fundamentally ill-posed inverse problem, and the network is trained end-to-end to minimize the prediction error between its generated map and the ground truth. This error is quantified by a composite loss function, \(\mathcal{L}_\text{BEV}\), which holistically evaluates the quality of the prediction across multiple modalities. Each component of this loss can be interpreted as imposing a different physical prior on the generative model, guiding it towards physically plausible solutions:
\begin{equation}
\begin{aligned}
\mathcal{L}_\text{BEV} &= w_\text{occ}\mathcal{L}_\text{occ} + w_\text{rgb}\mathcal{L}_\text{rgb} + w_\text{smooth}\mathcal{L}_\text{smooth},
\end{aligned}
\end{equation}
where \(w_\text{occ}\), \(w_\text{rgb}\), and \(w_\text{smooth}\) are weighting hyper-parameters. The occupancy loss, \(\mathcal{L}_\text{occ}\), is a binary cross-entropy term on the predicted alpha channel (\(\hat{\alpha}\)) that forces the network to accurately predict the binary state of the world (i.e., navigable or occupied). The appearance loss, \(\mathcal{L}_\text{rgb}\), which is a masked mean-squared error term, compels the model to predict the correct target color within navigable regions. Finally, the smoothness loss, \(\mathcal{L}_\text{smooth}\), regularizes the predictive model by penalizing sharp spatial gradients, incorporating a prior that physical environments are generally continuous. To solve this complex, cross-view prediction task, we employ a sophisticated Transformer-based encoder-decoder architecture. The detailed layer-by-layer specification of this architecture is provided in \textbf{Supplementary Tables \ref{tab:bev_encoder}} and \textbf{\ref{tab:bev_decoder}}, which leverages a Transformer's self-attention mechanism to overcome perspective distortion inherent in ground-level views. The encoder, \(E_\text{BEV}\), uses a ResNet-18 backbone to extract spatial features. A key architectural innovation lies in treating the image's vertical scanlines as a sequence. This allows the Transformer's self-attention mechanism to disambiguate visual features by considering the global context along the vertical axis of the image, which is crucial for the predictive model to accurately infer depth and overcome perspective distortion. The decoder, \(D_\text{BEV}\), then projects the resulting latent code into the final BEV map. The mathematical formulation and weighting of each loss component are detailed in \textbf{Supplementary Method \ref{sec: Spatial Memory Generation Module}}, which define the precise objective for learning photometrically and geometrically plausible BEV maps.

\paragraph{Predictive path integration for self-localization.}
As shown on the left side of \textbf{Fig. \ref{fig:framework_overview}b}, the module's objective is to minimize the prediction error \(\mathcal{L}_\text{path}\), a heavily regularized, weighted sum of the Kullback-Leibler (KL) divergence between the predicted and target cell activations related to the agent's pose:
\begin{equation}
\begin{aligned}
\mathcal{L}_\text{path} &= \mathcal{L}_\text{KL} + w_{\text{init}}\mathcal{L}_\text{init} + w_{\text{cont}}\mathcal{L}_\text{cont}.
\end{aligned}
\end{equation}
The design of this loss function is critical for maintaining a stable predictive process over long trajectories. The primary term, \(\mathcal{L}_\text{KL}\), applies a significantly higher weight to the prediction error in the initial trajectory frames, forcing the model to establish an accurate pose estimate early to prevent error accumulation. This is supplemented by an initial consistency term, \(\mathcal{L}_\text{init}\), and a continuity regularizer, \(\mathcal{L}_\text{cont}\). This structured learning pressure compels the LSTM's hidden units to discover the most efficient neural code for representing Euclidean space under translational motion. Remarkably, the optimal representation that emerges spontaneously to best solve this temporal self-prediction task is a grid-cell-like code, exhibiting the hexagonal symmetry found in the mammalian entorhinal cortex. This provides strong evidence that such complex, biologically plausible neural structures can arise not from explicit design, but as a fundamental and convergent solution to the problem of continuous self-prediction in a spatial environment. This emergent metric scaffold provides the essential, stable foundation upon which the local visual predictions can be reliably integrated, forming a coherent and robust individual spatial memory. Theoretically, this learning process allows the LSTM to function as an amortized Bayesian filter, effectively performing probabilistic state estimation, as detailed in \textbf{Supplementary Method \ref{sec: Biologically-Inspired Spatial Representation Networks}}. The architecture and training details are provided in \textbf{Supplementary Tables \ref{tab:grid_cell_network}} and \textbf{\ref{tab:grid_cell_training}}. 

\paragraph{Theoretical justification for predictive path integration.}

Here we provide a theoretical justification for how the self-prediction objective compels neural networks to spontaneously develop structured, grid-like spatial representations. The key insight is that the geometry of path integration—specifically, the requirement to maintain a consistent spatial metric under continuous motion—fundamentally constrains the form of efficient neural codes. We demonstrate that hexagonal grid patterns emerge not from architectural engineering, but as the unique solution to satisfying geometric consistency under representational efficiency constraints.

\textit{Probabilistic formulation.} We formalize path integration as a problem of sequential state inference. The agent's pose at time $t$ corresponds to the state vector $s_t = (\mathbf{r}_t, \theta_t) \in \mathbb{R}^2 \times \mathbb{S}^1$, where $\mathbf{r}_t = (r_{x,t}, r_{y,t})^\top$ denotes the position and $\theta_t$ denotes the heading angle. The agent's kinematics follow the dynamics:
\begin{align}
\theta_{t+1} &= \theta_t + \omega_t \Delta t \pmod{2\pi}, \nonumber \\
\mathbf{r}_{t+1} &= \mathbf{r}_t + R(\theta_t) \mathbf{v}_t \Delta t, \nonumber
\end{align}
where $u_t = (\mathbf{v}_t, \omega_t)$ is the control input comprising egocentric translational and angular velocities, and $R(\theta_t)$ is the rotation matrix (\textbf{Supplementary Definition~\ref{def:rotation-matrix}}) that transforms egocentric motor commands to allocentric position updates.

The path integrator network receives no direct supervision on the latent pose $s_t$ during trajectory execution. Instead, its sole learning signal derives from predicting the activity of virtual place and head-direction cells (artificial network units that encode spatial position and heading orientation, respectively, using biologically-inspired activation patterns; see \textbf{Supplementary Method~\ref{sec: Biologically-Inspired Spatial Representation Networks}} for detailed definitions) at the final timestep. These spatial cells exhibit tuning curves with log-potentials $\phi_\ell(s)$: place cells use Gaussian receptive fields $\phi_i(\mathbf{r}) = -\|\mathbf{r} - \boldsymbol\mu_i\|^2/(2\sigma_i^2)$ while head-direction cells employ von Mises tuning $\phi_j(\theta) = \kappa_j \cos(\theta - \mu_j)$. The probability of observing cell $\ell$ active in pose $s_t$ follows a softmax distribution over these potentials:
\begin{equation}
p(y_t = \ell \mid s_t) = \frac{\exp\{\phi_\ell(s_t)\}}{\sum_{m=1}^C \exp\{\phi_m(s_t)\}},
\nonumber
\end{equation}
where $C$ denotes the total number of cells and $y_t$ is the random variable for the active cell identity. The LSTM path integrator is trained to map a sequence of control inputs $\{u_t\}_{t=1}^T$ and an initial observation $y_0$ to a predictive distribution $\hat{p}(\cdot \mid y_0, \{u_t\})$ over cell activations at the final time $T$. The training objective minimizes the expected Kullback--Leibler (KL) divergence between the true activation distribution and the network's prediction:
\begin{equation}
\mathcal{L}(\boldsymbol\varTheta ) = \mathbb{E}_{y_0, \{u_t\}}
    \left[ \mathrm{KL}\big(p(\cdot \mid s_T) \,\|\, \hat{p}(\cdot \mid y_0, \{u_t\})\big) \right],
\end{equation}
where $y_0$ denotes the initial spatial cell observation encoding starting pose $s_0$,
${u_t}$ represents the sequence of motor commands, $s_T$ is the true final pose after path integration, $p(\cdot| s_T)$ is the target sensory distribution at the final pose, $\hat{p}(\cdot| y_0, {u_t})$ is the LSTM's predicted distribution, and $\boldsymbol\varTheta$ denotes the network parameters. Minimizing this prediction error over diverse trajectories compels the network to internalize the geometric structure of spatial navigation.

\textit{Architectural design for temporal integration.} The recurrent structure of an LSTM is particularly well-suited for path integration because its gated cell state mechanism enables stable accumulation of incremental motion signals over extended time horizons. The LSTM's cell state $\mathbf{c}_t$ serves as a continuous memory substrate that is updated at each timestep by integrating velocity inputs through its gating architecture:
\begin{equation}
\mathbf{c}_t = \mathbf{f}_t \odot \mathbf{c}_{t-1} + \mathbf{i}_t \odot \tilde{\mathbf{c}}_t,
\nonumber
\end{equation}
where $\mathbf{f}_t$ and $\mathbf{i}_t$ are the forget and input gates, and $\tilde{\mathbf{c}}_t$ is the candidate update. Critically, we impose a representational bottleneck by projecting the LSTM hidden state through a low-dimensional layer before prediction (\textbf{Supplementary Table~\ref{tab:grid_cell_network}}). This architectural choice is essential: the bottleneck forces the network to discover a compressed, information-efficient encoding of pose that can support accurate long-term prediction. It is this efficiency pressure, combined with the predictive objective, that catalyzes the emergence of structured spatial codes.

\textit{Emergence of hexagonal representations from geometric constraints.} Finally, we justify why hexagonal activity patterns spontaneously emerge as the optimal solution. The path integration task fundamentally requires representing accumulated two-dimensional displacements in a compressed latent space with limited dimensionality. As established in Supplementary Information through \textbf{Supplementary Proposition~\ref{prop:equivariance}} and \textbf{Theorem~\ref{thm:rotation-structure}}, any representation that faithfully performs path integration must satisfy translation equivariance: a physical displacement $\Delta \mathbf{r}$ must correspond to a consistent linear transformation in the latent space, independent of starting location. This geometric constraint, combined with the information efficiency requirement under the bottleneck architecture, necessitates a periodic encoding. \textbf{Supplementary Theorem~\ref{thm:minimal-isotropy}} proves that the minimal isotropic representation in two dimensions requires exactly three frequency directions uniformly distributed at $120^\circ$ angular separation, uniquely determined by the first-order isotropy (zero mean) and second-order isotropy (identity-proportional metric). As demonstrated in \textbf{Supplementary Corollary~\ref{cor:hexagonal}}, these three directions generate effective six-fold symmetry due to cosine function parity: their superposition creates an interference pattern whose peak activations form a hexagonal lattice with $60^\circ$ rotational symmetry. Consequently, minimizing prediction error under equivariance and efficiency constraints causes LSTM hidden units to self-organize into grid-cell-like hexagonal patterns, providing the consistent spatial metric needed for path integration and forming the computational foundation for downstream prediction and planning.

\subsubsection*{Level 2: Social predictive coding for emergent communication}
Having established a robust individual predictive model, the next challenge is to synchronize these models across agents. Central to this is a neural substrate designed to form a unified representation of the multi-agent system, which serves as the foundation for communication. We employ a dual-stream LSTM network that concurrently processes egocentric motion inputs (linear and angular velocities) from both the self-agent and its partner. The network is trained under a multi-faceted predictive objective: it must simultaneously predict its own future location, its partner's future location, and, critically, the future Euclidean distance between them via a dedicated relational regression head. This compound predictive pressure compels the network's shared latent representation to functionally specialize. 
This process gives rise to distinct neural populations, including units selectively tuned to the partner's location—an artificial analogue of social place cells (SPCs; distinct from the self-position-encoding place cells described above, these units specifically encode partner locations within the observer's reference frame), as detailed in our results (\textbf{Fig.~\ref{fig:social_place_code}}). This emergent social representation provides the rich, unified state, $S_{i,t}$, that informs the communication mechanism. Furthermore, the explicit distance prediction serves as a critical gating mechanism to determine \textbf{who and when} communication is feasible, as detailed in \textbf{Supplementary Method~\ref{sec: Social Place Cell}} (Social place cell module architecture in \textbf{Table~\ref{tab:spc_arch}}).

This architecture sets the stage for extending the predictive coding principle to the multi-agent domain, where we reformulate communication not as mere data transfer, but as a mechanism for collaborative prediction. The question shifts from ``What information should I send?'' to ``What piece of my knowledge will maximally reduce my partner's future uncertainty?'' This is the essence of social predictive coding: agents learn to share only information that is novel and decision-relevant from the receiver's perspective. To formalize this, we employ the information bottleneck (IB) principle. 
The goal is to learn a stochastic encoder $p(m_{i,t}\mid S_{i,t})$ that 
maps sender agent $i$'s state $S_{i,t}$ to a compressed message $m_{i,t}$. 
The optimal encoding must balance two competing objectives: maximizing the 
message's predictive utility for the receiver while minimizing communication 
cost. Following the rate-distortion framework, we formulate this as minimizing 
a loss function that balances distortion (predictive loss) and rate 
(compression cost):
\begin{equation}
\label{eq:ib_objective_revised}
\mathcal{L}_{\text{IB}} = \underbrace{-I(m_{i,t}; O_{j,t+1} \mid S_{j,t})}_{\text{Distortion}} 
+ \beta \underbrace{I(S_{i,t}; m_{i,t})}_{\text{Rate}},
\end{equation}
where $I(\cdot;\cdot)$ denotes mutual information, which is a measure of the mutual dependence between the two variables. $O_{j,t+1}$ is the future 
observation of the receiving agent $j$, $S_{j,t}$ is the receiver's current 
state (serving as side information), and $\beta > 0$ is a hyper-parameter 
balancing the trade-off. The \textbf{Distortion} term quantifies the negative 
of the information the message provides about the receiver's future conditioned 
on what the receiver has already known—minimizing distortion thus maximizes predictive 
utility. The \textbf{Rate} term quantifies the information the message retains 
about the sender's state—minimizing rate enforces compression. The parameter 
$\beta$ controls this trade-off: larger $\beta$ prioritizes compression, while 
smaller $\beta$ prioritizes predictive accuracy.

Directly optimizing \textbf{Eq.~(\ref{eq:ib_objective_revised})} is intractable because computing mutual information requires integrating over high-dimensional and unknown data distributions. We therefore derive a tractable surrogate objective by constructing variational bounds for both the rate and distortion terms. We introduce a parametric encoder $q_{\varphi}(z|S_{i,t})$ that maps the sender's state to a latent variable $z$ (representing the message $m=g(z)$, where $g(\cdot)$ is message decoder.), and a parametric decoder $p_{\vartheta}(O_{j,t+1}|z, S_{j,t})$ that predicts the receiver's future observation. The full derivation of the variational bounds is provided in \textbf{Supplementary Method~\ref{sec: Derivation of the Variational Bound for the Predictive Information Bottleneck}}.

First, we derive a tractable upper bound for the communication rate, $I(S_{i,t}; m_{i,t})$. By introducing a fixed prior distribution $p(z)$ over the latent message space, the rate can be bounded by the Kullback-Leibler (KL) divergence between the approximate posterior and the prior according to \textbf{Supplementary Lemma \ref{lem:rate-bound}}:
\begin{equation}
\label{eq:rate_bound_revised}
I(S_{i,t}; m_{i,t}) \le \mathbb{E}_{p(S_{i,t})}\Big[D_\text{KL}\big(q_{\varphi}(z\mid S_{i,t})\,\|\,p(z)\big)\Big].
\end{equation}
Minimizing this KL divergence thus serves to minimize an upper bound on the true communication rate, effectively enforcing compression. 

Second, we derive a tractable upper bound for the distortion term, 
$-I(m_{i,t}; O_{j,t+1} \mid S_{j,t})$. Using the entropy decomposition, 
the conditional mutual information can be written as $I(m_{i,t}; O_{j,t+1} 
\mid S_{j,t}) = H(O_{j,t+1} \mid S_{j,t}) - H(O_{j,t+1} \mid m_{i,t}, S_{j,t})$. 
Since the baseline entropy $H(O_{j,t+1} \mid S_{j,t})$ is independent of model 
parameters, minimizing distortion $-I(m_{i,t}; O_{j,t+1} \mid S_{j,t})$ is 
equivalent to minimizing the conditional entropy $H(O_{j,t+1} \mid m_{i,t}, 
S_{j,t})$. We upper-bound this entropy using a parametric decoder 
$p_{\vartheta}(O_{j,t+1}|z, S_{j,t})$, yielding:
\begin{equation}
\label{eq:distortion_bound_revised}
-I(m_{i,t}; O_{j,t+1} \mid S_{j,t}) \le -H(O_{j,t+1} \mid S_{j,t}) 
+ \mathbb{E}_{p(S_{i,t}, S_{j,t}, O_{j,t+1})} \Big[ \mathbb{E}_{q_{\varphi}(z\mid S_{i,t})}
\big[-\log p_{\vartheta}(O_{j,t+1}\mid z, S_{j,t})\big] \Big].
\end{equation}

Since $H(O_{j,t+1} \mid S_{j,t})$ is a constant, minimizing this bound is equivalent 
to minimizing the expected negative log-likelihood (reconstruction error). 
Combining these bounds yields the variational information bottleneck (VIB) 
objective, which is a tractable surrogate for the intractable IB loss. By substituting 
the upper bound for rate and the upper bound for distortion (omitting the 
constant $H(O_{j,t+1} \mid S_{j,t})$), we obtain a new loss function:

\begin{equation}
\label{eq:vib_loss_revised}
\begin{aligned}
\mathcal{L}_{\text{VIB}}(\varphi,\vartheta) = \underbrace{\mathbb{E}_{p(S_{i,t}, S_{j,t}, O_{j,t+1})}
\Big[\mathbb{E}_{q_{\varphi}(z\mid S_{i,t})}\big[-\log p_{\vartheta}(O_{j,t+1}\mid z, 
S_{j,t})\big]\Big]}_{\text{Distortion (Reconstruction Loss)}} 
+ \beta \underbrace{\mathbb{E}_{p(S_{i,t})}\Big[D_{\text{KL}}\big(q_{\varphi}(z\mid S_{i,t}) 
\,\|\, p(z)\big)\Big]}_{\text{Rate (KL Regularizer)}}
\end{aligned}
\end{equation}
where expectations are taken over the joint data distribution. The first term 
minimizes the reconstruction error, maximizing the message's predictive utility. 
The second term minimizes the KL divergence from the prior, enforcing compression. 
We operationalize this using a convolutional variational autoencoder (VAE)
architecture, where the entire system is trained end-to-end to minimize
$\mathcal{L}_{\text{VIB}}$. The encoder implements a hierarchical downsampling
structure that maps a $64 \times 64$ spatial memory map to a compressed latent
message $z$, while the decoder symmetrically reconstructs the partner's predicted
occupancy map from this compressed representation (detailed architecture specifications
in \textbf{Supplementary Tables~\ref{tab:comm_encoder}} and \textbf{\ref{tab:comm_decoder}}).

Each term in this loss directly implements a component of the social predictive
coding framework. The first term (reconstruction loss) corresponds to distortion,
minimizing prediction error and thereby maximizing the message's utility for the
receiver. The second term (KL divergence) corresponds to rate, enforcing compression
and driving the emergence of efficient symbolic mechanisms. The hyperparameter
$\beta$ enables systematic exploration of the rate-distortion trade-off, revealing
how bandwidth constraints shape emergent communication structure (training configuration
in \textbf{Supplementary Table~\ref{tab:comm_training}}). 

\subsubsection*{Level 3: Predictive uncertainty as a guide for strategic exploration}

The predictive coding framework culminates at the level of strategic decision-making. Having established models to predict its environment (Level 1) and its partners' states (Level 2), the agent must now decide how to act in order to improve these predictive models over time. This changes an agent from a passive observer into an active learner, a concept known as active inference. The optimal exploration strategy, from a predictive coding perspective, is to seek out experiences that maximally reduce the uncertainty of an agent's internal generative model. In a large, partially-observable environment, this requires a principled approach to balance exploration and exploitation, a challenge we address with a hierarchical reinforcement learning (HRL) framework explicitly guided by predictive uncertainty.

This framework decomposes the complex navigation problem into two levels. The low-level controller is a deterministic planner (A* search) that executes concrete navigational sub-goals based on an agent's current understanding of the shared world map. This offloads the complexities of pathfinding, allowing the high-level policy (meta-controller) to focus exclusively on the strategic question: ``\textit{Where should I go next to learn the most?}'' The meta-controller is implemented as an actor-critic network and trained using multi-agent proximal policy optimization (MAPPO), a robust algorithm for cooperative settings. To fulfill this, each agent first partitions its occupancy grid into a $4 \times 4$ regional summary, generating a 48-dimensional feature vector that includes the exploration ratio, walkability, and agent occupancy for each region. These features are fed into a shared actor-critic network with two fully connected layers (256 units with ReLU activation) to extract a common embedding. This embedding then branches into an actor head, which outputs a masked categorical distribution over the 16 regions for goal selection, and a critic head, which predicts scalar state values (\textbf{Supplementary Table~\ref{tab:meta_controller_arch}}). The crucial insight is that the reward signal driving this high-level policy is not based on sparse, external rewards (like finding the target), but is instead generated intrinsically from an agent's own state of uncertainty.

This is realized through an enhanced intrinsic curiosity module (ICM), which functions as the direct implementation of the active inference principle. Rather than relying on a learned forward dynamics model, our ICM directly estimates epistemic uncertainty by analyzing the geometry of the known-unknown boundary in an agent's local map, augmented with social spatial information from the SPC module (Level 2). The composite intrinsic reward signal, $r_t^\text{int}$, is a weighted sum of three components, each reflecting a different facet of uncertainty reduction:
\begin{equation}
r_t^\text{int} = w_{\text{curiosity}} r_{\text{curiosity}} + w_{\text{coord}} r_{\text{coord}} + w_{\text{explore}} r_{\text{explore}}.
\end{equation}
Specifically, these components are implemented as follows: (1) The curiosity reward, $r_{\text{curiosity}}$, is a frontier-based score that directly rewards the selection of high-level goals in regions bordering unknown territory, where the shared predictive map is most uncertain. (2) The coordination reward, $r_{\text{coord}}$, promotes spatial division of labor by leveraging the inter-agent distance estimates $\hat{d}_{ij,t}$ provided by the SPC module. This component discourages redundant exploration by rewarding goal selections that maintain appropriate separation from teammates, directly operationalizing the principle that distributed exploration maximizes information gain. Critically, this establishes a computational dependency: the SPC module's learned distance-tuned neurons (\textbf{Fig.~\ref{fig:social_place_code}c}) provide the essential geometric information that enables the ICM to generate coordination signals. (3) The exploration reward, $r_{\text{explore}}$, grants credit for discovering map cells that are unknown to the entire team, targeting states that are, by definition, maximally unpredictable. More details about ICM reward formulations and communication mechanisms are provided in \textbf{Supplementary Method \ref{sec: HRL-ICM Framework}}.

By training the MAPPO policy to maximize this rich, uncertainty-driven intrinsic reward, the HRL-ICM framework learns a sophisticated, coordinated exploration strategy from first principles. It does not rely on hand-crafted heuristics for exploration but instead learns to perform a form of Bayesian experimental design, constantly choosing actions to gather the most informative data. The entire system is trained end-to-end with the SPC module that provides continuous geometric awareness of partner states, the ICM that translates this into strategic exploration incentives, and the meta-controller that selects goals based on these incentives (hyperparameters and training configuration in \textbf{Supplementary Table~\ref{tab:hrl_training_config}}). Thus, from the lowest level of perception to the highest level of strategic planning, the entire architecture is unified under the overarching goal of building and refining a predictive model of the world by actively seeking out and resolving uncertainty.

%% file: sections/suppl.tex
\appendix

\newcommand{\suppmatheading}{%
  \begin{center}
  {\LARGE\bfseries Supplementary Information}\par
  \vspace{1.6em}
  {\large\bfseries Shared Spatial Memory Through Predictive Coding}\par
  \vspace{1em}
  {\normalsize
  Zhengru Fang$^{1,4,\star}$, 
  Yu Guo$^{1,\star}$, 
  Jingjing Wang$^{2,\star}$, 
  Yuang Zhang$^{3}$, 
  Haonan An$^{1}$, \\
  Yinhai Wang$^{3}$, 
  Wenbo Ding$^{4}$, 
  Yuguang Fang$^{1}$
  }\par
  \vspace{0.5em}
  {\footnotesize
  $^1$City University of Hong Kong, 
  $^2$Beihang University, \\
  $^3$University of Washington, 
  $^4$Tsinghua University, 
  $^\star$Equal contribution
  }\par
  \vspace{1em}
  \end{center}
}

\renewcommand{\thefigure}{S\arabic{figure}}
\renewcommand{\thetable}{S\arabic{table}}
\renewcommand{\theequation}{S\arabic{equation}}

\setcounter{figure}{0}
\setcounter{table}{0}
\setcounter{equation}{0}

\renewcommand{\thesection}{S\arabic{section}}
\renewcommand{\thesubsection}{S\arabic{section}.\arabic{subsection}}
\renewcommand{\thesubsubsection}{S\arabic{section}.\arabic{subsection}.\arabic{subsubsection}}

\setcounter{tocdepth}{3}
\setcounter{secnumdepth}{3}

\suppmatheading

\vspace{0.5em}

% \begin{center}
%     {\Large Supplementary information}
% \end{center}
% \vspace{-0.2in}
% \noindent\makebox[\linewidth]{\rule{\linewidth}{1.5pt}}

% \tableofcontents

{\noindent The supplementary material includes:\\[0.25em]
\textbf{Sections} S1 to S2\\
\textbf{Figures} S1 to S11\\
\textbf{Tables} S1 to S13\\
\textbf{Supplementary Videos} S1 to S5\\
}
\titlecontents{section}[1.5em]
{\addvspace{1em}\large\bfseries}
{\contentslabel{2em}}
{\hspace*{-2em}}
{\titlerule*[0.3pc]{.}\contentspage}
[\addvspace{0.5em}]
\titlecontents{subsection}[3em]
{\addvspace{0.6em}\normalsize}
{\contentslabel{3em}}
{\hspace*{-2.5em}}
{\titlerule*[0.3pc]{.}\contentspage}
[\addvspace{0.3em}]
\titlecontents{subsubsection}[5em]
{\addvspace{0.4em}\small}
{\contentslabel{4em}}
{\hspace*{-3em}}
{\titlerule*[0.3pc]{.}\contentspage}
[\addvspace{0.2em}]
\startcontents[sections]
\printcontents[sections]{}{1}{\setcounter{tocdepth}{2}}

\newpage

\input{sections/SI/system_overview}

\input{sections/SI/suppl_network}

\input{sections/SI/Path_Integration_via_grid_cell}

\input{sections/SI/PIB}

\input{sections/SI/sup_figure}

\clearpage

\newpage

\input{sections/SI/sup_movie}

\newpage

\clearpage

\renewcommand{\bibsection}{\section*{Supplementary References}}

\input{sections/SI/ref_sup}
\clearpage

%% file: sections/SI/system_overview.tex
\section{System Overview}

This section provides a comprehensive overview of our system architecture, training pipeline, and experimental configuration. The following detailed supplementary methods (Supplementary Method~\ref{sec: Biologically-Inspired Spatial Representation Networks}--\ref{sec: HRL-ICM Framework}) describe individual components; here we explain how these components integrate into a unified framework for collaborative spatial navigation.

\subsection{System Architecture Overview}

Our framework implements a modular architecture where specialized perceptual and cognitive components are composed into an end-to-end reinforcement learning system. The design philosophy follows a two-stage approach: first learning robust perceptual primitives independently, then integrating them for collaborative decision-making.

\paragraph{Pre-trained perceptual components.}
Three core modules are trained independently before integration into the full multi-agent system:

\begin{itemize}
    \item \textbf{Grid Cell Network} (detailed in Supplementary Method~\ref{sec: Biologically-Inspired Spatial Representation Networks}): A recurrent neural network trained via self-supervised path integration to maintain a dynamic estimate of self-location from egocentric velocity commands. Through predictive learning alone, this module spontaneously develops hexagonally-periodic spatial representations analogous to biological grid cells, providing a stable metric foundation for navigation.

    \item \textbf{Spatial Memory Generation Module} (detailed in Supplementary Method~\ref{sec: Spatial Memory Generation Module}): A transformer-based encoder-decoder architecture that translates first-person RGB observations into structured bird's-eye-view (BEV) spatial maps. Trained end-to-end on an egocentric-to-allocentric translation task, this module enables agents to construct persistent spatial memories from partial visual observations.

    \item \textbf{Emergent Communication Mechanism} (detailed in Supplementary Method~\ref{sec: Emergent Communication Mechanism}): A variational autoencoder (VAE) implementing the information bottleneck principle for bandwidth-constrained message passing. Trained to compress spatial maps while preserving predictive utility for partner agents, this module learns an efficient symbolic communication protocol without explicit supervision on message semantics.
\end{itemize}

\paragraph{Social representation substrate.}
Building upon the individual grid cell network, a specialized \textbf{Social Place Cell Module} (detailed in Supplementary Method~\ref{sec: Social Place Cell}, \textbf{Table~\ref{tab:spc_arch}}) extends spatial coding to the multi-agent domain. This dual-stream LSTM architecture processes motion information from both self and partner agents, learning to predict not only individual positions but also inter-agent distance relationships. Through multi-task predictive training, this module develops functionally specialized neural populations—including units selectively tuned to partner locations (social place cells)—that provide the geometric awareness necessary for coordination.

\paragraph{End-to-end reinforcement learning framework.}
The \textbf{HRL-ICM Framework} (detailed in Supplementary Method~\ref{sec: HRL-ICM Framework}) integrates all pre-trained components into a hierarchical policy optimized for collaborative exploration and target search. A meta-controller network selects high-level exploration goals based on epistemic uncertainty, while an intrinsic curiosity module (ICM) generates dense reward signals from map coverage and inter-agent coordination. The social place cell module's distance estimates enable the ICM to compute coordination rewards that encourage spatial division of labor, while the communication mechanism allows selective information sharing when bandwidth-efficient messages can reduce collective uncertainty. This entire system is trained end-to-end using Multi-Agent Proximal Policy Optimization (MAPPO) on the Memory-Maze benchmark.

The modular design enables interpretability (each component's emergent representations can be analyzed independently), sample efficiency (pre-trained components provide strong initialization), and scalability (individual modules can be upgraded without full system retraining).

\subsection{Training Pipeline}

The complete training procedure follows a three-stage curriculum designed to establish stable perceptual foundations before optimizing collaborative behavior:

\begin{enumerate}
    \item \textbf{Stage 1 -- Individual Component Pre-training (Independent Training)}:
    \begin{itemize}
        \item \textit{Grid Cell Network}: Trained on synthetic trajectory datasets with randomized velocity sequences in $15 \times 15$ meter environments. The network learns to predict place and head-direction cell activations after integrating motion commands over 100-step sequences. Training uses Adam optimizer (learning rate $1 \times 10^{-3}$, batch size 64) for approximately 50 epochs until gridness scores stabilize (typically $> 0.3$).

        \item \textit{BEV Prediction Network}: Trained on paired egocentric RGB / allocentric BEV map datasets collected from agent trajectories in Memory-Maze environments. The transformer encoder-decoder is optimized to minimize a composite loss balancing occupancy prediction (binary cross-entropy), appearance reconstruction (masked MSE), and spatial smoothness. Training uses Adam optimizer (learning rate $1 \times 10^{-4}$, batch size 64) until validation IoU exceeds 0.85.

        \item \textit{Communication VAE}: Trained on spatial map reconstruction tasks with information bottleneck objective. Paired sender/receiver map samples are generated from dual-agent exploration trajectories. The encoder-decoder architecture is optimized to balance reconstruction accuracy against KL divergence from a Gaussian prior ($\beta = 1.0$). Training uses Adam optimizer (learning rate $1 \times 10^{-3}$, batch size 32) for 50 epochs, achieving 32:1 compression with $< 5\%$ reconstruction error.
    \end{itemize}

    \item \textbf{Stage 2 -- Social Place Cell Module Training (Dual-Agent Path Integration)}:
    \begin{itemize}
        \item The social place cell module is initialized with pre-trained grid cell network weights for the self-stream, then extended with a partner-stream LSTM.

        \item Training data consists of synchronized dual-agent trajectories where both agents' poses and velocities are recorded. The partner velocity is estimated from visual observations of the partner's motion when visible.

        \item The module is trained to predict final-timestep place/head-direction activations for both agents plus their Euclidean distance via multi-task loss: $\mathcal{L}_{\text{SPC}} = \mathcal{L}_{\text{self}} + \mathcal{L}_{\text{partner}} + \mathcal{L}_{\text{distance}}$.

        \item Training uses Adam optimizer (learning rate $1 \times 10^{-3}$, batch size 64) for 30 epochs. Success criterion: distance prediction MAE $< 1.0$ grid cells and emergence of distinct social place cell populations (verified via neural selectivity analysis).
    \end{itemize}

    \item \textbf{Stage 3 -- HRL-ICM End-to-End Training (Multi-Agent Reinforcement Learning)}:
    \begin{itemize}
        \item Pre-trained components (grid cell, BEV prediction, communication VAE, social place cell module) are integrated into the full HRL-ICM system.

        \item \textit{Initial Freezing Phase} (first $10^4$ environment steps): All pre-trained components have frozen parameters to allow meta-controller and ICM to stabilize against fixed perceptual representations. Only the meta-controller actor-critic network and communication gating policy are updated via MAPPO.

        \item \textit{Joint Fine-tuning Phase} (remaining training): All components are unfrozen and jointly optimized end-to-end. Meta-controller updates continue via MAPPO (learning rate $3 \times 10^{-4}$), while perceptual modules receive gradients from both task rewards and their original supervised objectives (weighted $0.1 \times$ to prevent catastrophic forgetting).

        \item Training environments: Procedurally generated Memory-Maze instances with increasing complexity (maze sizes $15 \times 15$ to $39 \times 39$, 2--5 agents). Training continues for $5 \times 10^6$ environment steps or until success rate on $29 \times 29$ mazes exceeds 70\%.

        \item Curriculum: Maze size increases every $10^6$ steps; communication bandwidth decreases from 128 bits/step to target bandwidth (4--16 bits/step) over first $2 \times 10^6$ steps to encourage robust communication strategies.
    \end{itemize}
\end{enumerate}

This staged training strategy ensures that higher-level reasoning builds upon stable low-level representations, preventing the RL policy from corrupting perceptual modules during exploration. The modular architecture also enables ablation studies where individual components can be selectively frozen or replaced.

\subsection{Experimental Configuration}

The complete experimental setup encompasses hardware infrastructure, simulation environment parameters, and training hyperparameters across all system components.

\paragraph{Evaluation protocol.}
All reported results are averaged over 100 evaluation episodes with fixed random seeds for reproducibility. For bandwidth scaling experiments (Fig.~\ref{fig:communication}c--d), agents are tested across latent dimensions $z_\text{dim} \in \{16, 32, 64, 128, 256\}$ corresponding to $\{4, 8, 16, 32, 64\}$ bits/step (assuming 4-bit quantization). For generalization tests, models trained on $29 \times 29$ mazes are evaluated on unseen $39 \times 39$ mazes without further training. Statistical significance is assessed via two-sample $t$-tests with Bonferroni correction for multiple comparisons.

\paragraph{Forward reference.}
Detailed specifications for each component's network architecture, loss functions, and training procedures are provided in the following supplementary methods sections: Grid Cell Network (Supplementary Method~\ref{sec: Biologically-Inspired Spatial Representation Networks}, \textbf{Tables~\ref{tab:grid_cell_network}, \ref{tab:grid_cell_training}}), Spatial Memory Generation (Supplementary Method~\ref{sec: Spatial Memory Generation Module}, \textbf{Tables~\ref{tab:resnet_detector}, \ref{tab:bev_encoder}, \ref{tab:bev_decoder}}), Communication Mechanism (Supplementary Method~\ref{sec: Emergent Communication Mechanism}, \textbf{Tables~\ref{tab:comm_encoder}, \ref{tab:comm_decoder}, \ref{tab:comm_training}}), and HRL-ICM Framework (Supplementary Method~\ref{sec: HRL-ICM Framework}, \textbf{Tables~\ref{tab:spc_arch}, \ref{tab:meta_controller_arch}, \ref{tab:hrl_training_config}}).

%% file: sections/SI/suppl_network.tex
\section{Supplementary Methods}

\subsection{Grid cell network: Biological-inspired spatial representation}
\label{sec: Biologically-Inspired Spatial Representation Networks}
A fundamental cognitive capability for navigation is the ability to maintain a dynamic estimate of self-location and orientation from idiothetic cues—a process known as path integration. In mammals, this function is robustly implemented by the hippocampal-entorhinal system, which provides a canonical example of a stable internal metric for space. Inspired by this biological solution~\cite{banino2018vector_supp}, we developed a computational model (termed grid cell network) designed to test the hypothesis that the characteristic neural codes for space, such as grid cells, can emerge from a general predictive learning objective, without being explicitly engineered into the system's architecture. Grid cell network, a recurrent neural network, is tasked not only with path integration, but also with a more fundamental objective: to predict its own future sensory state given a sequence of self-motion cues. We demonstrate that to solve this continuous self-prediction problem efficiently~\cite{cueva2020emergence_supp}, the network is compelled to develop a highly structured internal representational scheme. This scheme, we show, is a convergent solution that recapitulates key properties of biological spatial representations, most notably the spontaneous formation of hexagonally symmetric grid-like firing patterns.

\subsubsection{Network architecture}

The core of our grid cell network is a recurrent neural network (RNN) formulated as a predictive state-space model. Its objective is to learn the transition dynamics of an agent's pose by continuously predicting its future location and orientation. The sequential and cumulative nature of path integration presents a significant challenge involving long-term temporal dependencies and the integration of noisy velocity signals. To address this, we selected a long short-term memory (LSTM) network as the central recurrent component~\cite{graves2012long_supp}. The LSTM's gating mechanisms are exceptionally well-suited to learning to retain or discard information over extended time horizons, providing a robust substrate for integrating velocity commands while mitigating the vanishing gradient problems that would plague a simpler RNN architecture in this task. The network's full architecture is specified in \textbf{Fig. \ref{fig:figS-gcn}} and \textbf{Table \ref{tab:grid_cell_network}}.

\begin{figure}[h]
\centering
\includegraphics[width=1\linewidth]{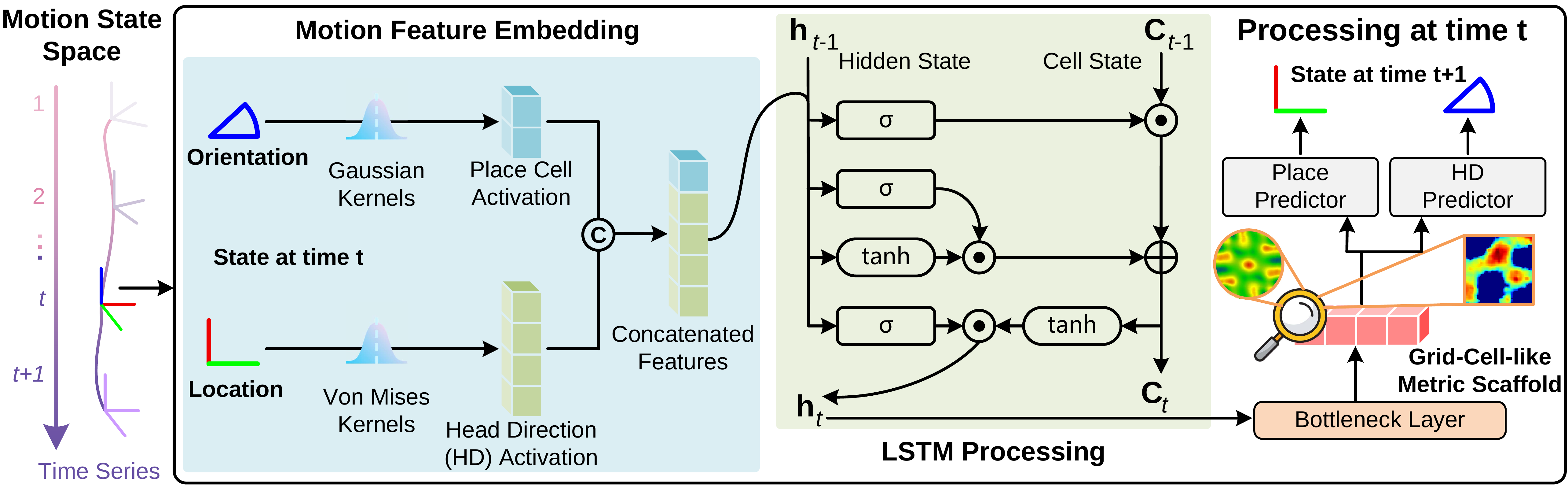}
\caption{\textbf{Architecture of grid cell network.}}
\label{fig:figS-gcn}
\end{figure}

The computational flow is designed to test a key principle of biological navigation: the anchoring of a dynamic, self-motion-based estimate of position to stable, absolute sensory cues. The process begins by encoding the agent's initial pose ($s_0$) into a high-dimensional, distributed representation using ensembles of virtual place and head direction (HD) cells~\cite{cueva2020emergence_supp}. This initial representation, grounded in an absolute reference frame, is projected through learnable linear layers to initialize the LSTM's hidden state ($\mathbf{h}_0$) and cell state ($\mathbf{c}_0$). This initialization strategy is not a mere technical convenience; it is a critical design choice that simulates the neural mechanism of ``remapping'' or ``resetting'' observed when an animal enters a new environment. By grounding the integrator's initial state in a veridical sensory observation, we prevent the unbounded accumulation of drift that is the primary failure mode of pure idiothetic integration.

From this anchored starting point, the LSTM iteratively updates its internal state by processing the sequence of velocity inputs. At each timestep, the LSTM's hidden state—its internal representation of the agent's current pose—is passed through a linear bottleneck layer before prediction. This bottleneck is a central element of our experimental design. By constraining the dimensionality of the representation passed to the predictors, we impose a powerful efficiency constraint on the network. This forces the network to discover a compressed, abstract, and maximally informative latent code for pose. The central scientific question this architecture addresses is: What is the geometry of this emergent representation? Is it an unstructured code, or does it converge to a structured, periodic representation, such as the hexagonal lattice of grid cells, which is hypothesized to be an information-theoretically optimal code for 2D space? Finally, two linear predictor heads decode this latent representation back into the ``sensory'' space of place and head direction cell activations. The entire network is trained end-to-end by minimizing the discrepancy between these predictions and the ground-truth sensory states, thereby closing the predictive loop and forcing the internal dynamics of the model to learn the dynamics of the external world.

\begin{table}[t]
\centering
\caption{\textbf{Details of grid cell network.} $B$ and $Seq$ mean the batch size and sequence length, respectively.}
\setlength{\tabcolsep}{22pt}
\label{tab:grid_cell_network}
\begin{tabular}{llll}
\toprule
\textbf{Component} & \textbf{Input Size} & \textbf{Output Size} & \textbf{Parameters} \\
\midrule
\multicolumn{4}{c}{Velocity Processing} \\
\midrule
Velocity Input & $B \times Seq \times$ 3 & $B \times Seq \times$ 3 & - \\
\midrule
\multicolumn{4}{c}{Initial State Computation} \\
\midrule
Place Cell Activation & $B \times$ 2 & $B \times$ 256 & Gaussian Kernels \\
Head Direction Activation & $B \times$ 1 & $B \times$ 32 & Von Mises Kernels \\
Concatenated Features & $B \times$ 288 & $B \times$ 288 & - \\
\midrule
\multicolumn{4}{c}{LSTM Processing} \\
\midrule
Initial Hidden State ($\mathbf{h}_0$) & $B \times$ 288 & 1 $\times B \times$ 128 & Linear(288, 128) \\
Initial Cell State ($\mathbf{c}_0$) & $B \times$ 288 & 1 $\times B \times$ 128 & Linear(288, 128) \\
LSTM Layer & $B \times Seq \times$ 3 & $B \times Seq \times$ 128 & 1 layer, 128 units \\
\midrule
\multicolumn{4}{c}{Bottleneck and Prediction} \\
\midrule
Bottleneck Layer & $B \times Seq \times$ 128 & $B \times Seq \times$ 256 & Linear + Dropout(0.5) \\
Place Predictor & $B \times Seq \times$ 256 & $B \times Seq \times$ 256 & Linear(256, 256) \\
HD Predictor & $B \times Seq \times$ 256 & $B \times Seq \times$ 32 & Linear(256, 32) \\
\bottomrule
\end{tabular}
\end{table}

\subsubsection{Place and head direction cell ensembles as a sensory ground truth}

To provide the network with a stable, biologically-plausible sensory ground truth for its predictive task, we constructed virtual ensembles of place cells and head direction cells. These cell types are not merely convenient choices; they are well-established models for neural representations of location and orientation, respectively. Their tuning profiles provide a smooth, distributed, and overcomplete basis for representing the agent's state, a feature that confers robustness to noise and local errors.

\paragraph{Place cell ensemble.}
Location is encoded by an ensemble of place cells with Gaussian receptive fields distributed uniformly across the environment. This creates a population code where each location elicits a unique pattern of activation. The activation of each cell $i$ is a normalized Gaussian function of the agent's position $\mathbf{x}$:
\begin{equation}
a_i(\mathbf{x}) = \frac{\exp\left(-\frac{||\mathbf{x} - \boldsymbol{\mu}_i||^2}{2\sigma_i^2}\right)}{\sum_{j=1}^{N} \exp\left(-\frac{||\mathbf{x} - \boldsymbol{\mu}_j||^2}{2\sigma_j^2}\right)},
\end{equation}
where $\boldsymbol{\mu}_i$ is the preferred location of cell $i$. This normalization ensures the population vector of activations forms a probabilistic distribution over the state space.

\paragraph{Head direction cell ensemble.}
Orientation is encoded by an ensemble of head direction cells whose firing is tuned to the agent's heading angle. We model these using von Mises distributions, the circular analogue of the Gaussian distribution, which accurately captures the tuning curves of biological head direction cells. The activation of each cell $j$ is given by:
\begin{equation}
b_j(\theta) = \frac{\exp(\kappa_j \cos(\theta - \phi_j))}{\sum_{k=1}^{M} \exp(\kappa_k \cos(\theta - \phi_k))},
\end{equation}
where $\theta$ is the heading angle, $\phi_j$ is the cell's preferred direction, and $\kappa_j$ controls the tuning width, allowing for a diversity of specificities within the population.

% \begin{table}[t]
% \centering
% \caption{Spatial Cell Ensemble Configuration}
% \label{tab:spatial_cells}
% \begin{tabular}{|l|l|l|l|}
% \hline
% \textbf{Cell Type} & \textbf{Number of Cells} & \textbf{Distribution Parameters} & \textbf{Spatial Coverage} \\
% \hline
% Place Cells & 256 & $\boldsymbol{\mu}_i \sim \text{Uniform}(0, 15)^2$, $\sigma = 1.0$ & $15 \times 15$ m environment \\
% Head Direction Cells & 32 & $\phi_j \sim \text{Uniform}(-\pi, \pi)$, $\kappa = 20$ & Full $2\pi$ angular range \\
% \hline
% \end{tabular}
% \end{table}

\subsubsection{Objective function for predictive spatial learning}
\label{sec: Objective Function for Predictive Spatial Learning}

To drive the emergence of structured representations, we design a specialized, multi-component loss function that guides the learning process. This objective function is not arbitrary but is carefully structured to enforce key constraints on the predictive task, which can be defined as:
\begin{equation}
\mathcal{L}_\text{total} = \sum_{t=0}^{T-1} w_t \left( \mathcal{L}_\text{place}^{(t)} + \mathcal{L}_\text{HD}^{(t)} \right) + w_{\text{init}} \mathcal{L}_\text{init} + w_{\text{cont}} \mathcal{L}_\text{cont}.
\end{equation}

\paragraph{Probabilistic loss and temporal weighting.}
We frame the prediction task in probabilistic terms by using the Kullback-Leibler (KL) divergence between the network's predicted distribution of cell activations and the ground-truth distribution. This treats learning as a process of minimizing the informational ``surprise'' between the model's belief and reality, a more principled approach than simple regression for dealing with distributed neural codes.
\begin{align}
\mathcal{L}_\text{place}^{(t)} &= \text{KL}\left( \text{softmax}(\hat{\mathbf{p}}_t) \| \mathbf{p}_t^\text{target} \right), \\
\mathcal{L}_\text{HD}^{(t)} &= \text{KL}\left( \text{softmax}(\hat{\mathbf{h}}_t) \| \mathbf{h}_t^\text{target} \right).
\end{align}

Furthermore, path integration errors are cumulative; small initial errors can propagate and corrupt the entire trajectory estimate. To counteract this, we implement a temporally-weighted loss scheme. The weights $w_t$ place a strong emphasis on accuracy in the initial phase of a trajectory, effectively creating a curriculum that forces the network to first master the fundamental single-step dynamics.
\begin{equation}
w_t = \begin{cases}
2 \cdot w_\text{init} & \text{if } t = 0 \\
w_\text{init} \cdot \rho^{t-1} & \text{if } 1 \leq t < 5 \\
1.0 & \text{if } t \geq 5
\end{cases} \quad \text{with } w_\text{init} = 5.0, \rho = 0.8.
\end{equation}

\paragraph{Regularization for spatial and temporal coherence.}
Two additional regularization terms enforce plausible constraints on the spatial representation. The initial consistency loss, $\mathcal{L}_\text{init}$, explicitly penalizes any mismatch between the network's initial prediction and the initial sensory ground truth. This reinforces the ``anchoring'' mechanism described previously.
\begin{equation}
\mathcal{L}_\text{init} = 2.0 \cdot \text{KL}\left( \text{softmax}(\hat{\mathbf{p}}_0) \| \mathbf{p}_0^\text{target} \right).
\end{equation}

The continuity loss, $\mathcal{L}_\text{cont}$, encourages smoothness in the network's predictions over consecutive timesteps. This is a biologically plausible prior, as an agent's belief about its location should not change drastically. This term regularizes the learned dynamics, preventing jittery state estimates and promoting the learning of a continuous manifold representation.
\begin{equation}
\mathcal{L}_\text{cont} = \frac{1}{5} \sum_{t=1}^{5} (1 - 0.15(t-1)) \cdot \text{KL}\left( \text{softmax}(\hat{\mathbf{p}}_t) \| \text{softmax}(\hat{\mathbf{p}}_{t-1}) \right).
\end{equation}

More details about grid cell network training are listed in \textbf{Table \ref{tab:grid_cell_training}}.

\begin{table}[t]
\centering
\caption{\textbf{Training configuration of grid cell network.}}
\setlength{\tabcolsep}{35pt}
\label{tab:grid_cell_training}
\begin{tabular}{lll}
\toprule
\textbf{Parameter Category} & \textbf{Parameter} & \textbf{Value} \\
\midrule
\multirow{6}{*}{Loss Weights} & Initial Frame Weight ($w_\text{init}$) & 5.0 \\
 & Decay Factor ($\rho$) & 0.8 \\
 & Initial Consistency Weight ($w_{\text{init}}$) & 2.0 \\
 & Continuity Weight ($w_{\text{cont}}$) & 0.1 \\
 & Sequence Length & 100 time steps \\
 & Temporal Focus Window & First 5 time steps \\
\midrule
\multirow{4}{*}{Network Parameters} & LSTM Hidden Size & 128 \\
 & Bottleneck Size & 256 \\
 & Dropout Rate & 0.5 \\
 & Input Velocity Dimension & 3 (2D + angular) \\
\midrule
\multirow{3}{*}{Training Setup} & Batch Size & 64 \\
 & Learning Rate & $1 \times 10^{-3}$ \\
 & Environment Size & $15 \times 15$ m \\
\bottomrule
\end{tabular}
\end{table}

\subsubsection{Quantitative analysis of emergent spatial representations}
\label{sec: Quantitative analysis of emergent spatial representations}
To objectively determine whether the network's learned representations exhibit the characteristic firing patterns of grid cells, we employ a standardized quantitative analysis pipeline adapted directly from the field of neurophysiology. This procedure allows us to rigorously identify the presence of hexagonal periodicity in the activity of individual units within the network's bottleneck layer, providing the crucial empirical link between our computational model and the biological phenomenon it seeks to explain.

\paragraph{Rate map.}
For any given unit, we first compute an occupancy-normalised firing rate map, which represents its average activation as a function of the agent's 2D position. For a unit with activation sequence $\{a_t\}_{t=1}^{T}$ along a trajectory $\{(x_t,y_t)\}_{t=1}^{T}$, the arena is discretised into bins $(i,j)$. The rate map $R(i,j)$ is calculated by dividing the total activation within a bin, $\mathcal{S}(i,j)$, by the time spent in that bin, $\mathcal{O}(i,j)$:
\begin{equation}
\label{eq:rate-map}
R(i,j)=
\begin{cases}
\mathcal{S}(i,j)\big/\mathcal{O}(i,j), & \mathcal{O}(i,j)>0,\\[2pt]
\mathrm{NaN}, & \text{otherwise.}
\end{cases}
\end{equation}

This map is typically smoothed with a Gaussian kernel to mitigate sampling noise. Let the mask of visited bins be $\mathcal{V}=\{(i,j): \mathcal{O}(i,j)>0\}$.

\paragraph{Spatial autocorrelogram (SAC).}
To reveal periodic structure in the rate map, we compute its 2D spatial autocorrelation. The autocorrelogram, $\mathrm{SAC}(\boldsymbol{d})$, measures the Pearson correlation between the rate map $R$ and a spatially shifted version of itself for every possible offset $\boldsymbol{d}=(d_x,d_y)$.
\begin{equation}
\label{eq:sac}
\mathrm{SAC}(\boldsymbol{d})=
\frac{\sum_{(i,j)\in \mathcal{V} _{\boldsymbol{d}}}{\!}(R(i,j)-\mu _{\boldsymbol{d}})\,(R(i+d_x,j+d_y)-\mu _{\boldsymbol{d}}\prime)}{\sqrt{\sum_{(i,j)\in \mathcal{V} _{\boldsymbol{d}}}{\!}(R(i,j)-\mu _{\boldsymbol{d}})^2}\;\sqrt{\sum_{(i,j)\in \mathcal{V} _{\boldsymbol{d}}}{\!}(R(i+d_x,j+d_y)-\mu _{\boldsymbol{d}}^{\prime})^2}},
\end{equation}
where $\mathcal{V}_{\boldsymbol{d}}$ is the set of overlapping valid bins for a given shift, and $\mu, \mu', \sigma, \sigma'$ are the respective sample means and standard deviations. A hexagonally periodic firing pattern manifests as a central peak at $\boldsymbol{d}=\mathbf{0}$ surrounded by a ring of six additional peaks, forming a hexagonal lattice in the correlogram.

\paragraph{Gridness score calculation.}
To quantify the degree of hexagonal symmetry, we compute a ``gridness score''. This involves isolating the ring of peaks in the SAC surrounding the central peak using an annular mask, $\mathcal{A}(r_{\min},r_{\max})$. We then measure the rotational symmetry of the pattern within this annulus by calculating the Pearson correlation, $C_\theta$, between the annulus and a version of itself rotated by an angle $\theta$.
\begin{equation}
\label{eq:ring-corr}
C_\theta(r_{\min},r_{\max})
=\frac{\displaystyle\sum_{\boldsymbol{d}\in\mathcal{A}(r_{\min},r_{\max})}
\big(S(\boldsymbol{d})-\bar s\big)\,\big(S_\theta(\boldsymbol{d})-\bar s\big)}
{\displaystyle\sum_{\boldsymbol{d}\in\mathcal{A}(r_{\min},r_{\max})}\big(S(\boldsymbol{d})-\bar s\big)^2},
\end{equation}
where $S$ is the SAC image and $\bar s$ its mean value over the annulus. $S_\theta(\boldsymbol{d}) = S(\mathbf{R}_\theta \boldsymbol{d})$, where the matrix $\mathbf{R}_\theta$ is a rotation matrix: $\mathbf{R}_\theta = \begin{bmatrix} \cos\theta & -\sin\theta \\ \sin\theta & \cos\theta \end{bmatrix}$. The final gridness score contrasts the correlations at angles consistent with a hexagonal lattice ($60^\circ, 120^\circ$) with those at inconsistent angles ($30^\circ, 90^\circ, 150^\circ$).
\begin{equation}
\label{eq:gridness-60}
G_{60}(r_{\min},r_{\max})
=\frac{C_{60}+C_{120}}{2}-\frac{C_{30}+C_{90}+C_{150}}{3},
\end{equation}
where $C_{\theta}$ omits $(r_{\min},r_{\max})$, defined by \textbf{Eq. (\ref{eq:ring-corr})}.
A high positive score indicates strong hexagonal symmetry. For robustness, this score is maximized over a range of possible annulus radii $(r_{\min},r_{\max})$ for each unit. An analogous score, $G_{90}$, can be computed to test for four-fold (square) symmetry.
\begin{equation}
\label{eq:gridness-90}
G_{90}(r_{\min},r_{\max})=C_{90}-\tfrac{1}{2}\big(C_{45}+C_{135}\big).
\end{equation}

This rigorous, standardized methodology allows for a direct, quantitative comparison between the representations learned by our model and those observed in electrophysiological recordings.

\subsection{Spatial memory generation module}
\label{sec: Spatial Memory Generation Module}

The generation of a stable, allocentric spatial memory from a stream of egocentric visual inputs is a cornerstone of an agent's individual world model. This process addresses a fundamentally ill-posed inverse problem: inferring a structured, top-down world representation from ambiguous, high-dimensional localized sensory data. Our architecture is designed as a hierarchical processing pipeline consisting of three primary stages: (1) object-centric feature extraction from the raw first-person view; (2) cross-view encoding of visual features into a latent representation; and (3) decoding into a structured bird's-eye-view (BEV) spatial map~\cite{philion2020lift_supp,pan2020cross_supp}. This pipeline leverages the complementary strengths of convolutional and transformer-based networks, and is trained end-to-end under a composite predictive objective that imposes a set of physical priors on the generative model, guiding it toward geometrically and photometrically plausible solutions. We describe each processing stage in detail below.

\begin{table}[t]
\centering
\caption{\textbf{Details of ResNet-based target detection network.}}
\setlength{\tabcolsep}{13pt}
\label{tab:resnet_detector}
\begin{tabular}{lllll}
\toprule
\textbf{Layer} & \textbf{Input Size} & \textbf{Output Size} & \textbf{Parameters} & \textbf{Activation} \\
\midrule
\multicolumn{5}{c}{Spatial Feature Extractor} \\
\midrule
Input Image & $3 \times 64 \times 64$ & $3 \times 64 \times 64$ & - & - \\
ResNet18 Conv1 & $3 \times 64 \times 64$ & $64 \times 32 \times 32$ & $k=7, s=2, p=3$ & ReLU \\
BatchNorm + MaxPool & $64 \times 32 \times 32$ & $64 \times 16 \times 16$ & $k=3, s=2, p=1$ & - \\
ResNet18 Layer1 & $64 \times 16 \times 16$ & $64 \times 16 \times 16$ & 2 residual blocks & ReLU \\
ResNet18 Layer2 & $64 \times 16 \times 16$ & $128 \times 8 \times 8$ & 2 residual blocks & ReLU \\
ResNet18 Layer3 & $128 \times 8 \times 8$ & $256 \times 8 \times 8$ & 2 residual blocks & ReLU \\
Conv2d Compression & $256 \times 8 \times 8$ & $128 \times 8 \times 8$ & $k=1$ & ReLU \\
BatchNorm2d & $128 \times 8 \times 8$ & $128 \times 8 \times 8$ & - & - \\
Conv2d Refinement & $128 \times 8 \times 8$ & $64 \times 8 \times 8$ & $k=3, p=1$ & ReLU \\
BatchNorm2d & $64 \times 8 \times 8$ & $64 \times 8 \times 8$ & - & - \\
\midrule
\multicolumn{5}{c}{Position Predictor ($N_{\text{obj}} = 6$)} \\
\midrule
Shared Attention Conv & $64 \times 8 \times 8$ & $32 \times 8 \times 8$ & $k=1$ & ReLU \\
BatchNorm2d & $32 \times 8 \times 8$ & $32 \times 8 \times 8$ & - & - \\
Attention Logits & $32 \times 8 \times 8$ & $6 \times 8 \times 8$ & $k=1$ & - \\
Spatial Softmax & $6 \times 8 \times 8$ & $6 \times 8 \times 8$ & Per-object norm & - \\
Weighted Pooling & $(64 \times 8 \times 8) \times 6$ & $6 \times 64$ & Attention-weighted & - \\
Concat w/ Orientation & $6 \times 64$ & $6 \times 66$ & Angle: $[\cos\theta, \sin\theta]$ & - \\
Position MLP-1 & $6 \times 66$ & $6 \times 64$ & Linear + LayerNorm & ReLU \\
Dropout & $6 \times 64$ & $6 \times 64$ & $p=0.2$ & - \\
Position Output & $6 \times 64$ & $6 \times 2$ & Linear(64, 2) & - \\
\midrule
\multicolumn{5}{c}{Visibility Predictor} \\
\midrule
Global Avg Pool & $64 \times 8 \times 8$ & $64$ & Spatial reduction & - \\
Concat w/ Orientation & $64$ & $66$ & Angle: $[\cos\theta, \sin\theta]$ & - \\
Visibility MLP-1 & $66$ & $64$ & Linear + LayerNorm & ReLU \\
Dropout & $64$ & $64$ & $p=0.2$ & - \\
Visibility Output & $64$ & $6$ & Linear(64, 6) & Sigmoid \\
\midrule
\multicolumn{5}{c}{Output Composition} \\
\midrule
Masked Positions & $(6 \times 2) \times (6 \times 1)$ & $6 \times 2$ & Element-wise product & - \\
\bottomrule
\end{tabular}
\end{table}

\begin{figure}[t]
\centering
\includegraphics[width=1\linewidth]{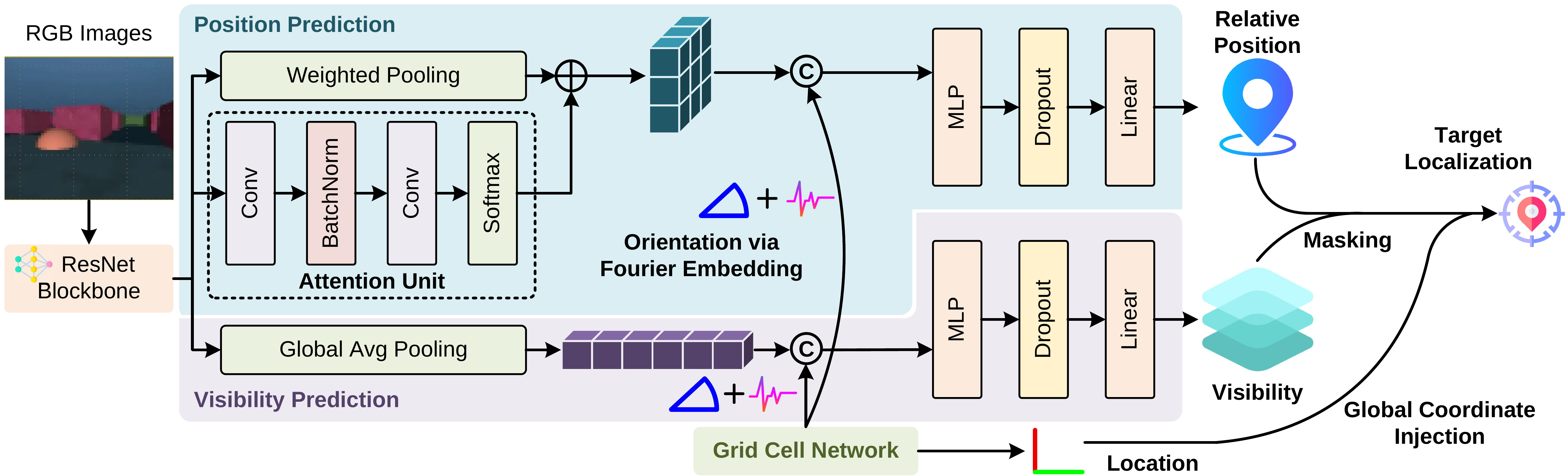}
\caption{\textbf{Architecture of target localization network.}}
\label{fig:figS-tln}
\end{figure}

\begin{figure}[h]
\centering
\includegraphics[width=1\linewidth]{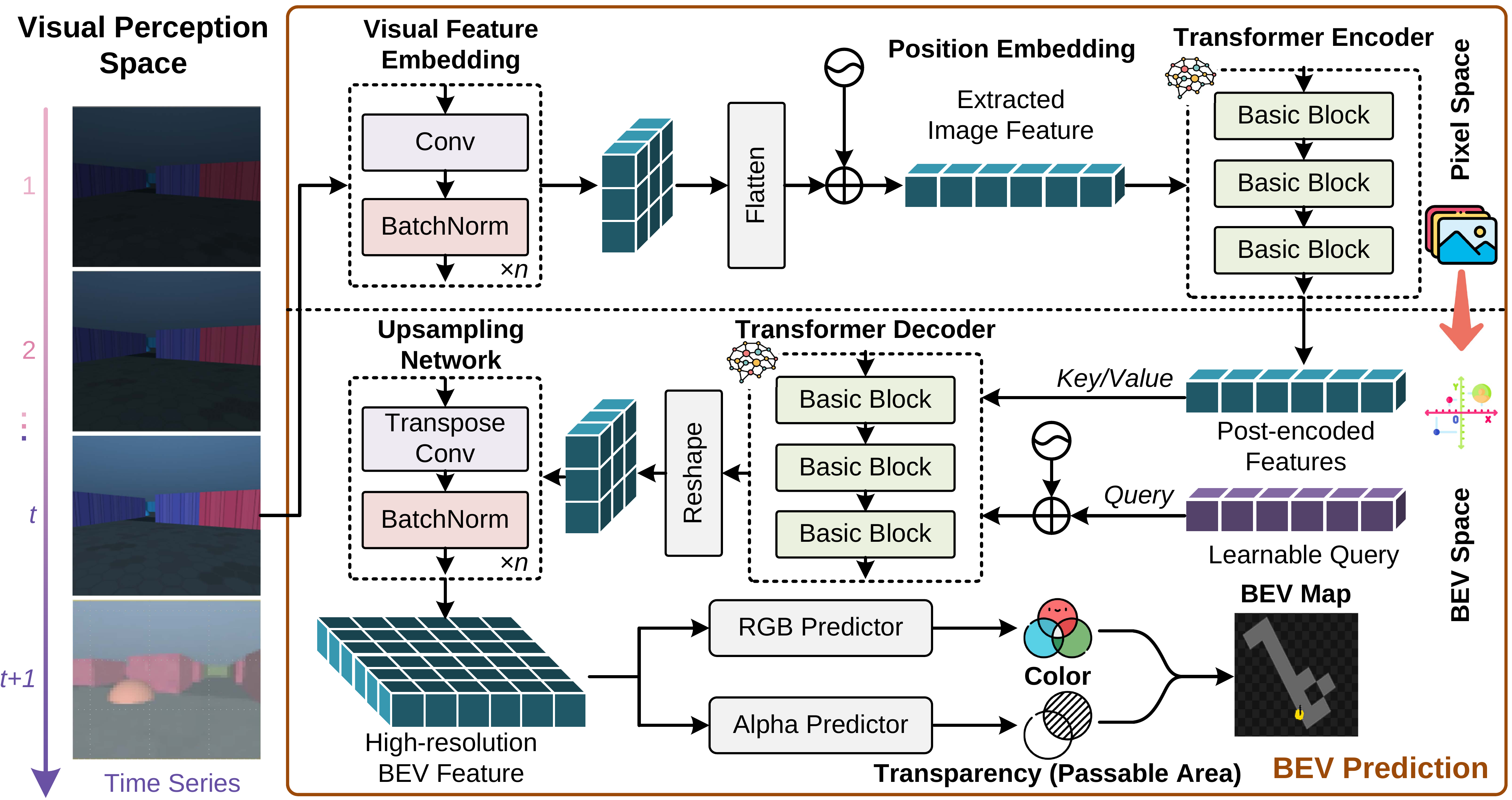}
\caption{\textbf{Architecture of BEV prediction network.}}
\label{fig:figS-bevpn}
\end{figure}

\subsubsection{Target localization network}
\label{sec:resnet_target_detector}

Accurate detection and localization of task-relevant objects within the egocentric visual field is a prerequisite for effective spatial memory construction. This module addresses a fundamental challenge in embodied perception: extracting sparse, object-centric representations from dense, pixel-level observations under significant geometric ambiguity. The target localization network must simultaneously solve three interrelated inference problems—identifying which objects are visible, estimating their relative positions in the agent's reference frame, and filtering out visual distractors—all from a single monocular image acquired from a dynamically moving viewpoint.

Our design is grounded in a biologically-inspired hierarchical processing principle observed in the ventral visual stream of the primate brain, where object recognition emerges through a cascade of increasingly abstract feature representations. Early visual cortex extracts local edge orientations and textures, intermediate areas construct viewpoint-invariant part-based representations, and higher areas in the inferotemporal cortex achieve categorical object identity. Analogously, our architecture employs a convolutional backbone to progressively abstract the raw visual input into a hierarchy of increasingly semantic features, culminating in explicit position and visibility estimates for each potential target object. This hierarchical abstraction is critical for achieving robustness to variations in lighting, viewing angle, and partial occlusions that plague end-to-end direct regression approaches.

The network architecture, specified in \textbf{Fig.~\ref{fig:figS-tln}} and \textbf{Table~\ref{tab:resnet_detector}}, is composed of three functionally distinct stages. The \textit{Spatial Feature Extractor} leverages a truncated ResNet-18 backbone, retaining only the initial convolutional stem and the first three residual blocks (up to \texttt{layer3}). This design choice reflects a deliberate trade-off: deeper layers of ResNet are optimized for semantic categorization tasks, whereas our task requires preserving fine-grained spatial structure for accurate localization. The truncated backbone outputs a $256$-channel feature map at $1/8$ spatial resolution, which is subsequently refined by a compact spatial processing module consisting of two $1 \times 1$ and $3 \times 3$ convolutional layers that compress the representation to $64$ channels while enhancing locality-sensitive features through grouped convolutions and batch normalization.

The \textit{Position Predictor} implements a spatial attention mechanism to extract object-specific representations from the shared feature map. Rather than employing independent processing streams for each object—which would scale poorly and ignore inter-object spatial relationships—we utilize a shared convolutional attention generator that produces a multi-channel attention map, with each channel corresponding to one potential target. This design enforces a structural prior that object locations are spatially disentangled and can be decoded through weighted spatial pooling. For each object $k \in \{1, \ldots, N_{\text{obj}}\}$, the attention map $\mathbf{A}_k \in \mathbbm{R}^{H \times W}$ is normalized via spatial softmax, yielding a probability distribution over image locations. The attended feature vector $\mathbf{f}_k = \sum_{i,j} \mathbf{A}_k(i,j) \cdot \mathbf{F}(i,j)$ is then concatenated with a $2$-dimensional encoding of the agent's head orientation (represented as $[\cos \theta, \sin \theta]$) and passed through a compact MLP to produce the object's relative position $(x_k, y_k)$ in the agent's local coordinate frame. The inclusion of orientation information is critical, as the same visual observation corresponds to different relative positions depending on the agent's heading direction.

The \textit{Visibility Predictor} operates on global context, aggregating the spatial feature map via spatial average pooling to produce a holistic scene descriptor. This global representation is fused with the orientation encoding and processed by a lightweight MLP with layer normalization and dropout (rate $0.2$) to predict a per-object visibility score. The sigmoid activation enforces a probabilistic interpretation, allowing the network to express graded uncertainty about object presence. During inference, the predicted positions are masked by the visibility scores, effectively gating the position estimates to suppress hallucinated detections for occluded or out-of-view objects. This architectural separation of position estimation (which operates on spatially-localized features) and visibility prediction (which requires holistic scene understanding) reflects the computational principle of divide-and-conquer, enabling each sub-module to specialize on complementary aspects of the detection task.

The network is trained using a multi-task loss function that jointly optimizes position accuracy and visibility classification. For position estimation, we employ a masked mean squared error (MSE) loss that is computed only for ground-truth visible objects, preventing the network from being penalized for arbitrary predictions on occluded targets. For visibility prediction, we use binary cross-entropy (BCE) loss. The composite objective is formulated as:
\begin{equation}
\mathcal{L}_{\text{detector}} = \frac{1}{N_{\text{obj}}} \sum_{k=1}^{N_{\text{obj}}} \left[ v_k^{\text{gt}} \cdot \|\mathbf{p}_k - \mathbf{p}_k^{\text{gt}}\|_2^2 + \lambda_{\text{vis}} \cdot \text{BCE}(\hat{v}_k, v_k^{\text{gt}}) \right],
\end{equation}
where $v_k^{\text{gt}} \in \{0,1\}$ is the ground-truth visibility label, $\mathbf{p}_k \in \mathbbm{R}^2$ is the predicted position, $\mathbf{p}_k^{\text{gt}}$ is the ground-truth position, $\hat{v}_k$ is the predicted visibility score, and $\lambda_{\text{vis}} = 1.0$ balances the two objectives. This formulation ensures that the position predictor receives meaningful gradients only for visible objects, while the visibility predictor learns a robust binary classifier across all object states. The use of a pre-trained ResNet-18 backbone (trained on ImageNet) provides strong initialization for low-level feature extraction, significantly accelerating convergence and improving sample efficiency in the target detection task. The entire module is trained with the Adam optimizer using a learning rate of $1 \times 10^{-4}$ and weight decay of $1 \times 10^{-5}$ for regularization.

\subsubsection{BEV prediction network}
\begin{table}[thbp]
\centering
\caption{\textbf{Encoder architecture of BEV prediction.} $B$ means the batch size.}
\setlength{\tabcolsep}{13pt}
\label{tab:bev_encoder}
\begin{tabular}{lllll}
\toprule
\textbf{Layer} & \textbf{Input Size} & \textbf{Output Size} & \textbf{Parameters} & \textbf{Activation} \\
\midrule
\multicolumn{5}{c}{Visual Feature Embedding} \\
\midrule
Conv2d-1 & $3 \times 64 \times 64$ & $32 \times 32 \times 32$ & $k=3, s=2, p=1$ & ReLU \\
BatchNorm2d-1 & $32 \times 32 \times 32$ & $32 \times 32 \times 32$ & - & - \\
Conv2d-2 & $32 \times 32 \times 32$ & $64 \times 16 \times 16$ & $k=3, s=2, p=1$ & ReLU \\
BatchNorm2d-2 & $64 \times 16 \times 16$ & $64 \times 16 \times 16$ & - & - \\
Conv2d-3 & $64 \times 16 \times 16$ & $128 \times 8 \times 8$ & $k=3, s=2, p=1$ & ReLU \\
BatchNorm2d-3 & $128 \times 8 \times 8$ & $128 \times 8 \times 8$ & - & - \\
Conv2d-4 & $128 \times 8 \times 8$ & $256 \times 4 \times 4$ & $k=3, s=2, p=1$ & ReLU \\
BatchNorm2d-4 & $256 \times 4 \times 4$ & $256 \times 4 \times 4$ & - & - \\
\midrule
\multicolumn{5}{c}{Transformer Processing} \\
\midrule
Reshape & $256 \times 4 \times 4$ & $4 \times (B \times 4) \times 256$ & - & - \\
Position Embedding & $4 \times (B \times 4) \times 256$ & $4 \times (B \times 4) \times 256$ & $d_\text{model}=256$ & - \\
Transformer Encoder & $4 \times (B \times 4) \times 256$ & $4 \times (B \times 4) \times 256$ & 3 layers, 8 heads & - \\
\bottomrule
\end{tabular}
\end{table}

\begin{table}[ht]
\centering
\caption{\textbf{Decoder architecture of BEV prediction.} $B$ means the batch size.}
\setlength{\tabcolsep}{10pt}
\label{tab:bev_decoder}
\begin{tabular}{lllll}
\toprule
\textbf{Layer} & \textbf{Input Size} & \textbf{Output Size} & \textbf{Parameters} & \textbf{Activation} \\
\midrule
\multicolumn{5}{c}{Transformer Decoder} \\
\midrule
Query Init & - & $64 \times (B \times 4) \times 256$ & $8 \times 8$ spatial & - \\
PositionalEncoding & $64 \times (B \times 4) \times 256$ & $64 \times (B \times 4) \times 256$ & $d_\text{model}=256$ & - \\
TransformerDecoder & Memory + Query & $(B \times 4) \times 256 \times 8 \times 8$ & 3 layers, 8 heads & - \\
Linear Projection & $(B \times 4) \times 256$ & $256 \times 256$ & - & - \\
\midrule
\multicolumn{5}{c}{Upsampling Network} \\
\midrule
ConvTranspose2d-1 & $256 \times 8 \times 8$ & $128 \times 16 \times 16$ & $k=4, s=2, p=1$ & ReLU \\
BatchNorm2d-1 & $128 \times 16 \times 16$ & $128 \times 16 \times 16$ & - & - \\
ConvTranspose2d-2 & $128 \times 16 \times 16$ & $64 \times 32 \times 32$ & $k=4, s=2, p=1$ & ReLU \\
BatchNorm2d-2 & $64 \times 32 \times 32$ & $64 \times 32 \times 32$ & - & - \\
ConvTranspose2d-3 & $64 \times 32 \times 32$ & $32 \times 64 \times 64$ & $k=4, s=2, p=1$ & ReLU \\
BatchNorm2d-3 & $32 \times 64 \times 64$ & $32 \times 64 \times 64$ & - & - \\
ConvTranspose2d-4 & $32 \times 64 \times 64$ & $16 \times 128 \times 128$ & $k=4, s=2, p=1$ & ReLU \\
BatchNorm2d-4 & $16 \times 128 \times 128$ & $16 \times 128 \times 128$ & - & - \\
ConvTranspose2d-5 & $16 \times 128 \times 128$ & $8 \times 256 \times 256$ & $k=4, s=2, p=1$ & ReLU \\
BatchNorm2d-5 & $8 \times 256 \times 256$ & $8 \times 256 \times 256$ & - & - \\
Conv2d-Final & $8 \times 256 \times 256$ & $4 \times 256 \times 256$ & $k=3, p=1$ & - \\
\midrule
\multicolumn{5}{c}{Output Processing} \\
\midrule
RGB Channels & $3 \times 256 \times 256$ & $3 \times 250 \times 250$ & Bilinear Resize & Sigmoid \\
Alpha Channel & $1 \times 256 \times 256$ & $1 \times 250 \times 250$ & Bilinear Resize & Sigmoid \\
\bottomrule
\end{tabular}
\end{table}

\paragraph{Encoder.}
The encoder's primary function is to transform a raw, first-person-view image into a compact and contextually-rich latent representation suitable for cross-view translation. This transformation must overcome severe geometric distortions inherent in ground-level perspectives, where depth cues are ambiguous and distant objects appear compressed. To address this challenge, our encoder employs a two-stage processing pipeline that combines hierarchical convolutional feature extraction with global contextual reasoning via transformer-based attention.

The first stage employs a standard convolutional neural network (CNN) backbone to extract a hierarchy of spatial features from an input image. This initial processing stage, detailed in \textbf{Fig.~\ref{fig:figS-bevpn}} and \textbf{Table \ref{tab:bev_encoder}}, consists of four progressive downsampling blocks, each halving the spatial resolution while doubling the channel capacity. This pyramid structure effectively captures low-level visual world patterns such as edges, textures, and local shapes in early layers, while deeper layers encode increasingly abstract semantic features such as object boundaries and surface orientations. Each convolutional layer is followed by batch normalization to stabilize training dynamics and ReLU activation to introduce non-linearity.

The critical architectural innovation lies in the subsequent processing stage, where we treat the vertical scanlines of the resulting $256 \times 4 \times 4$ feature map as a sequence. This reformulation enables the application of a Transformer encoder, whose self-attention mechanism can integrate information across the entire vertical axis of the image. This global context is essential for accurately inferring depth relationships and disambiguating occluded regions—tasks at which purely convolutional architectures often struggle due to their limited receptive fields. The transformer processes 4 vertical scanlines, each represented as a sequence of 4 spatial tokens with 256-dimensional feature vectors. Sinusoidal positional encodings are added to preserve spatial ordering, and 3 transformer layers with 8 attention heads refine the representation by modeling long-range dependencies. This encoder design enables the network to effectively bridge the perspective gap between egocentric observation and allocentric representation.

\paragraph{Decoder.}
The decoder is tasked with projecting the abstract latent code generated by the encoder into a structured, allocentric bird's-eye-view (BEV) map. This inverse mapping must reconstruct fine-grained spatial details from a heavily compressed representation, requiring both spatial upsampling and semantic refinement. Our decoder architecture, specified in  \textbf{Fig.~\ref{fig:figS-bevpn}} and \textbf{Table \ref{tab:bev_decoder}}, implements this transformation through a two-stage process combining transformer-based cross-attention with hierarchical convolutional upsampling.

The first stage employs a Transformer decoder to establish spatial correspondences between the encoded visual features and the target BEV grid. We initialize a set of learnable query vectors arranged in an $8 \times 8$ spatial grid, with each query corresponding to a specific location in the coarse-resolution BEV map. These queries are augmented with sinusoidal positional encodings to inject spatial awareness. The transformer decoder then refines these queries through 3 layers of cross-attention, where each query attends to the encoder's output memory to extract relevant visual evidence for its corresponding spatial location. This attention mechanism effectively implements a learned, content-dependent resampling operation that warps the egocentric visual features into the allocentric frame. The output is a $256$-channel feature map at $8 \times 8$ resolution, representing a semantically-rich but spatially-coarse BEV representation.

The second stage progressively upsamples this coarse representation to the final $250 \times 250$ resolution through a cascade of five transposed convolutional layers. Each upsampling block doubles the spatial resolution while halving the channel depth, gradually translating abstract semantic features into concrete pixel-level predictions. Batch normalization and ReLU activations are applied after each layer to maintain stable gradients and non-linear expressiveness. The final $4$-channel output is partitioned into RGB appearance channels and a single alpha channel representing occupancy probability. Both outputs are resized to $250 \times 250$ via bilinear interpolation and passed through sigmoid activations to enforce valid probability ranges. This hierarchical decoder design ensures that the network can reconstruct both the geometric layout (via the alpha channel) and the visual appearance (via RGB channels) of the environment from a unified latent representation.

\begin{table}[t]
\centering
\caption{\textbf{Details of transformer component in BEV prediction.}}
\label{tab:transformer_specs}
\setlength{\tabcolsep}{40pt}
\begin{tabular}{lll}
\toprule
\textbf{Component} & \textbf{Parameter} & \textbf{Value} \\
\midrule
\multirow{5}{*}{Transformer Encoder} & Model Dimension ($d_\text{model}$) & 256 \\
 & Number of Heads & 8 \\
 & Feed-forward Dimension & 2048 \\
 & Number of Layers & 3 \\
 & Dropout Rate & 0.1 \\
\midrule
\multirow{5}{*}{Transformer Decoder} & Model Dimension ($d_\text{model}$) & 256 \\
 & Number of Heads & 8 \\
 & Feed-forward Dimension & 2048 \\
 & Number of Layers & 3 \\
 & Dropout Rate & 0.1 \\
\midrule
\multirow{2}{*}{Positional Encoding} & Max Sequence Length & 5000 \\
 & Encoding Type & Sinusoidal \\
\bottomrule
\end{tabular}
\end{table}

\paragraph{Transformer component.}
The capacity and performance of the transformer-based components are critical to the model's ability to reason about long-range spatial dependencies. The specific hyperparameters, detailed in \textbf{Table \ref{tab:transformer_specs}}, were chosen to balance representational power with computational efficiency. A model dimension ($d_\text{model}$) of 256 provides a sufficiently high-dimensional space for embedding complex visual features. The use of 8 attention heads allows the model to simultaneously focus on different aspects of the visual input—for instance, attending to distant landmarks with one head while focusing on nearby wall textures with another. The 3-layer depth for both the encoder and decoder was empirically determined to be deep enough to learn complex cross-view transformations without incurring excessive computational cost or overfitting. These parameters collectively equip the module with the necessary capacity to learn the non-trivial mapping from a narrow, first-person perspective to a comprehensive, top-down allocentric map.

\begin{table}[t]
\centering
\caption{\textbf{Details of loss function in BEV prediction.}}
\label{tab:loss_weights}
\setlength{\tabcolsep}{25pt}
\begin{tabular}{lll}
\toprule
\textbf{Loss Component} & \textbf{Mathematical Form} & \textbf{Weight} \\
\midrule
Alpha Loss & $\mathcal{L}_{\text{occ}} = -\sum_{i,j} [y_{ij} \log(\hat{y}_{ij}) + (1-y_{ij}) \log(1-\hat{y}_{ij})]$ & $w_\alpha = 1.0$ \\
RGB Loss & $\mathcal{L}_\text{rgb} = \frac{1}{|M|} \sum_{(i,j) \in M} ||\mathbf{c}_{ij} - \hat{\mathbf{c}}_{ij}||_2^2$ & $w_\text{rgb} = 0.5$ \\
Smoothness Loss & $\mathcal{L}_\text{smooth} = \frac{1}{2}(\mathbbm{E}[|\nabla_x \alpha|] + \mathbbm{E}[|\nabla_y \alpha|])$ & $w_\text{smooth} = 0.1$ \\
\bottomrule
\end{tabular}
\end{table}

\paragraph{Loss function details.}
To guide the learning process toward physically plausible world models, the module is optimized by minimizing a composite loss function, $\mathcal{L}_\text{BEV}$, which holistically evaluates the quality of the BEV prediction. This objective function is a weighted sum of three distinct terms, each imposing a different physical prior on the model's predictions. The primary term, the occupancy loss ($\mathcal{L}_\text{occ}$), uses binary cross-entropy to enforce geometric consistency, compelling the network to make a clear distinction between navigable space and solid obstacles. The appearance loss ($\mathcal{L}_\text{rgb}$) employs a masked mean-squared error to ensure photorealistic accuracy, forcing the model to predict the correct surface textures but only within regions identified as navigable. Finally, a smoothness loss ($\mathcal{L}_\text{smooth}$) acts as a regularizer, penalizing sharp, unnatural gradients in the predicted occupancy map and incorporating the prior that physical environments are generally continuous. The mathematical formulation and weighting of each component are detailed in \textbf{Table \ref{tab:loss_weights}}.

\begin{table}[t]
\centering
\caption{\textbf{Communication encoder architecture.} This network implements the compression stage of the VIB, mapping a spatial map to a stochastic latent variable $z$.}
\label{tab:comm_encoder}
\setlength{\tabcolsep}{15pt}
\begin{tabular}{lllll}
\toprule
\textbf{Layer} & \textbf{Input Size} & \textbf{Output Size} & \textbf{Parameters} & \textbf{Activation} \\
\midrule
\multicolumn{5}{c}{Spatial Memory Input} \\
\midrule
Input Map & $1 \times 64 \times 64$ & $1 \times 64 \times 64$ & - & - \\
\midrule
\multicolumn{5}{c}{Convolutional Downsampling} \\
\midrule
Conv2d-1 & $1 \times 64 \times 64$ & $32 \times 32 \times 32$ & $k=4, s=2, p=1$ & ReLU \\
BatchNorm2d-1 & $32 \times 32 \times 32$ & $32 \times 32 \times 32$ & - & - \\
Conv2d-2 & $32 \times 32 \times 32$ & $64 \times 16 \times 16$ & $k=4, s=2, p=1$ & ReLU \\
BatchNorm2d-2 & $64 \times 16 \times 16$ & $64 \times 16 \times 16$ & - & - \\
Conv2d-3 & $64 \times 16 \times 16$ & $128 \times 8 \times 8$ & $k=4, s=2, p=1$ & ReLU \\
BatchNorm2d-3 & $128 \times 8 \times 8$ & $128 \times 8 \times 8$ & - & - \\
Conv2d-4 & $128 \times 8 \times 8$ & $256 \times 4 \times 4$ & $k=4, s=2, p=1$ & ReLU \\
BatchNorm2d-4 & $256 \times 4 \times 4$ & $256 \times 4 \times 4$ & - & - \\
\midrule
\multicolumn{5}{c}{VAE Latent Space Projection} \\
\midrule
Flatten & $256 \times 4 \times 4$ & 4096 & - & - \\
Mean Projection ($\mu$) & 4096 & $z_{dim}$ & Linear(4096, $z_{dim}$) & - \\
LogVar Projection ($\log\sigma^2$) & 4096 & $z_{dim}$ & Linear(4096, $z_{dim}$) & - \\
Reparameterization & $z_{dim}$ & $z_{dim}$ & $z = \mu + \epsilon \cdot \sigma$ & - \\
\bottomrule
\end{tabular}
\end{table}

\subsection{Emergent communication mechanism architecture}
\label{sec: Emergent Communication Mechanism}

To operationalize the principle of Social Predictive Coding, we design a communication architecture that directly implements the variational information bottleneck (VIB) framework~\cite{alemi2017deep_supp,tishby2000information_supp}. The central hypothesis is that an efficient communication mechanism need not be manually designed but can emerge when agents are optimized to transmit only the information that is maximally reductive of their partner's future uncertainty. Consequently, the network architecture is not merely a tool for data compression, but  a principled mechanism for learning a compact, structured, and task-relevant symbolic language from the ground up. This process is built upon a neural substrate capable of generating a rich, unified social representation, which serves as the input to the communication module. We first detail this substrate, followed by the communication architecture itself.

\subsubsection{Social representation substrate: Emergence of social place cells}
Before an agent can decide what information to transmit, it must first form a comprehensive internal representation of the whole multi-agent system~\cite{eccles2019biases_supp}. This state, denoted as $S_{i,t}$ in our method, serves as the foundation for the entire communication mechanism. To this end, we design a specialized neural architecture, the \texttt{Social Place Coding}, to learn this rich social representation, as shown in the main text (\textbf{Fig. \ref{fig:social_place_code}a}).

The network's backbone is a dual-stream path integration module built upon a single recurrent LSTM core. This LSTM concurrently processes egocentric motion inputs (linear and angular velocities) from both the self-agent and its partner. A key design choice is the use of an asymmetric input representation: the self-agent's velocity vector is concatenated with a learned \texttt{ego\_token}, while the partner's velocity is concatenated with a zero vector. This allows the shared LSTM to distinguish between self and other motion streams while processing them with the same set of weights, encouraging the development of a unified representational space.

The hidden states from both processing streams are then fused via element-wise addition to form a \texttt{joint\_representation}. This unified state is then passed through a bottleneck layer. This shared latent representation is compelled to functionally specialize under a multi-faceted predictive objective. The network is trained not only to predict its own future location (place and head-direction cell activations) and its partner's future location, but also, critically, the future Euclidean distance between them via a dedicated \texttt{relational\_head}  (the relative positioning). This compound predictive pressure ensures that the \texttt{joint\_representation} encodes not just individual trajectories but also the dynamic spatial relationship between the agents. It is this relational predictive task that catalyzes the emergence of specialized neural populations, including social place cells (SPCs) that are selectively tuned to the partner's location. This emergent social representation provides the rich, disentangled state $S_{i,t}$ that is subsequently fed into the VIB communication encoder, forming a crucial bridge between social cognition and emergent communication.

\subsubsection{Encoder architecture: Implementing the VIB compression term}

The encoder's role is to transform the sender's high-dimensional state, an occupancy map $S_{i,t}$ of size $64 \times 64$, into a compressed latent message $z$. This directly corresponds to learning the stochastic mapping $q_{\varphi}(z \mid S_{i,t})$ in our VIB formulation. Our choice of a deep convolutional structure is critical for this task. Unlike fully-connected networks that would discard spatial topology, convolutional layers impose a strong and relevant inductive bias—namely, locality and translation equivariance. This ensures that the learned features capture the geometric nature of the agent's environment.

The architecture, detailed in \textbf{Table~\ref{tab:comm_encoder}}, employs a symmetric downsampling hierarchy. Each block halves the spatial resolution while doubling the channel capacity, creating a pyramid of progressively more abstract and semantically rich feature maps. This hierarchical processing allows the network to capture not just fine-grained local details (e.g., narrow corridors) in its initial layers, but also the global layout and topological structure of the environment in its deeper layers. This multi-scale feature extraction is essential for generating a message that is both comprehensive and compact.

The information bottleneck itself is realized at the VAE's latent space. The flattened $4096$-dimensional feature vector is projected onto the parameters of a diagonal Gaussian distribution, $\mu$ and $\log\sigma^2$. The stochastic message $z$ is then sampled using the reparameterization trick ($z = \mu + \epsilon \sigma, \epsilon \sim \mathcal{N}(0,I)$), which permits gradient flow during backpropagation. This stochastic encoding is the key to rate-limiting the communication channel. The KL-divergence term in the loss function, $\mathcal{L}_{\mathrm{KL}}$, penalizes the encoded distribution $q_{\varphi}(z \mid S_{i,t})$ for deviating from the uninformative prior $p(z)$. This pressure constrains the mutual information $I(S_{i,t}; z)$, forcing the encoder to discard non-essential information and retain only the most salient features of the input map.

\subsubsection{Decoder architecture: Implementing the VIB predictive utility term}

The decoder's function is to quantify the predictive utility of the message $z$. It operationalizes the predictive model $p_{\vartheta}(O_{j,t+1} \mid z, S_{j,t})$ from our VIB framework, where its primary goal is to reconstruct the receiving agent's future state from the compressed message. The architecture, detailed in \textbf{Table~\ref{tab:comm_decoder}}, symmetrically mirrors the encoder. It first projects the latent code $z$ back to a high-dimensional feature space ($4096$ dimensions) and reshapes it into a spatial tensor ($256 \times 4 \times 4$).

A sequence of transposed convolutional layers then hierarchically upsamples this representation, progressively doubling the spatial resolution while halving the channel depth. This process effectively inverts the abstraction performed by the encoder, translating the compact, semantic message back into a concrete spatial map. The final layer applies a sigmoid activation function to produce pixel-wise occupancy probabilities, yielding a reconstructed map $\hat{y}$ that represents the agent's prediction of its partner's view. The quality of this reconstruction serves as the measure of the message's utility.

\begin{table}[t]
\centering
\caption{\textbf{Communication decoder architecture.} This network instantiates the predictive component of the VIB, reconstructing a spatial map from the latent message $z$.}
\label{tab:comm_decoder}
\setlength{\tabcolsep}{18pt}
\begin{tabular}{lllll}
\toprule
\textbf{Layer} & \textbf{Input Size} & \textbf{Output Size} & \textbf{Parameters} & \textbf{Activation} \\
\midrule
\multicolumn{5}{c}{Latent Space Processing} \\
\midrule
Latent Input & $z_\text{dim}$ & $z_\text{dim}$ & - & - \\
Projection & $z_\text{dim}$ & 4096 & Linear($z_\text{dim}$, 4096) & ReLU \\
Reshape & 4096 & $256 \times 4 \times 4$ & - & - \\
\midrule
\multicolumn{5}{c}{Transposed Convolutional Upsampling} \\
\midrule
ConvTranspose2d-1 & $256 \times 4 \times 4$ & $128 \times 8 \times 8$ & $k=4, s=2, p=1$ & ReLU \\
BatchNorm2d-1 & $128 \times 8 \times 8$ & $128 \times 8 \times 8$ & - & - \\
ConvTranspose2d-2 & $128 \times 8 \times 8$ & $64 \times 16 \times 16$ & $k=4, s=2, p=1$ & ReLU \\
BatchNorm2d-2 & $64 \times 16 \times 16$ & $64 \times 16 \times 16$ & - & - \\
ConvTranspose2d-3 & $64 \times 16 \times 16$ & $32 \times 32 \times 32$ & $k=4, s=2, p=1$ & ReLU \\
BatchNorm2d-3 & $32 \times 32 \times 32$ & $32 \times 32 \times 32$ & - & - \\
ConvTranspose2d-4 & $32 \times 32 \times 32$ & $1 \times 64 \times 64$ & $k=4, s=2, p=1$ & Sigmoid \\
\midrule
\multicolumn{5}{c}{Spatial Map Output} \\
\midrule
Reconstructed Map & $1 \times 64 \times 64$ & $1 \times 64 \times 64$ & Continuous $[0,1]$ & - \\
\bottomrule
\end{tabular}
\end{table}

\begin{table}[t]
\centering
\caption{\textbf{Training configuration and performance metrics of communication mechanism.}}
\setlength{\tabcolsep}{40pt}
\label{tab:comm_training}
\begin{tabular}{lll}
\toprule
\textbf{Parameter Category} & \textbf{Parameter} & \textbf{Value} \\
\midrule
\multirow{4}{*}{Architecture} & Input Map Size & $64 \times 64$ pixels \\
& Latent Dimensions ($z_\text{dim}$) & 128 (Optimal) \\
& Compression Ratio & $32{:}1$ \\
& Reconstruction Error & $< 4.1\%$ \\
\midrule
\multirow{4}{*}{Training} & Optimizer & Adam \\
& Batch Size & 32 \\
& Learning Rate & $1 \times 10^{-3}$ \\
& Training Epochs & 50 \\
\midrule
\multirow{2}{*}{VIB Objective} & Reconstruction Weight & 1.0 \\
& KL Divergence Weight ($\beta$) & 1.0 \\
\bottomrule
\end{tabular}
\end{table}

\subsubsection{VIB objective: Driving the emergence of efficient communication}

The encoder and decoder are trained jointly by minimizing the VIB objective function, which is functionally equivalent to the VAE's evidence lower bound (ELBO). This loss function elegantly captures the fundamental rate-distortion trade-off at the heart of our theory:
\begin{equation}
\mathcal{L}_{\mathrm{VIB}}=\underbrace{\mathcal{L}_{\mathrm{reconstruction}}}_{\text{Predictive Utility}} + \underbrace{\beta \cdot \mathcal{L}_{\mathrm{KL}}.}_{\text{Communication Cost}}
\end{equation}

\paragraph{Reconstruction loss: Maximizing predictive utility.}
The reconstruction loss, a pixel-wise binary cross-entropy, drives the decoder to produce accurate predictions, thereby rewarding messages that contain high predictive utility:
\begin{equation}
\mathcal{L}_{\mathrm{reconstruction}} = -\sum_{i,j}\!\left[y_{ij}\log(\hat y_{ij})+(1-y_{ij})\log(1-\hat y_{ij})\right].
\end{equation}
This term directly instantiates the distortion component of the VIB framework, ensuring that compressed messages $z$ retain sufficient information to enable accurate reconstruction of the receiver's spatial map.

\paragraph{KL divergence loss: Enforcing communication efficiency.}
Simultaneously, the KL-divergence loss regularizes the encoder, penalizing deviation from the simple Gaussian prior and thus minimizing the information capacity of the communication channel. This term operationalizes the rate constraint in the VIB formulation. Specifically, the KL term measures the information-theoretic divergence between the learned variational posterior $q_{\varphi}(z \mid S_{i,t})$ produced by the encoder and the standard Gaussian prior $p(z) = \mathcal{N}(\mathbf{0}, \mathbf{I})$. By definition, the Kullback-Leibler divergence is:
\begin{equation}
\label{eq:kl_definition}
\mathcal{L}_{\mathrm{KL}} = D_\mathrm{KL}\left(q_{\varphi}(z|S_{i,t}) \| p(z)\right) = \mathbb{E}_{q_{\varphi}(z|S_{i,t})}\left[\log q_{\varphi}(z|S_{i,t}) - \log p(z)\right].
\end{equation}
This expectation quantifies the average ``surprise'' or information cost of encoding the agent's state $S_{i,t}$ relative to the uninformative prior. Minimizing this divergence ensures that the latent code $z$ remains statistically indistinguishable from the prior unless the information is essential for prediction.

\textit{Closed-form solution under Gaussian assumptions.} Our encoder architecture parametrizes the variational posterior as a diagonal Gaussian distribution:
\begin{equation}
q_{\varphi}(z|S_{i,t}) = \mathcal{N}\left(\boldsymbol{\mu}(S_{i,t}), \, \text{diag}(\boldsymbol{\sigma}^2(S_{i,t}))\right),
\end{equation}
where $\boldsymbol{\mu} = [\mu_1, \ldots, \mu_{z_\text{dim}}]^\top$ and $\boldsymbol{\sigma}^2 = [\sigma_1^2, \ldots, \sigma_{z_\text{dim}}^2]^\top$ are outputs of the encoder's mean and log-variance projection layers, respectively. The diagonal covariance structure reflects the architectural assumption that latent dimensions are conditionally independent given the input state. Under this Gaussian parametrization with a standard normal prior $p(z) = \mathcal{N}(\mathbf{0}, \mathbf{I})$, the KL divergence factorizes across dimensions:
\begin{equation}
D_\mathrm{KL}\left(q_{\varphi}(z|S_{i,t}) \| p(z)\right) = \sum_{k=1}^{z_\text{dim}} D_\mathrm{KL}\left(\mathcal{N}(\mu_k, \sigma_k^2) \| \mathcal{N}(0, 1)\right).
\end{equation}

For any two univariate Gaussian distributions $\mathcal{N}(\mu_k, \sigma_k^2)$ and $\mathcal{N}(0, 1)$, the KL divergence has the well-known closed form:
\begin{equation}
\label{eq:kl_gaussian}
D_\mathrm{KL}\left(\mathcal{N}(\mu_k, \sigma_k^2) \| \mathcal{N}(0, 1)\right) = \frac{1}{2}\left(\sigma_k^2 + \mu_k^2 - 1 - \log\sigma_k^2\right).
\end{equation}
This expression is derived by evaluating the expectation in \textbf{Eq.~\eqref{eq:kl_definition}} under the Gaussian density functions and simplifying the resulting integral.

Summing over all latent dimensions, we obtain the final training objective:
\begin{equation}
\label{eq:kl_final}
\mathcal{L}_{\mathrm{KL}} = \frac{1}{2}\sum_{k=1}^{z_\text{dim}}\left(\mu_k^2 + \sigma_k^2 - 1 - \log\sigma_k^2\right),
\end{equation}
where ($\mu_k^2$) penalizes the encoder for shifting the posterior distribution's mean away from zero. This term encourages the latent code to be centered, preventing the encoder from arbitrarily offsetting the representation space. The variance regularization ($\sigma_k^2 - 1$) penalizes deviation of the posterior variance from unity. When $\sigma_k^2 > 1$, the distribution is overly diffuse, indicating the encoder is uncertain; when $\sigma_k^2 < 1$, the distribution is overly concentrated, indicating overconfidence. This term encourages calibrated uncertainty. Besides, variance collapse prevention ($-\log\sigma_k^2$) acts as a negative log-determinant term that becomes large (highly penalizing) as $\sigma_k^2 \to 0$. This prevents the posterior from collapsing to a delta function, which would eliminate stochasticity and reduce the model to deterministic encoding. The stochastic bottleneck is essential for learning representations robust to input variations.

% Together, these three components enforce a soft constraint on the mutual information $I(S_{i,t}; z)$, compelling the latent space to remain close to the uninformative prior while encoding only the minimal information necessary for accurate prediction. This pressure catalyzes the emergence of abstract, compressed symbolic representations—the agent learns not to transmit raw sensory data, but rather discovers a reduced code that captures task-relevant structure.

\paragraph{The rate-distortion trade-off.}

The hyperparameter $\beta$ serves as the Lagrange multiplier from the original information bottleneck formulation, allowing explicit control over this trade-off. A larger $\beta$ intensifies the pressure to compress, forcing the system to discover a more abstract and efficient communication mechanism. It is this pressure that catalyzes the emergence of a specialized symbolic system. When the channel capacity is limited, the agents learn that the most valuable, uncertainty-reducing information often pertains to the locations and trajectories of other agents. The VIB objective thus guides the latent space to develop a disentangled structure where specific dimensions become selectively tuned to this critical social-spatial information. This process explains the spontaneous emergence of social place cell (SPC)-like representations within the communication code—not as a pre-programmed feature, but as the optimal solution to the problem of collaborative prediction under bandwidth constraints. As shown in our training configuration (\textbf{Table~\ref{tab:comm_training}}), this architecture achieves a 32:1 compression ratio with under 5\% reconstruction error, demonstrating the efficacy of this emergent communication.

\subsection{HRL-ICM framework}
\label{sec: HRL-ICM Framework}

To operationalize the strategic exploration policy described in the main text, we implement a hierarchical reinforcement learning framework augmented with an intrinsic curiosity module (HRL-ICM). This architecture decomposes the navigation task into strategic goal selection (handled by a learned meta-controller) and tactical path execution (delegated to a deterministic A* planner). The framework is trained using multi-agent proximal policy optimization (MAPPO), a robust variant of the PPO algorithm adapted for cooperative multi-agent~\cite{yu2022multi_supp}. By separating strategic and tactical reasoning, the system can learn long-horizon exploration strategies without the burden of low-level motor control, enabling efficient credit assignment and scalable coordination across multiple agents.

Critically, the framework's ability to coordinate exploration across agents depends on accurate estimation of inter-agent spatial relationships. This is achieved through the social place cell (SPC) module described in the main text, which provides the essential relational geometry required for the ICM to compute coordination rewards and for the communication gating policy to make informed transmission decisions. We first describe the integration of this social spatial representation module before detailing the meta-controller and ICM components.

\subsubsection{Social place cell module for partner state estimation}\label{sec: Social Place Cell}

The social place cell module serves as the perceptual foundation for multi-agent coordination by continuously estimating partner locations and computing inter-agent distances. This module extends the grid-cell-based path integration architecture described in {\textbf{Method \ref{sec: Biologically-Inspired Spatial Representation Networks}}} to the social domain through a dual-stream processing design.

The architecture consists of two parallel LSTM encoders that concurrently process motion information from both the observing agent (self) and its partner. The self-stream receives the agent's proprioceptive velocity commands $\mathbf{v}_{\text{self},t} = (v_x, v_y, \omega)$ as direct sensory input. The partner-stream processes estimated partner velocities $\hat{\mathbf{v}}_{\text{partner},t}$, which are inferred from visual observations of the partner's motion across multiple frames. Specifically, when a partner is visible within the agent's field of view, its displacement and orientation change over a short temporal window (typically 3-5 frames) are used to estimate its instantaneous translational and angular velocities. This visual motion estimation provides a noisy but informative signal about the partner's navigation state.

Both streams are initialized at episode start using the respective agents' known starting poses, encoded through ensembles of virtual place cells and head-direction cells identical to those used in the individual path integration module (detailed in \textbf{Method \ref{sec: Biologically-Inspired Spatial Representation Networks}}). The two LSTM hidden states, $\mathbf{h}_{\text{self},t}$ and $\mathbf{h}_{\text{partner},t}$, are then fused via element-wise summation to form a unified joint representation $\mathbf{h}_{\text{joint},t} = \mathbf{h}_{\text{self},t} + \mathbf{h}_{\text{partner},t}$. This shared representation is passed through a bottleneck layer (256 units, 50\% dropout) that enforces a compressed, information-efficient encoding of the two-agent system state.

The network is trained under a multi-task predictive objective with three supervision signals, each computed at the final timestep $T$ of a trajectory segment:
\begin{equation}
\mathcal{L}_{\text{SPC}} = \mathcal{L}_{\text{self}} + \mathcal{L}_{\text{partner}} + w_{\text{dist}} \mathcal{L}_{\text{distance}},
\end{equation}
where $\mathcal{L}_{\text{self}}$ and $\mathcal{L}_{\text{partner}}$ are KL divergences between predicted and target place/head-direction cell activations for the self and partner agents, respectively (identical in form to the individual path integration loss), and $\mathcal{L}_{\text{distance}}$ is a mean-squared error on the Euclidean distance between agents, $\|\mathbf{r}_{\text{self},T} - \mathbf{r}_{\text{partner},T}\|_2$. The distance prediction is implemented via a dedicated regression head (linear layer) that projects the bottleneck representation to a scalar distance estimate. We set $w_{\text{dist}} = 1.0$ to balance the three objectives.

This compound predictive objective compels the bottleneck representation to develop functionally specialized subpopulations. As demonstrated in the main text (\textbf{Fig.~\ref{fig:social_place_code}}), analysis of the learned representations reveals distinct neuron types: pure place cells tuned exclusively to self-position, social place cells (SPCs) selective for partner position, and mixed-selectivity units encoding conjunctions of self- and partner-locations. Critically, a subset of units forms a population code for inter-agent distance, exhibiting graded tuning curves that tile the distance space from close proximity to far separation. This distance-tuned population is causally necessary for accurate distance estimation, as confirmed by targeted \textit{in-silico} lesion experiments (\textbf{Fig.~\ref{fig:social_place_code}e}).

For integration into the HRL-ICM framework, the SPC module operates continuously during exploration. At each timestep, the module outputs: (1) an updated estimate of the partner's position $\hat{\mathbf{r}}_{\text{partner},t}$ (decoded from the partner place cell activations), and (2) a predicted inter-agent distance $\hat{d}_t$ (from the relational regression head). These outputs are consumed by downstream components: the distance estimate $\hat{d}_t$ directly informs the ICM's coordination reward (detailed below), while the partner position estimate enables the communication gating policy to assess whether agents are within communication range. The SPC module thus closes the loop between perception and coordination, transforming visual observations of partners into structured spatial representations that guide strategic decision-making. The architecture specifications are provided in \textbf{Table~\ref{tab:spc_arch}}.

\begin{table}[t]
\centering
\caption{\textbf{Details of social place cell module.}}
\label{tab:spc_arch}
\setlength{\tabcolsep}{25pt}
\begin{tabular}{llll}
\toprule
\textbf{Component} & \textbf{Input Dim} & \textbf{Output Dim} & \textbf{Description} \\
\midrule
\multicolumn{4}{c}{Input Processing} \\
\midrule
Self Velocity & 3 & 3 & $(v_x, v_y, \omega)$ \\
Partner Velocity (est.) & 3 & 3 & Visual motion estimation \\
Initial Pose Encoding & 2+1 & 288 & Place (256) + HD (32) cells \\
\midrule
\multicolumn{4}{c}{Dual-Stream LSTM} \\
\midrule
Self LSTM & 3 & 128 & Hidden state $\mathbf{h}_{\text{self}}$ \\
Partner LSTM & 3 & 128 & Hidden state $\mathbf{h}_{\text{partner}}$ \\
Fusion (sum) & $128 \times 2$ & 128 & $\mathbf{h}_{\text{joint}}$ \\
\midrule
\multicolumn{4}{c}{Bottleneck \& Prediction Heads} \\
\midrule
Bottleneck Layer & 128 & 256 & Dropout 0.5 \\
Self Place Predictor & 256 & 256 & KL loss \\
Self HD Predictor & 256 & 32 & KL loss \\
Partner Place Predictor & 256 & 256 & KL loss \\
Partner HD Predictor & 256 & 32 & KL loss \\
Distance Regression & 256 & 1 & MSE loss \\
\bottomrule
\end{tabular}
\end{table}

\subsubsection{Meta-controller architecture}

The meta-controller is implemented as a shared actor-critic network that maps high-level spatial abstractions to goal selections. To construct a tractable state representation from the high-dimensional occupancy map, each agent first partitions its local $H \times W$ grid into a coarse $g \times g$ regional summary (typically $g=4$, yielding 16 macro-regions). For each region $k \in \{1, \ldots, g^2\}$, the agent computes three summary statistics: the \textit{exploration ratio} (proportion of cells with known occupancy), the \textit{walkability ratio} (proportion of known cells that are navigable), and a binary \textit{agent presence indicator}. These features are concatenated into a $3g^2$-dimensional state vector that serves as input to the policy network.

The network architecture follows a standard actor-critic design with shared feature extraction. Two fully-connected layers (each with 256 hidden units and ReLU activation) process the regional feature vector to produce a shared embedding. This embedding then branches into two specialized heads. The \textit{actor head} projects the embedding through an additional hidden layer (256 units, ReLU) before outputting a $g^2$-dimensional logit vector, which is normalized via softmax to yield a categorical distribution over candidate goal regions. The \textit{critic head}, operating in parallel, projects the shared embedding through its own hidden layer to produce a scalar state-value estimate. All network weights are initialized using orthogonal initialization with a gain of 0.01, a choice that promotes stable early-stage training by preventing gradient explosion. The architecture is formally specified in \textbf{Table~\ref{tab:meta_controller_arch}}.

\begin{table}[t]
\centering
\caption{\textbf{Details of meta-controller actor-critic.}}
\label{tab:meta_controller_arch}
\setlength{\tabcolsep}{37pt}
\begin{tabular}{llll}
\toprule
\textbf{Component} & \textbf{Input Dim} & \textbf{Output Dim} & \textbf{Activation} \\
\midrule
\multicolumn{4}{c}{Shared Feature Extraction} \\
\midrule
FC-1 & $48$ & $256$ & ReLU \\
FC-2 & $256$ & $256$ & ReLU \\
\midrule
\multicolumn{4}{c}{Actor Branch} \\
\midrule
Actor FC & $256$ & $256$ & ReLU \\
Action Logits & $256$ & $16$ & Softmax \\
\midrule
\multicolumn{4}{c}{Critic Branch} \\
\midrule
Critic FC & $256$ & $256$ & ReLU \\
Value Output & $256$ & $1$ & Linear \\
\bottomrule
\end{tabular}
\end{table}

\subsubsection{Intrinsic curiosity module}

The intrinsic curiosity module (ICM) translates the abstract principle of uncertainty-driven exploration into a concrete, dense reward signal that guides the meta-controller's learning~\cite{sun2025curiosity_supp,jaderberg2016reinforcement_supp}. Rather than relying on a learned forward dynamics model, which can be sample-inefficient in high-dimensional discrete spaces, our ICM directly estimates epistemic uncertainty by analyzing the geometry of the known-unknown boundary in the agent's local map, augmented with social spatial information from the SPC module. The module generates a composite intrinsic reward $r_t^{\text{int}}$ as a weighted sum of three interpretable components:
\begin{equation}
r_t^{\text{int}} = w_{\text{curiosity}} \, r_{\text{curiosity}} + w_{\text{coord}} \, r_{\text{coord}} + w_{\text{explore}} \, r_{\text{explore}},
\end{equation}
where $(w_{\text{curiosity}}, w_{\text{coord}}, w_{\text{explore}}) = (1.0, 0.5, 0.3)$ are fixed hyperparameters that balance the contribution of each term.
The \textit{curiosity reward} $r_{\text{curiosity}}$ encourages agents to select goal regions that border unknown territory. It computes a spatial ``curiosity map'' by identifying frontier pixels—known navigable cells adjacent to unexplored areas—and weighting them by their proximity to the agent and the local density of unknown neighbors. The reward is then the normalized curiosity value of the selected goal region, effectively incentivizing movement toward the boundary of the agent's knowledge.

The \textit{coordination reward} $r_{\text{coord}}$ promotes spatial division of labor by leveraging the inter-agent distance estimates provided by the SPC module. This component discourages redundant exploration by rewarding goal selections that maintain appropriate separation from teammates. For each agent $i$, the SPC module continuously estimates the distances $\{\hat{d}_{ij,t}\}_{j \neq i}$ to all partner agents. When agent $i$ selects a goal region centered at $\mathbf{g}_i$, the coordination reward is computed as:
\begin{equation}
r_{\text{coord}}^i = \sum_{j \neq i} \min\left(\frac{\hat{d}_{ij,t}}{d_{\text{norm}}}, 1.0\right) \cdot \mathbbm{1}[\hat{d}_{ij,t} \geq d_{\text{min}}],
\end{equation}
where $d_{\text{norm}} = 10.0$ grid cells is a normalization constant, $d_{\text{min}} = 3.0$ grid cells is a minimum desired separation threshold, and $\mathbbm{1}[\cdot]$ is an indicator function that provides a bonus only when agents maintain at least the minimum distance. This design creates a repulsive force between agents proportional to their estimated separation, directly operationalizing the principle that distributed exploration maximizes information gain. Critically, this component relies entirely on the SPC module's distance predictions $\hat{d}_{ij,t}$—without accurate distance estimation, agents cannot effectively coordinate their exploration strategies. This establishes a direct computational dependency: the SPC module's learned distance-tuned neurons (described in \textbf{Fig.~\ref{fig:social_place_code}c}) provide the essential geometric information that enables the ICM to generate coordination signals.

The \textit{exploration reward} $r_{\text{explore}}$ provides a direct bonus for discovering cells that were previously unknown to the entire team, quantified by counting newly revealed map cells in a local neighborhood around the agent and applying an exponential distance decay with decay constant $\alpha = 0.1$:
\begin{equation}
r_{\text{explore}} = \sum_{(x,y) \in \mathcal{N}(\mathbf{r}_i, r_{\text{local}})} \mathbbm{1}[M^{\text{shared}}_{x,y,t-1} = 0 \land M^{i}_{x,y,t} > 0] \cdot \exp(-\alpha \|\mathbf{r}_i - (x,y)\|_2),
\end{equation}
where $\mathcal{N}(\mathbf{r}_i, r_{\text{local}})$ is a local neighborhood of radius $r_{\text{local}} = 2$ grid cells around the agent's position $\mathbf{r}_i$, $M^{\text{shared}}$ denotes the team's collective map, and $M^{i}$ is agent $i$'s local map. This three-component design ensures that agents are simultaneously attracted to informative frontiers, repelled from teammates to avoid overlap (mediated by SPC distance estimates), and directly rewarded for expanding the collective map.

\subsubsection{Communication mechanism}

Inter-agent communication is realized through a bandwidth-limited message-passing system that operates in parallel with the high-level decision loop, with communication feasibility determined by the SPC module's distance estimates. Each agent is endowed with a finite budget of \textit{communication tokens} (typically 10 per episode), which depletes with each message transmission and regenerates slowly over time (refill rate $\rho_{\text{refill}} = 1/60$ per step). The decision to communicate involves two stages: first, the SPC module determines which partners are within communication range (defined as $\hat{d}_{ij,t} \leq d_{\text{comm}})$, establishing the set of feasible communication targets; second, a learned gating policy decides whether to actually transmit a message to these reachable partners.

The gating policy is modeled as a logistic classifier that takes as input a 9-dimensional feature vector encoding: (1) the agent's current exploration progress (mean exploration ratio across regions), (2) remaining token budget (normalized by initial budget), (3) local map confidence (average confidence in visible region), (4) spatial connectivity (number of adjacent navigable cells), (5-8) a one-hot encoding of location type (junction, corridor, dead-end, open-area), and (9) a bias term. Critically, this feature vector does \textit{not} explicitly include partner distance—the distance constraint is enforced at the architectural level by the SPC module's range-gating, ensuring that communication is only physically possible when $\hat{d}_{ij,t} \leq d_{\text{comm}}$. The policy outputs a binary communication decision via a sigmoid activation: $p_{\text{comm}} = \sigma(\mathbf{w}^\top \mathbf{f}_t + b)$, where $\mathbf{w} \in \mathbbm{R}^9$ are learned weights, $\mathbf{f}_t$ is the feature vector, and $b$ is a learned bias. When an agent elects to communicate and has available tokens, it broadcasts its local occupancy map and current position to all teammates within the SPC-determined communication range.

This design establishes a clear functional hierarchy: the SPC module's distance-tuned neurons provide the low-level geometric constraint that defines when communication is physically feasible (mimicking limited-range radio communication), while the learned gating policy operates within these constraints to decide when communication is strategically valuable. This separation ensures that the system respects realistic communication limitations while still learning an intelligent transmission strategy. The causal necessity of accurate distance estimation is evident: if the SPC module's distance predictions $\hat{d}_{ij,t}$ are inaccurate (as occurs after SPC lesioning, \textbf{Fig.~\ref{fig:social_place_code}e}), agents will incorrectly estimate which partners are reachable, leading to failed communication attempts or missed opportunities for coordination.

Upon receiving a message from a partner confirmed to be within range (via bidirectional SPC distance checks), the recipient performs an intelligent map fusion operation. Rather than naively overwriting its local map, the agent maintains auxiliary \textit{confidence} and \textit{timestamp} matrices that track the reliability and recency of each cell's occupancy estimate. When integrating received information, conflicting observations are resolved via a multi-criteria decision rule: more recent information is preferred over stale data, higher-confidence estimates override lower-confidence ones, and in cases of equal confidence, wall observations are given precedence over free-space observations as a safety heuristic. This confidence-weighted fusion mechanism ensures that the shared spatial memory remains coherent despite asynchronous and potentially noisy observations, while the token-based gating prevents communication saturation and encourages agents to transmit selectively at moments of high informational value.

\subsubsection{Training configuration and optimization}

The entire HRL-ICM system is trained end-to-end using the MAPPO algorithm with the following configuration. The meta-controller makes a high-level goal selection every $K=20$ environment steps, during which the low-level A* planner executes primitive movement actions (move-forward, turn-left, turn-right, stay) to navigate toward the chosen region. Extrinsic rewards include a large bonus $(+500)$ for task success (locating the hidden goal), a small step penalty $(-0.01)$ to encourage efficiency, and a collision penalty $(-3)$. The total reward at each decision point is the sum of extrinsic and intrinsic components. Advantages for policy gradient updates are computed using Generalized Advantage Estimation (GAE) with discount factor $\gamma = 0.99$ and trace parameter $\lambda = 0.95$. The policy is optimized using the Adam optimizer with a learning rate of $3 \times 10^{-4}$, gradient clipping at norm 1.0, and an entropy regularization coefficient of $\eta = 0.01$ to maintain exploration. Training is conducted over multiple episodes on procedurally generated mazes of varying sizes ($15 \times 15$ to $39 \times 39$), ensuring that the learned policy generalizes across diverse spatial layouts rather than overfitting to a fixed map. The key hyperparameters are summarized in \textbf{Table~\ref{tab:hrl_training_config}}.

\begin{table}[t]
\centering
\caption{\textbf{Training configuration of HRL-ICM.}}
\label{tab:hrl_training_config}
\setlength{\tabcolsep}{37pt}
\begin{tabular}{lll}
\toprule
\textbf{Category} & \textbf{Parameter} & \textbf{Value} \\
\midrule
\multirow{4}{*}{Meta-Controller} & Regional Grid Size ($g \times g$) & $4 \times 4$ \\
 & Hidden Dimension & 256 \\
 & Input Feature Dimension & 48 \\
 & Output Actions (Regions) & 16 \\
\midrule
\multirow{5}{*}{SPC Module} & LSTM Hidden Size & 128 \\
 & Bottleneck Dimension & 256 \\
 & Distance Loss Weight ($w_{\text{dist}}$) & 1.0 \\
 & Distance Normalization ($d_{\text{norm}}$) & 10.0 cells \\
 & Min. Coordination Distance ($d_{\text{min}}$) & 3.0 cells \\
\midrule
\multirow{5}{*}{ICM Weights} & Curiosity Weight ($w_{\text{curiosity}}$) & 1.0 \\
 & Coordination Weight ($w_{\text{coord}}$) & 0.5 \\
 & Exploration Weight ($w_{\text{explore}}$) & 0.3 \\
 & Exploration Radius ($r_{\text{local}}$) & 2 cells \\
 & Distance Decay ($\alpha$) & 0.1 \\
\midrule
\multirow{4}{*}{Communication} & Token Budget & 10 \\
 & Refill Rate ($\rho_{\text{refill}}$) & $1/60$ per step \\
 & Gating Feature Dimension & 9 \\
 & Communication Range ($d_{\text{comm}}$) & 5.0 cells \\
\midrule
\multirow{7}{*}{MAPPO Training} & Algorithm & MAPPO \\
 & High-Level Decision Interval ($K$) & 20 steps \\
 & Learning Rate & $3 \times 10^{-4}$ \\
 & Discount Factor ($\gamma$) & 0.99 \\
 & GAE Parameter ($\lambda$) & 0.95 \\
 & Entropy Coefficient ($\eta$) & 0.01 \\
 & Gradient Clip Norm & 1.0 \\
\midrule
\multirow{3}{*}{Rewards} & Task Success & +500 \\
 & Step Penalty & $-0.01$ \\
 & Collision Penalty & $-3$ \\
\bottomrule
\end{tabular}
\end{table}

The hierarchical training procedure ensures that all components are jointly optimized. The SPC module is pre-trained on trajectory prediction tasks to establish stable distance estimation before being integrated into the full HRL-ICM loop. During full system training, the SPC parameters are kept frozen for the first $10^4$ environment steps to allow the meta-controller and ICM to stabilize, after which all components are fine-tuned end-to-end. This staged training strategy prevents the SPC module from adapting to spurious reward signals and ensures that its distance predictions remain grounded in the geometric structure of agent trajectories. The integration of these components creates a complete computational loop: the SPC module provides geometric awareness of partner states, the ICM translates this into strategic exploration incentives, the meta-controller selects goals based on these incentives, and the communication mechanism leverages distance estimates to coordinate information sharing—all unified under the MAPPO training objective.

%% file: sections/SI/Path_Integration_via_grid_cell.tex
\subsection{Theoretical framework: Path integration and grid cell emergence}

This section establishes a rigorous mathematical framework demonstrating how recurrent networks trained on predictive coding objectives naturally discover path integration dynamics and hexagonally-periodic spatial representations. We develop a principled theoretical argument showing that these solutions emerge from fundamental geometric constraints imposed by the task structure through three stages: (1) establishing preliminaries; (2) proving equivariance-driven path integration; and (3) deriving hexagonal symmetry from isotropy requirements.

\subsubsection{Preliminaries: Mathematical foundations}
\label{sec: preliminaries}

We formalize the mathematical structure underlying two-dimensional spatial navigation, establishing notation and dynamical equations that any successful path integration system typically incorporates.

\paragraph{State space and kinematics.}

\begin{definition}[Pose and state space]
\label{def:pose}
An agent navigating in two-dimensional space is characterized by its \emph{pose} $s_t = (\mathbf{r}_t, \theta_t) \in \mathbb{R}^2 \times \mathbb{S}^1$, where $\mathbf{r}_t = (r_{x,t}, r_{y,t})^\top \in \mathbb{R}^2$ denotes the position vector in Cartesian coordinates (allocentric reference frame) and $\theta_t \in [0, 2\pi)$ denotes the heading angle measured counterclockwise from the positive $x$-axis. Throughout this work, boldfaced lowercase letters denote column vectors while regular lowercase letters denote scalars.
\end{definition}

\begin{definition}[Rotation Matrix]
\label{def:rotation-matrix}
The rotation matrix $R(\alpha) \in SO(2)$ rotates vectors in $\mathbb{R}^2$ counterclockwise by angle $\alpha$~\cite{lynch2017modern_supp}:
\begin{equation}
\label{eq:rotation-matrix}
R(\alpha) = \begin{pmatrix} \cos\alpha & -\sin\alpha \\ \sin\alpha & \cos\alpha \end{pmatrix}.
\end{equation}
\end{definition}

\begin{lemma}[Rotation matrix properties]
\label{lem:rotation-properties}
The rotation matrices satisfy: (i) composition law $R(\alpha)R(\gamma) = R(\alpha + \gamma)$, (ii) identity $R(0) = I_2$, (iii) inverse $R(\alpha)^{-1} = R(-\alpha) = R(\alpha)^\top$, and (iv) orthogonality $R(\alpha)^\top R(\alpha) = I_2$, implying norm preservation $\|R(\alpha)\mathbf{v}\| = \|\mathbf{v}\|$~\cite{lynch2017modern_supp}.
\end{lemma}

\begin{proof}
Property (i) follows from trigonometric angle addition formulas:
\begin{align}
R(\alpha)R(\gamma) &= \begin{pmatrix} \cos\alpha & -\sin\alpha \\ \sin\alpha & \cos\alpha \end{pmatrix}
\begin{pmatrix} \cos\gamma & -\sin\gamma \\ \sin\gamma & \cos\gamma \end{pmatrix} \nonumber \\
&= \begin{pmatrix}
\cos(\alpha+\gamma) & -\sin(\alpha+\gamma) \\
\sin(\alpha+\gamma) & \cos(\alpha+\gamma)
\end{pmatrix} = R(\alpha + \gamma),
\end{align}
where the second equality uses $\cos(\alpha+\gamma) = \cos\alpha\cos\gamma - \sin\alpha\sin\gamma$ and $\sin(\alpha+\gamma) = \sin\alpha\cos\gamma + \cos\alpha\sin\gamma$. Properties (ii)-(iv) follow by direct computation.
\end{proof}

Motor commands naturally arise in the egocentric (body-centered) frame, while stable spatial memory requires allocentric representations. A vector $\mathbf{v}^{\text{ego}}$ in egocentric coordinates transforms to allocentric coordinates via $\mathbf{v}^{\text{allo}} = R(\theta_t) \mathbf{v}^{\text{ego}}$, mediating the heading-dependent transformation central to path integration.

\begin{definition}[Path integration dynamics\cite{gao2021path_supp}]
\label{def:pi-dynamics}
At each discrete time step $t$ with uniform interval $\Delta t > 0$, the agent receives motor command $u_t = (\mathbf{v}_t, \omega_t) \in \mathbb{R}^2 \times \mathbb{R}$ specifying egocentric translational velocity $\mathbf{v}_t = (v_{x,t}, v_{y,t})^\top$ and angular velocity $\omega_t$. The pose evolves according to:
\begin{align}
\theta_{t+1} &= \theta_t + \omega_t \Delta t \pmod{2\pi}, \label{eq:heading-dynamics} \\
\mathbf{r}_{t+1} &= \mathbf{r}_t + \Delta\mathbf{r}_t, \quad \mathrm{where} \quad \Delta\mathbf{r}_t = R(\theta_t) \mathbf{v}_t \Delta t~. \label{eq:position-dynamics}
\end{align}

Given initial pose $s_0 = (\mathbf{r}_0, \theta_0)$ and command sequence $\{u_t\}_{t=0}^{T-1}$, the cumulative final pose is $\mathbf{r}_T = \mathbf{r}_0 + \sum_{t=0}^{T-1} R(\theta_t) \mathbf{v}_t \Delta t$ and $\theta_T = \theta_0 + \sum_{t=0}^{T-1} \omega_t \Delta t \pmod{2\pi}$.
\end{definition}

\begin{remark}[Position encoding vs Full pose]
\label{rem:position-vs-pose}
Although complete pose $(\mathbf{r}, \theta)$ requires three dimensions of information (two for position, one for heading), our subsequent analysis focuses exclusively on position encoding $\mathbf{r} \in \mathbb{R}^2$ since: (i) position integration $\sum R(\theta_t) \mathbf{v}_t \Delta t$ involves heading-dependent rotations and is computationally complex while heading integration $\sum \omega_t \Delta t$ is simple scalar accumulation, (ii) biological grid cells specifically encode position independent of heading while head direction cells separately encode orientation, and (iii) LSTM hidden states $\mathbf{h} \in \mathbb{R}^d$ can decompose into independent subspaces for position (requiring sophisticated encoding) and heading (requiring simple integration).
\end{remark}

\paragraph{Learning objective.}

The latent pose $s_t$ is not directly observed. Instead, the system observes activity from $C$ spatial cells with log-firing potential $\phi_i: \mathbb{R}^2 \times \mathbb{S}^1 \to \mathbb{R}$. Place cells use Gaussian receptive fields $\phi_i(\mathbf{r}) = -\|\mathbf{r} - \boldsymbol\mu_i\|^2/(2\sigma_i^2)$ while head-direction cells use von Mises tuning $\phi_j(\theta) = \kappa_j \cos(\theta - \mu_j)$. Observations follow softmax distribution $p(y=\ell \mid s) = \exp\{\phi_\ell(s)\}/\sum_{m=1}^C \exp\{\phi_m(s)\}$.

\begin{definition}[Prediction objective]
\label{def:lstm-objective}
An LSTM network with hidden state $h_t \in \mathbb{R}^d$ receives initial observation $y_0$ and command sequence $\{u_t\}_{t=1}^{T}$, producing predictive distribution $\hat{p}(\cdot \mid y_0, \{u_t\})$ over the final sensory pattern. The learning objective minimizes prediction error:
\begin{equation}
\label{eq:prediction-objective}
\mathcal{L}(\boldsymbol\varTheta ) = \mathbb{E}_{y_0, \{u_t\}} 
    \left[ \mathrm{KL}\big(p(\cdot \mid s_T) \,\|\, \hat{p}(\cdot \mid y_0, \{u_t\})\big) \right],
\end{equation}
where $s_T$ is the true final pose from \textbf{Eqs.~\eqref{eq:heading-dynamics}--\eqref{eq:position-dynamics}}, and the expectation is over diverse trajectories. $\boldsymbol\varTheta$ denotes the network parameters. This objective does not explicitly prescribe path integration or specific internal representations; rather, path integration emerges as the optimal strategy for minimizing prediction error over diverse navigation trajectories.
\end{definition}

\paragraph{Roadmap.} The emergence of hexagonal grid patterns proceeds through three steps: \textbf{Proposition \ref{prop:equivariance}} proves equivariance under rigid body transformations, \textbf{Theorem \ref{thm:rotation-structure}} derives cosine-sine phase encoding from stability constraints, and \textbf{Corollary \ref{cor:hexagonal}} establishes hexagonal symmetry from isotropy requirements.

%============================================================
\subsubsection{Equivariance under rigid body transformations}
\label{sec: Equivariance Under Rigid Body Transformations}

Having established the dynamical foundations of path integration, we now demonstrate that these dynamics possess a critical geometric property: equivariance under rigid body transformations. This symmetry fundamentally constrains how neural networks can represent spatial information. 

\begin{definition}[Rigid body transformation]
\label{def:rigid-transform}
A rigid body transformation $G_{\boldsymbol\delta, \Phi}$ parameterized by translation $\boldsymbol\delta \in \mathbb{R}^2$ and rotation angle $\phi \in [0, 2\pi)$ acts on pose $s = (\mathbf{r}, \theta)$ as $G_{\boldsymbol\delta, \Phi}(\mathbf{r}, \theta) = (R(\phi) \mathbf{r} + \boldsymbol\delta, \theta + \phi)$, rotating position by $\phi$ then translating by $\boldsymbol\delta$ while simultaneously rotating heading. The set $\{G_{\boldsymbol\delta, \Phi}\}$ forms the special Euclidean group $SE(2)$, the symmetry group of planar rigid motions.
\end{definition}

\begin{proposition}[Equivariance of physical dynamics]
\label{prop:equivariance}
The path integration \textbf{Eqs. \eqref{eq:heading-dynamics}--\eqref{eq:position-dynamics}} are equivariant with respect to rigid body transformations: if trajectory $\{s_t\}_{t=0}^T$ evolves from initial condition $s_0$ under motor commands $\{u_t\}_{t=0}^{T-1}$, then for any $G_{\boldsymbol\delta, \Phi}$, the transformed trajectory $\{\tilde{s}_t\}_{t=0}^T$ defined by $\tilde{s}_t = G_{\boldsymbol\delta, \Phi}(s_t)$ evolves from transformed initial condition $\tilde{s}_0 = G_{\boldsymbol\delta, \Phi}(s_0)$ under the same command sequence.
\end{proposition}

\begin{proof}
Since $t \in \{0, 1, 2, \ldots, T\}$ is a discrete index enumerating time steps (with each step having duration $\Delta t > 0$), we proceed by mathematical induction on $P(t)$: $(\tilde{\mathbf{r}}_t, \tilde{\theta}_t) = (R(\phi)\mathbf{r}_t + \boldsymbol\delta, \theta_t + \phi)$. Base case ($t=0$): $P(0)$ holds by definition. Inductive step: assuming $P(k)$ holds, applying \textbf{Eqs. \eqref{eq:heading-dynamics}--\eqref{eq:position-dynamics}} to transformed state $\tilde{s}_k$ with motor command $u_k = (\mathbf{v}_k, \omega_k)$ yields for heading $\tilde{\theta}_{k+1} = \tilde{\theta}_k + \omega_k \Delta t = (\theta_k + \phi) + \omega_k \Delta t = \theta_{k+1} + \phi$, and for position:
\begin{align}
\tilde{\mathbf{r}}_{k+1} &\stackrel{\eqref{eq:position-dynamics}}{=} \tilde{\mathbf{r}}_k + R(\tilde{\theta}_k) \mathbf{v}_k \Delta t
\stackrel{P(k)}{=} (R(\phi)\mathbf{r}_k + \boldsymbol\delta) + R(\theta_k + \phi) \mathbf{v}_k \Delta t \nonumber \\
&\stackrel{\text{Lem.}\ref{lem:rotation-properties}(i)}{=} (R(\phi)\mathbf{r}_k + \boldsymbol\delta) + R(\phi) R(\theta_k) \mathbf{v}_k \Delta t
= R(\phi) (\mathbf{r}_k + R(\theta_k) \mathbf{v}_k \Delta t) + \boldsymbol\delta \nonumber \\
&\stackrel{\eqref{eq:position-dynamics}}{=} R(\phi)\mathbf{r}_{k+1} + \boldsymbol\delta,
\end{align}
where we used rotation composition $R(\theta_k + \phi) = R(\phi)R(\theta_k)$ and matrix distributivity. Thus $P(k+1)$ holds and by induction, $P(t)$ holds for all $t \geq 0$. 
\end{proof}

\begin{corollary}[Equivariance constraint on learned representations]
\label{cor:network-equivariance}
Any network achieving zero prediction error on \textbf{objective~\eqref{eq:prediction-objective}} must learn an internal representation that respects the equivariance structure of the physical dynamics established in \textbf{Proposition~\ref{prop:equivariance}}.
\end{corollary}

\begin{proof}
Suppose the network achieves zero prediction error, meaning $\hat{p}(\cdot \mid y_0, \{u_t\}) = p(\cdot \mid s_T)$ almost surely where $s_T$ results from the dynamics in \textbf{Eqs. \eqref{eq:heading-dynamics}--\eqref{eq:position-dynamics}}. Consider a transformed trajectory starting from $\tilde{s}_0 = G_{\boldsymbol\delta, \Phi}(s_0)$ with corresponding initial observation $\tilde{y}_0$. The network receives the same control sequence $\{u_t\}$ and must predict the spatial cell activations at the final transformed pose $\tilde{s}_T$. By \textbf{Proposition \ref{prop:equivariance}}, we know $\tilde{s}_T = G_{\boldsymbol\delta, \Phi}(s_T)$. Since the place and head-direction cell activations are determined uniquely by pose, perfect prediction on both the original and transformed trajectories requires that the network's internal state evolution mirrors the geometric transformation: if the network maps $(y_0, \{u_t\})$ to some internal representation $\mathbf{h}_T$ that encodes pose $s_T$, then it must map $(\tilde{y}_0, \{u_t\})$ to a representation $\tilde{\mathbf{h}}_T$ that encodes the transformed pose $\tilde{s}_T = G_{\boldsymbol\delta, \Phi}(s_T)$. This equivariance constraint on the latent representation is a necessary consequence of achieving zero prediction error across all possible trajectories and their rigid transformations.
\end{proof}

This geometric constraint, combined with stability requirements for recurrent integration, uniquely determines the structure of position-encoding neural codes as demonstrated next.

%============================================================
\subsubsection{Emergence of sinusoidal phase encoding}
\label{sec: Emergence of Sinusoidal Phase Encoding}

Having established that learned networks must exhibit equivariance, we now examine how LSTM hidden states encode position information. The key insight is that stable, composable position updates naturally lead to periodic representations. We posit that the LSTM hidden state $\mathbf{h} \in \mathbb{R}^d$ contains at least one two-dimensional subspace $\mathbf{y} = (y_1, y_2)^\top \in \mathbb{R}^2$ encoding position $\mathbf{r} \in \mathbb{R}^2$ through invertible linear mapping $\mathbf{r} = W \mathbf{y} + \mathbf{b}$ where $W \in \mathbb{R}^{2 \times 2}$ is invertible ($\det(W) \neq 0$) and $\mathbf{b} \in \mathbb{R}^2$ is bias\cite{schaeffer2022no_supp}. The subspace updates via $\mathbf{y}_{t+1} = M(\Delta\mathbf{r}_t) \mathbf{y}_t$ where $M: \mathbb{R}^2 \to \mathbb{R}^{2 \times 2}$ is an update operator and $\Delta\mathbf{r}_t = \mathbf{r}_{t+1} - \mathbf{r}_t$ is position increment. This analyses one encoding frequency; multiple subspaces with different frequencies may coexist as shown later in \textbf{Corollary \ref{cor:hexagonal}}.

\begin{lemma}[Update operator constraints]
\label{lem:update-constraints}
To support stable, composable path integration, we impose the following four constraints on the update operator $M$: (i) identity at zero $M(\mathbf{0}) = I_2$ ensuring no change under no displacement, (ii) continuity in $\Delta\mathbf{r}$ reflecting smooth neural dynamics, (iii) composability $M(\mathbf{a})M(\mathbf{b}) = M(\mathbf{a} + \mathbf{b})$ ensuring path-independence, and (iv) orthogonality $M(\Delta\mathbf{r})^\top M(\Delta\mathbf{r}) = I_2$ preserving norm for stability.
\end{lemma}

\begin{proof}[Justification]
Property (i): For path integration to be well-defined, zero displacement should not change the internal representation for any initial state $\mathbf{y}_t$, the requirement that $\mathbf{y}_{t+1} = M(\mathbf{0})\mathbf{y}_t = \mathbf{y}_t$ for all $\mathbf{y}_t$, which implies $M(\mathbf{0}) = I_2$. Property (ii): LSTM operations (matrix multiplications, additions, smooth activations) are continuous functions, naturally yielding continuity of $M$ in $\Delta\mathbf{r}$. Property (iii): Consider two successive displacements $\mathbf{r}_t \xrightarrow{\mathbf{a}} \mathbf{r}_{t+1} = \mathbf{r}_t + \mathbf{a}$ yielding $\mathbf{y}_{t+1} = M(\mathbf{a})\mathbf{y}_t$ and $\mathbf{r}_{t+1} \xrightarrow{\mathbf{b}} \mathbf{r}_{t+2} = \mathbf{r}_{t+1} + \mathbf{b}$ yielding $\mathbf{y}_{t+2} = M(\mathbf{b})\mathbf{y}_{t+1} = M(\mathbf{b})M(\mathbf{a})\mathbf{y}_t$; path-independence requires this equals direct displacement result $\mathbf{y}_{t+2} = M(\mathbf{a} + \mathbf{b})\mathbf{y}_t$ for all $\mathbf{y}_t$, giving $M(\mathbf{a})M(\mathbf{b}) = M(\mathbf{a} + \mathbf{b})$. Property (iv): To avoid gradient explosion/vanishing during recurrent updates, we require norm preservation $\|\mathbf{y}_{t+1}\| = \|\mathbf{y}_t\|$ for all $\mathbf{y}_t$, giving $\|M(\Delta\mathbf{r}_t)\mathbf{y}_t\|^2 = \mathbf{y}_t^\top M^\top M \mathbf{y}_t = \|\mathbf{y}_t\|^2$ for all $\mathbf{y}_t$, implying $M^\top M = I_2$.
\end{proof}

\begin{theorem}[Rotation matrix structure with linear phase]
\label{thm:rotation-structure}
Any operator $M: \mathbb{R}^2 \to \mathbb{R}^{2 \times 2}$ satisfying the four properties in \textbf{Lemma~\ref{lem:update-constraints}} must have the form $M(\Delta\mathbf{r}) = R(\mathbf{q} \cdot \Delta\mathbf{r})$ for some frequency vector $\mathbf{q} = (q_x, q_y)^\top \in \mathbb{R}^2$, where $\mathbf{q} \cdot \Delta\mathbf{r} = q_x \Delta r_x + q_y \Delta r_y$ and $R(\cdot)$ is the rotation matrix from \textbf{Definition \ref{def:rotation-matrix}}.
\end{theorem}

\begin{proof}
The proof proceeds by leveraging each constraint from \textbf{Lemma \ref{lem:update-constraints}} systematically to determine the unique functional form of $M$. We divide the procedure of proof into the following 4 steps:

\emph{Step 1: $M$ maps to the rotation group $SO(2)$.}
Property (iv) in \textbf{Lemma \ref{lem:update-constraints}} establishes that $M(\Delta\mathbf{r})$ is orthogonal: $M(\Delta\mathbf{r})^\top M(\Delta\mathbf{r}) = I_2$. This means $M(\Delta\mathbf{r}) \in O(2)$, the orthogonal group. Additionally, Property (i) gives $M(\mathbf{0}) = I_2$ which has determinant $+1$. By Property (ii), $M$ is continuous in $\Delta\mathbf{r}$, so $\det(M(\Delta\mathbf{r}))$ varies continuously with $\Delta\mathbf{r}$ while remaining in $\{-1, +1\}$ (the only possible determinants for orthogonal matrices). Since $\det(M(\mathbf{0})) = \det(I_2) = +1$ and determinant cannot jump discontinuously, we conclude $\det(M(\Delta\mathbf{r})) = +1$ for all $\Delta\mathbf{r}$. Therefore, $M(\Delta\mathbf{r}) \in SO(2)$, the special orthogonal group of rotation matrices.

\emph{Step 2: Deriving the functional equation from composability.}
Property (iii) in \textbf{Lemma \ref{lem:update-constraints}} states $M(\mathbf{a})M(\mathbf{b}) = M(\mathbf{a} + \mathbf{b})$ for all $\mathbf{a}, \mathbf{b} \in \mathbb{R}^2$. Applying this recursively with $\mathbf{a} = \mathbf{b} = \Delta\mathbf{r}$ yields:
\begin{equation}
M(2\Delta\mathbf{r}) = M(\Delta\mathbf{r} + \Delta\mathbf{r}) = M(\Delta\mathbf{r})M(\Delta\mathbf{r}) = M(\Delta\mathbf{r})^2.
\end{equation}

Continuing inductively, $M(n\Delta\mathbf{r}) = M(\Delta\mathbf{r})^n$ for any positive integer $n \in \mathbb{N}$. For negative integers, using Property (i) we have $I_2 = M(\mathbf{0}) = M(\Delta\mathbf{r} + (-\Delta\mathbf{r})) = M(\Delta\mathbf{r})M(-\Delta\mathbf{r})$, so $M(-\Delta\mathbf{r}) = M(\Delta\mathbf{r})^{-1}$, extending the power rule to all integers $n \in \mathbb{Z}$. For rational exponents, consider $n,m \in \mathbb{N}$ with $m \neq 0$. Let $\mathbf{w} = \Delta\mathbf{r}/m$. Then, we have $M(\Delta\mathbf{r}) = M(m\mathbf{w}) = M(\mathbf{w})^m$. In $SO(2)$, every rotation $R(\theta)$ has $m$ distinct $m$-th roots: $R(\theta/m + 2\pi k/m)$ for $k = 0, 1, \ldots, m-1$. However, by continuity of $M$ (Property (ii)) and $M(\mathbf{0}) = I_2$, for small $\|\Delta\mathbf{r}\|$, we must have $M(\Delta\mathbf{r})$ close to $I_2$, which uniquely determines the principal root with $k=0$. Extending by continuity for all $\Delta\mathbf{r}$, we obtain $M(\mathbf{w}) = M(\Delta\mathbf{r})^{1/m}$ as the unique continuous choice. Therefore:
\begin{equation}
M\left(\frac{n}{m}\Delta\mathbf{r}\right) = M(n\mathbf{w}) = M(\mathbf{w})^n = \left(M(\Delta\mathbf{r})^{1/m}\right)^n = M(\Delta\mathbf{r})^{n/m}.
\end{equation}

Thus $M(\alpha\Delta\mathbf{r}) = M(\Delta\mathbf{r})^\alpha$ holds for all $\alpha \in \mathbb{Q}$ (Set of rational number). Invoking Property (ii), continuity of $M$ in $\Delta\mathbf{r}$, combined with density of rationals in reals, extends this to all $\alpha \in \mathbb{R}$:
\begin{equation}
\label{eq:power-law}
M(\alpha\Delta\mathbf{r}) = M(\Delta\mathbf{r})^\alpha \quad \text{for all } \alpha \in \mathbb{R}, \, \Delta\mathbf{r} \in \mathbb{R}^2.
\end{equation}

\emph{Step 3: Characterizing rotation matrices via one-parameter subgroups.}
Since $M(\Delta\mathbf{r}) \in SO(2)$ (from Step 1), we can write $M(\Delta\mathbf{r}) = R(\theta(\Delta\mathbf{r}))$ for some angle function $\theta: \mathbb{R}^2 \to \mathbb{R}$. From Property (iii), the composability $M(\mathbf{a})M(\mathbf{b}) = M(\mathbf{a} + \mathbf{b})$ translates to:
\begin{equation}
R(\theta(\mathbf{a}))R(\theta(\mathbf{b})) \stackrel{\text{Lem.}\ref{lem:rotation-properties}(i)}{=} R(\theta(\mathbf{a}) + \theta(\mathbf{b})) = R(\theta(\mathbf{a} + \mathbf{b})).
\end{equation}

Since $R(\cdot)$ is injective modulo $2\pi$, this gives $\theta(\mathbf{a} + \mathbf{b}) = \theta(\mathbf{a}) + \theta(\mathbf{b}) \pmod{2\pi}$. Property (ii) ensures $\theta$ is continuous. Property (i) provides $M(\mathbf{0}) = I_2 = R(0)$, thus $\theta(\mathbf{0}) = 0 \pmod{2\pi}$. Choosing the continuous branch with $\theta(\mathbf{0}) = 0$, we obtain the additive Cauchy functional equation:
\begin{equation}
\label{eq:cauchy}
\theta(\mathbf{a} + \mathbf{b}) = \theta(\mathbf{a}) + \theta(\mathbf{b}), \quad \theta(\mathbf{0}) = 0.
\end{equation}

\emph{Step 4: Solving the Cauchy equation yields linear form.}
The continuous solution to the additive functional \textbf{Eq. \eqref{eq:cauchy}} on $\mathbb{R}^2$ must be linear: $\theta(\Delta\mathbf{r}) = \mathbf{q} \cdot \Delta\mathbf{r}$ for some constant vector $\mathbf{q} = (q_x, q_y)^\top \in \mathbb{R}^2$. To see this, first restrict to one-dimensional subspaces. For any fixed unit vector $\mathbf{e} \in \mathbb{R}^2$ with $\|\mathbf{e}\| = 1$, define $f_{\mathbf{e}}(\xi) = \theta(\xi\mathbf{e})$ for $\xi \in \mathbb{R}$. Then \textbf{Eq. \eqref{eq:cauchy}} gives:
\begin{equation}
f_{\mathbf{e}}(\xi + \eta) = \theta(\xi\mathbf{e} + \eta\mathbf{e}) = \theta(\xi\mathbf{e}) + \theta(\eta\mathbf{e}) = f_{\mathbf{e}}(\xi) + f_{\mathbf{e}}(\eta),
\end{equation}
with $f_{\mathbf{e}}(0) = 0$. Since $\theta$ is continuous (Property (ii)), $f_{\mathbf{e}}$ is also continuous. The unique continuous solution to Cauchy's functional equation $f(\xi+\eta) = f(\xi) + f(\eta)$ on $\mathbb{R}$ is $f_{\mathbf{e}}(\xi) = c_{\mathbf{e}} \xi$ for some constant $c_{\mathbf{e}} \in \mathbb{R}$. Thus $\theta(\xi\mathbf{e}) = c_{\mathbf{e}} \xi$.

For a general displacement $\Delta\mathbf{r} = \Delta r_x \mathbf{e}_x + \Delta r_y \mathbf{e}_y$ where $\mathbf{e}_x = (1,0)^\top$ and $\mathbf{e}_y = (0,1)^\top$ are the standard basis vectors, linearity of $\theta$ from \textbf{Eq. \eqref{eq:cauchy}} yields:
\begin{align}
\theta(\Delta\mathbf{r}) &= \theta(\Delta r_x \mathbf{e}_x + \Delta r_y \mathbf{e}_y) = \theta(\Delta r_x \mathbf{e}_x) + \theta(\Delta r_y \mathbf{e}_y) \nonumber \\
&= c_{\mathbf{e}_x} \Delta r_x + c_{\mathbf{e}_y} \Delta r_y = \mathbf{q} \cdot \Delta\mathbf{r},
\end{align}
where we defined $\mathbf{q} = (c_{\mathbf{e}_x}, c_{\mathbf{e}_y})^\top = (q_x, q_y)^\top$. This is the inner product between frequency vector $\mathbf{q}$ and displacement $\Delta\mathbf{r}$.

Combining Steps 1--4, we have established that $M(\Delta\mathbf{r}) = R(\theta(\Delta\mathbf{r})) = R(\mathbf{q} \cdot \Delta\mathbf{r})$, where the rotation angle is a linear function of the displacement. Each of the four properties in \textbf{Lemma \ref{lem:update-constraints}} was essential: Property (i) fixed the identity element and normalization, Property (ii) enabled application of continuous functional equation theory, Property (iii) imposed the group homomorphism structure yielding the additive Cauchy equation, and Property (iv) restricted the image to rotation matrices in $SO(2)$. The frequency vector $\mathbf{q}$ emerges as the unique free parameter characterizing the rate of phase advance per unit displacement, thereby determining the spatial periodicity of the neural representation.
\end{proof}

\begin{corollary}[Cosine-Sine phase encoding]
\label{cor:phase-encoding}
Under rotation update $M(\Delta\mathbf{r}) = R(\mathbf{q} \cdot \Delta\mathbf{r})$, accumulated displacement $\mathbf{R} = \mathbf{r}_T - \mathbf{r}_0$ from initial $\mathbf{y}_0 = (1, 0)^\top$ yields $\mathbf{y}_T = (\cos(\mathbf{q} \cdot \mathbf{R}), \sin(\mathbf{q} \cdot \mathbf{R}))^\top$, revealing position encoding through sinusoidal functions in quadrature—the fundamental signature of grid cells.
\end{corollary}

\begin{proof}
Iterating the update equation $\mathbf{y}_{t+1} = M(\Delta\mathbf{r}_t) \mathbf{y}_t$ from $t=0$ to $T-1$ yields $\mathbf{y}_T = \prod_{t=0}^{T-1} M(\Delta\mathbf{r}_t) \mathbf{y}_0$. Substituting $M(\Delta\mathbf{r}_t) = R(\mathbf{q} \cdot \Delta\mathbf{r}_t)$ from \textbf{Theorem \ref{thm:rotation-structure}} and using composition property from \textbf{Lemma \ref{lem:rotation-properties}}(i) gives $\mathbf{y}_T = R\left(\sum_{t=0}^{T-1} \mathbf{q} \cdot \Delta\mathbf{r}_t\right) \mathbf{y}_0 = R(\mathbf{q} \cdot \mathbf{R}) \mathbf{y}_0$ where $\mathbf{R} = \sum_{t=0}^{T-1} \Delta\mathbf{r}_t = \mathbf{r}_T - \mathbf{r}_0$ is the cumulative displacement. Taking $\mathbf{y}_0 = (1, 0)^\top$ as initial state and applying \textbf{Definition \ref{def:rotation-matrix}}, we have:
\begin{equation}
\mathbf{y}_T = R(\mathbf{q} \cdot \mathbf{R}) \begin{pmatrix} 1 \\ 0 \end{pmatrix} = \begin{pmatrix} \cos(\mathbf{q} \cdot \mathbf{R}) & -\sin(\mathbf{q} \cdot \mathbf{R}) \\ \sin(\mathbf{q} \cdot \mathbf{R}) & \cos(\mathbf{q} \cdot \mathbf{R}) \end{pmatrix} \begin{pmatrix} 1 \\ 0 \end{pmatrix} = \begin{pmatrix} \cos(\mathbf{q} \cdot \mathbf{R}) \\ \sin(\mathbf{q} \cdot \mathbf{R}) \end{pmatrix}.
\end{equation}

Thus, path integration over arbitrary trajectories reduces to evaluating sinusoidal functions of the net displacement's projection onto frequency vector $\mathbf{q}$. The phase $\phi = \mathbf{q} \cdot \mathbf{R}$ linearly encodes position, while the quadrature pair $(\cos\phi, \sin\phi)$ provides full $2\pi$-periodic coverage necessary for unambiguous spatial representation—the fundamental computational signature observed in biological grid cells. 
\end{proof}

%============================================================
\subsubsection{Hexagonal symmetry from directional isotropy}
\label{sec: Hexagonal Symmetry from Directional Isotrop}

Having derived sinusoidal phase encoding for individual subspaces, we now address how multiple frequency components combine to produce hexagonal spatial patterns. The crucial insight is that robust spatial representation requires directional isotropy—uniform sensitivity across all navigation directions. For this purpose, LSTM hidden states employ multiple independent two-dimensional subspaces $\{\mathbf{y}^{(j)}\}_{j=1}^K$ each with frequency vector $\mathbf{q}_j \in \mathbb{R}^2$ of equal magnitude $\|\mathbf{q}_j\| = q > 0$ with unit direction $\mathbf{u}_j = \mathbf{q}_j/q$ with $\|\mathbf{u}_j\| = 1$. Multiple frequencies provide spatial resolution at various granularities, robustness against noise, and disambiguation through coarse-fine scale combinations.

\begin{definition}[Directional isotropy]
\label{def:isotropy}
A multi-frequency representation achieves directional isotropy (uniform sensitivity across all displacement directions) if: (i) first-order isotropy $\sum_{j=1}^K \mathbf{u}_j = \mathbf{0}$, ensuring no net directional bias, and (ii) second-order isotropy $\sum_{j=1}^K \mathbf{u}_j \mathbf{u}_j^\top = \lambda I_2$ for some $\lambda > 0$, ensuring equal representational capacity along all axes. The outer product $\mathbf{u}_j \mathbf{u}_j^\top \in \mathbb{R}^{2 \times 2}$ is a rank-one matrix encoding directional concentration.
\end{definition}

\begin{lemma}[Minimal configuration size]
\label{lem:minimal-K}
The minimum number of unit vectors satisfying both isotropy conditions is $K=3$.
% as $K=1$ violates first-order isotropy ($\mathbf{u}_1 \neq \mathbf{0}$) and $K=2$ violates second-order isotropy (first-order requires $\mathbf{u}_2 = -\mathbf{u}_1$, giving $\sum_{j=1}^2 \mathbf{u}_j \mathbf{u}_j^\top = 2 \mathbf{u}_1 \mathbf{u}_1^\top$ which has rank 1 whereas $\lambda I_2$ has rank 2).
\end{lemma}

\begin{proof}
For $K=1$, first-order condition gives $\mathbf{u}_1 = \mathbf{0}$ contradicting $\|\mathbf{u}_1\| = 1$, thus $K=1$ is impossible. For $K=2$, first-order condition gives $\mathbf{u}_1 + \mathbf{u}_2 = \mathbf{0}$, implying $\mathbf{u}_2 = -\mathbf{u}_1$. Then $\sum_{j=1}^2 \mathbf{u}_j \mathbf{u}_j^\top = \mathbf{u}_1 \mathbf{u}_1^\top + (-\mathbf{u}_1)(-\mathbf{u}_1)^\top = 2\mathbf{u}_1 \mathbf{u}_1^\top$, which is a rank-1 matrix (all rows are multiples of $\mathbf{u}_1^\top$) while $\lambda I_2$ has rank 2 for any $\lambda > 0$, yielding contradiction. Thus $K \geq 3$. 
\end{proof}

% \begin{theorem}[Minimal Isotropic Configuration]
% \label{thm:minimal-isotropy}
% For $K=3$, the unique configuration (up to global rotation) satisfying isotropy comprises three unit vectors separated by $120^{\circ}$. Explicitly, choosing
% \begin{align}
% \mathbf{u}_1 = \begin{pmatrix} 1 \\ 0 \end{pmatrix}, \quad
% \mathbf{u}_2 = \begin{pmatrix} -1/2 \\ \sqrt{3}/2 \end{pmatrix}, \quad
% \mathbf{u}_3 = \begin{pmatrix} -1/2 \\ -\sqrt{3}/2 \end{pmatrix}
% \end{align}
% at angles $0^{\circ}, 120^{\circ}, 240^{\circ}$ satisfies both isotropy conditions.
% \end{theorem}

% \begin{proof}
% Direct computation verifies first-order: $\mathbf{u}_1 + \mathbf{u}_2 + \mathbf{u}_3 = (1 - 1/2 - 1/2, 0 + \sqrt{3}/2 - \sqrt{3}/2)^\top = \mathbf{0}$. Computing outer products:
% \begin{align}
% \mathbf{u}_1 \mathbf{u}_1^\top = \begin{pmatrix} 1 & 0 \\ 0 & 0 \end{pmatrix}, \quad
% \mathbf{u}_2 \mathbf{u}_2^\top = \begin{pmatrix} 1/4 & -\sqrt{3}/4 \\ -\sqrt{3}/4 & 3/4 \end{pmatrix}, \quad
% \mathbf{u}_3 \mathbf{u}_3^\top = \begin{pmatrix} 1/4 & \sqrt{3}/4 \\ \sqrt{3}/4 & 3/4 \end{pmatrix},
% \end{align}
% and summing yields $\sum_{j=1}^3 \mathbf{u}_j \mathbf{u}_j^\top = \begin{pmatrix} 3/2 & 0 \\ 0 & 3/2 \end{pmatrix} = \frac{3}{2} I_2$, satisfying second-order isotropy with $\lambda = 3/2$. Uniqueness follows as first-order condition for $K=3$ geometrically requires three unit vectors forming a closed triangle, which combined with second-order symmetry must be equilateral. 
% \end{proof}

\begin{theorem}[Minimal isotropic configuration]
\label{thm:minimal-isotropy}
For $K=3$, the unique configuration (up to global rotation) satisfying both isotropy conditions comprises three unit vectors equally spaced by $120^{\circ}$ angular separation.
\end{theorem}

\begin{proof}
We construct the solution systematically from the isotropy conditions. Parametrize the three unit vectors as $\mathbf{u}_j = (\cos\alpha_j, \sin\alpha_j)^\top$ for $j=1,2,3$ where $\alpha_j \in [0, 2\pi)$ denote the angles measured counterclockwise from the positive $x$-axis. Our goal is to determine the angular configuration $\{\alpha_1, \alpha_2, \alpha_3\}$ uniquely from the two isotropy conditions in \textbf{Definition \ref{def:isotropy}}. Without loss of generality, we fix $\alpha_1 = 0$ by rotational invariance, reducing the problem to finding $\alpha_2$ and $\alpha_3$.
The first-order isotropy condition $\sum_{j=1}^3 \mathbf{u}_j = \mathbf{0}$ decomposes into two scalar equations by examining the $x$ and $y$ components separately:
\begin{align}
\cos 0 + \cos\alpha_2 + \cos\alpha_3 = 1 + \cos\alpha_2 + \cos\alpha_3 &= 0, \label{eq:first-order-x}\\
\sin 0 + \sin\alpha_2 + \sin\alpha_3 = \sin\alpha_2 + \sin\alpha_3 &= 0. \label{eq:first-order-y}
\end{align}

From \textbf{Eq. \eqref{eq:first-order-y}}, we obtain $\sin\alpha_3 = -\sin\alpha_2$. This constraint has two families of solutions: either $\alpha_3 = -\alpha_2 \pmod{2\pi}$, which gives $\alpha_3 = 2\pi - \alpha_2$ for $\alpha_2 \in (0, 2\pi)$, or $\alpha_3 = \pi + \alpha_2$. We examine each case separately.

For the First case $\alpha_3 = \pi + \alpha_2$, substituting into \textbf{Eq. \eqref{eq:first-order-x}}, we have:
\begin{equation}
1 + \cos\alpha_2 + \cos(\pi + \alpha_2) = 1 + \cos\alpha_2 - \cos\alpha_2 = 1 \neq 0,
\end{equation}
producing a contradiction. Thus this case is impossible. 

For the second case $\alpha_3 = 2\pi - \alpha_2$, substituting into \textbf{Eq. \eqref{eq:first-order-x}} yields:
\begin{equation}
\label{eq:second-case}
1 + \cos\alpha_2 + \cos(2\pi - \alpha_2) = 1 + \cos\alpha_2 + \cos\alpha_2 = 1 + 2\cos\alpha_2 = 0,
\end{equation}
where we used $\cos(2\pi - \alpha_2) = \cos(-\alpha_2) = \cos\alpha_2$. Solving \textbf{Eq. \eqref{eq:second-case}}, we obtain $\cos\alpha_2 = -1/2$. Since $\alpha_2 \in (0, 2\pi)$ and $\alpha_2 \neq \pi$ (else $\alpha_3 = \pi$, violating distinctness), we have $\alpha_2 = 2\pi/3 = 120^{\circ}$ or $\alpha_2 = 4\pi/3 = 240^{\circ}$. If $\alpha_2 = 120^{\circ}$, we have $\alpha_3 = 2\pi - 120^{\circ} = 240^{\circ}$, giving the configuration $\{0^{\circ}, 120^{\circ}, 240^{\circ}\}$ with consecutive spacing of $120^{\circ}$.
We conclude that the unique angular configuration is $\{\alpha_1, \alpha_2, \alpha_3\} = \{0^{\circ}, 120^{\circ}, 240^{\circ}\}$, corresponding to the explicit unit vectors:
\begin{align}
\mathbf{u}_1 = \begin{pmatrix} 1 \\ 0 \end{pmatrix}, \quad
\mathbf{u}_2 = \begin{pmatrix} -1/2 \\ \sqrt{3}/2 \end{pmatrix}, \quad
\mathbf{u}_3 = \begin{pmatrix} -1/2 \\ -\sqrt{3}/2 \end{pmatrix}.
\end{align}

To verify completeness, we confirm the second-order isotropy condition. Computing the outer products:
\begin{align}
\mathbf{u}_1 \mathbf{u}_1^\top = \begin{pmatrix} 1 & 0 \\ 0 & 0 \end{pmatrix}, \quad
\mathbf{u}_2 \mathbf{u}_2^\top = \begin{pmatrix} 1/4 & -\sqrt{3}/4 \\ -\sqrt{3}/4 & 3/4 \end{pmatrix}, \quad
\mathbf{u}_3 \mathbf{u}_3^\top = \begin{pmatrix} 1/4 & \sqrt{3}/4 \\ \sqrt{3}/4 & 3/4 \end{pmatrix},
\end{align}
and summing them up, we obtain $\sum_{j=1}^3 \mathbf{u}_j \mathbf{u}_j^\top = \begin{pmatrix} 3/2 & 0 \\ 0 & 3/2 \end{pmatrix} = \frac{3}{2} I_2$, confirming second-order isotropy with $\lambda = 3/2$. The configuration is unique up to global rotation since fixing $\alpha_1 = 0$ removes the rotational degree of freedom, and the choice $\alpha_2 = 240^{\circ}$ simply corresponds to rotating the entire configuration by $120^{\circ}$. 

\end{proof}

\begin{corollary}[Hexagonal periodicity from cosine symmetry]
\label{cor:hexagonal}
The three-directional isotropic configuration induces hexagonal spatial periodicity with $60^\circ$ rotational symmetry in neural firing patterns.
\end{corollary}

\begin{proof}
From \textbf{Corollary \ref{cor:phase-encoding}}, the two-dimensional neural activity in subspace $j$ is given by $\mathbf{y}^{(j)} = (\cos(\mathbf{q}_j \cdot \mathbf{r}), \sin(\mathbf{q}_j \cdot \mathbf{r}))^\top$. 

\emph{Directional symmetry from cosine evenness:} The cosine function possesses even symmetry: $\cos(-\mathbf{q}_j \cdot \mathbf{r}) = \cos(\mathbf{q}_j \cdot \mathbf{r})$. This means that frequency vector $\mathbf{q}_j$ and its opposite $-\mathbf{q}_j$ produce identical spatial modulation in the cosine component. When considering the combined activity pattern formed by summing cosine terms from all subspaces, the three frequency directions $\{\mathbf{u}_1, \mathbf{u}_2, \mathbf{u}_3\}$ at angles $\{0^{\circ}, 120^{\circ}, 240^{\circ}\}$ generate stripe patterns whose interference is equivalent to that produced by six plane-wave components at angles $\{0^{\circ}, 60^{\circ}, 120^{\circ}, 180^{\circ}, 240^{\circ}, 300^{\circ}\}$ evenly spaced by $60^{\circ}$. This six-fold directional symmetry in the cosine-based firing pattern is the hallmark of hexagonal geometry~\cite{krupic2015neural_supp}.

\emph{Stripe pattern superposition:} The combined neural activity aggregating contributions from all three subspaces takes the form:
\begin{equation}
A(\mathbf{r}) = \sum_{j=1}^3 \cos(\mathbf{q}_j \cdot \mathbf{r}).
\end{equation}

Each individual term $\cos(\mathbf{q}_j \cdot \mathbf{r})$ represents a periodic stripe pattern in two-dimensional space: regions where $\mathbf{q}_j \cdot \mathbf{r} \approx 0 \pmod{2\pi}$ exhibit high activity (bright stripes, such as red area in \textbf{Fig. \ref{fig:S1}}), while regions where $\mathbf{q}_j \cdot \mathbf{r} \approx \pi \pmod{2\pi}$ exhibit low activity (dark stripes, such as green or blue area in \textbf{Fig. \ref{fig:S1}}). The stripes are parallel lines perpendicular to direction $\mathbf{u}_j$, repeating with spatial period $\lambda = 2\pi/q$ where $q = \|\mathbf{q}_j\|$. Thus, $A(\mathbf{r})$ represents the superposition of three stripe patterns oriented at $0^{\circ}$, $120^{\circ}$, and $240^{\circ}$, each with identical spatial frequency.

\emph{Interference creates triangular grid:} The superposition $A(\mathbf{r})$ achieves its global maximum value of $3$ precisely at positions where all three stripe patterns simultaneously reach their individual peaks, satisfying $\mathbf{q}_j \cdot \mathbf{r} = 2\pi n_j$ for integers $n_j \in \mathbb{Z}$, $j=1,2,3$. Geometrically, this occurs at the intersection points of bright stripes from all three orientations. Consider three families of parallel lines in the plane, each family rotated $120^{\circ}$ from the others. Elementary geometry establishes that such a configuration produces intersection points forming an equilateral triangular grid: each ``cell'' bounded by six line segments forms a regular hexagon, and the vertices of these hexagons constitute a triangular lattice. Specifically, any intersection point has six nearest-neighbor intersection points arranged symmetrically at $60^{\circ}$ intervals around a circle, all at equal distance.

\emph{Hexagonal periodicity from triangular symmetry:} The triangular lattice of peak firing locations possesses $60^{\circ}$ rotational symmetry: rotating the entire pattern by $60^{\circ}$ maps the lattice onto itself. Equivalently, we can describe this structure as a hexagonal lattice by noting that each lattice point is surrounded by six equidistant neighbors forming a regular hexagon. The spatial periodicity repeats with characteristic length scale (determined by $q$), and the firing pattern exhibits the hallmark signatures of biological grid cells: hexagonal arrangement of firing fields, $60^{\circ}$ rotational symmetry, and periodic tiling of the navigable space. This demonstrates that the three-direction isotropic configuration from \textbf{Theorem \ref{thm:minimal-isotropy}} necessarily produces hexagonal spatial representations through wave interference, bridging the abstract frequency-space symmetry to the concrete real-space geometry observed in neural recordings. 
\end{proof}

\textbf{Corollary \ref{cor:hexagonal}} posits hexagonal symmetry as a necessary consequence of geometric isotropy, offers a distinct perspective from mechanisms centered on nonlinear optimization. For instance, recent work demonstrates that hexagonal patterns are actively {selected} as the optimal solution when a nonnegativity constraint on firing rates is imposed \cite{sorscher2023unified_supp}. In that framework, nonnegativity acts as a nonlinear stabilization force that breaks the degeneracy of many possible linear solutions (i.e., plane waves on an annulus), thereby favoring the 120$^\circ$ triplet configuration. Our theory, in contrast, establishes this 120$^\circ$ configuration not merely as a selected optimum, but as the unique minimal $(K=3)$ solution that fundamentally satisfies the kinematic requirement for directional isotropy.
% \paragraph{Summary.} We established a complete theoretical chain from task specification to emergent structure using only undergraduate mathematics: mathematical induction for equivariance (Proposition \ref{prop:equivariance}), linear algebra for orthogonal matrices and outer products (Theorem \ref{thm:rotation-structure}), real analysis for Cauchy functional equations (Theorem \ref{thm:rotation-structure}), and trigonometry for rotation matrices (throughout). No group theory, Lie algebras, differential geometry, or topology was required while maintaining complete rigor. This demonstrates that hexagonal grid cell patterns in predictive neural networks follow from fundamental computational principles, not arbitrary design choices or biological constraints.

%% file: sections/SI/PIB.tex
\subsection{Derivation of the variational bound for the predictive information bottleneck}
\label{sec: Derivation of the Variational Bound for the Predictive Information Bottleneck}

This section establishes a rigorous mathematical framework for optimizing emergent communication mechanism through variational inference. This framework is based on the well-known Information Bottleneck (IB) theory \cite{alemi2017deep_supp,tishby2000information_supp}, an elegant but intractable information-theoretic solution, optimizing inference quality under bandwidth constraints between two agents for fulfilling a specific mission. The resulting solution is the IB principle that balances the compression ratio for a certain level of inference quality and bandwidth consumption for such compressed information sharing. The core challenge lies in translating the IB principle into a computationally feasible training procedure. To use IB theory to solve our current  problem at hand, its formulation naturally balances communication efficiency against predictive power by simultaneously minimizing the compression cost (rate) and maximizing the utility of transmitted messages for predicting the receiver's future observations (distortion). However, direct optimization is fundamentally intractable, as it requires computing mutual information between high-dimensional neural representations governed by unknown probability distributions. We resolve this obstacle through variational inference, deriving tight upper bounds for communication efficiency and lower bounds for prediction accuracy that transform the abstract IB objective into a fully differentiable loss function amenable to gradient-based optimization. This derivation proceeds as follows: we first establish notation and mathematical preliminaries, then derive tractable bounds for the rate term (\textbf{Method \ref{sec:rate-bound}}), followed by bounds for the distortion term (\textbf{Method \ref{sec:distortion-bound}}), and finally synthesize these components into the unified variational information bottleneck (VIB) objective (\textbf{Method \ref{sec:vib-synthesis}}).

\paragraph{Notation and problem formulation.}
Throughout this derivation, we adopt the following notation for clarity and consistency. The sender agent $i$ at time $t$ maintains internal state $S_i \equiv S_{i,t}$, encoding its current environmental context, sensory history, and predictive model. The receiver agent $j$ possesses analogous state $S_j \equiv S_{j,t}$, representing its own perspective. Communication occurs through a latent message variable $z \equiv m_{i,t}$ generated stochastically by the sender's encoder. The receiver's future observation $O_j' \equiv O_{j,t+1}$ serves as the prediction target whose uncertainty we seek to minimize through communication. In multi-agent coordination tasks, effective communication requires transmitting information that maximally reduces the receiver's uncertainty about its own future sensory experience, conditioned on what the receiver already knows from its current state.

The Information Bottleneck objective from the main text \textbf{Eq. \eqref{eq:ib_objective_revised}} formalizes this principle as a loss function that balances predictive accuracy against communication cost:
\begin{equation}
\label{eq:ib-objective}
\mathcal{L}_{\text{IB}} = \underbrace{-I(z; O_j' \mid S_j)}_{\text{Distortion}} + \beta \underbrace{I(S_i; z)}_{\text{Rate}},
\end{equation}
where $I(\cdot;\cdot)$ denotes mutual information, and $\beta > 0$ is a hyperparameter controlling the rate-distortion trade-off. We minimize $\mathcal{L}_{\text{IB}}$ with respect to the encoder distribution. The \textbf{Distortion} term $-I(z; O_j' \mid S_j)$ is the negative conditional mutual information—minimizing distortion maximizes the information the message provides about the receiver's future observation $O_j'$ beyond what the receiver can infer from its current state $S_j$ alone. The \textbf{Rate} term $I(S_i; z)$ quantifies the communication cost—the amount of information the message retains about the sender's state after communication, which is to minimize for compression. The parameter $\beta$ implements a soft bandwidth constraint: larger $\beta$ penalizes communication cost more severely, promoting compressed, abstract representations.

Direct optimization of \textbf{Eq.~\eqref{eq:ib-objective}} is intractable for three fundamental reasons. First, computing mutual information requires integrating over the joint distribution of high-dimensional continuous variables (neural network hidden states), which cannot be evaluated analytically or estimated efficiently from finite samples. Second, the conditional mutual information $I(z; O_j' \mid S_j)$ involves the true data distribution $p(O_j' \mid S_j, z)$ governing environment dynamics, which is unknown and potentially complex. Third, even if these distributions were known, the required high-dimensional integrals lack closed-form solutions and suffer from exponential computational complexity. We therefore seek variational surrogates—computable bounds that preserve the essential structure of the IB objective while enabling practical optimization through standard deep learning tools.

\subsubsection{Deriving a tractable upper bound for the communication rate}
\label{sec:rate-bound}

Considering the intractability of the IB objective in \textbf{Eq. \eqref{eq:ib-objective}}, we now construct a tractable surrogate for the communication rate term $I(S_i; z)$, which quantifies the bandwidth of intended transmitted messages. We proceed through three steps: formalizing the mutual information under parametric encoding, identifying the computational obstacle posed by aggregate posterior marginalization, and deriving a tight variational upper bound through KL divergence.

\begin{definition}[Parametric encoder and aggregate posterior]
\label{def:encoder}
The sender encodes its state $S_i$ into message $z$ through a stochastic encoder $q_{\varphi}(z|S_i)$ parameterized by neural network weights $\varphi$. The mutual information between state and message under this encoding is
\begin{equation}
I_q(S_i; z) = \int p(S_i) \int q_{\varphi}(z|S_i) \log \frac{q_{\varphi}(z|S_i)}{q(z)} \, dz \, dS_i,
\end{equation}
where $q(z) = \int p(S_i) q_{\varphi}(z|S_i) \, dS_i$ denotes the aggregate posterior—the marginal distribution of messages obtained by averaging encoder outputs over all possible sender states weighted by their occurrence probability $p(S_i)$ in the environment.
\end{definition}

The computational obstacle arises from the aggregate posterior $q(z)$, which requires integrating over the entire data distribution $p(S_i)$—an unknown, high-dimensional distribution of neural network hidden states. This integration is intractable both analytically (no closed form exists even for simple encoders) and numerically (Monte Carlo estimation suffers from exponential sample complexity in high dimensions and introduces high variance gradients unsuitable for optimization). We circumvent this obstacle by introducing a fixed prior distribution that serves as a tractable surrogate for the aggregate posterior.

\begin{lemma}[Variational rate bound via prior substitution]
\label{lem:rate-bound}
Let $p(z)$ be a fixed prior distribution over message space, typically chosen as an isotropic Gaussian $\mathcal{N}(0, I)$ for analytical tractability. The mutual information between sender state and message satisfies the upper bound
\begin{equation}
\label{eq:rate-upper-bound}
I_q(S_i; z) \le \mathbb{E}_{p(S_i)}[D_\mathrm{KL}(q_{\varphi}(z|S_i) \,\|\, p(z))],
\end{equation}
where $D_\text{KL}(\cdot \| \cdot)$ denotes the Kullback-Leibler divergence and the right-hand side involves only pointwise encoder-prior comparisons, not the intractable aggregate posterior.
\end{lemma}

\begin{proof}
We establish the bound through algebraic manipulation exploiting KL divergence non-negativity. Expanding the expected KL divergence between encoder and prior yields
\begin{equation}
\mathbb{E}_{p(S_i)}[D_\text{KL}(q_{\varphi}(z|S_i) \,\|\, p(z))] = \int p(S_i) \int q_{\varphi}(z|S_i) \log \frac{q_{\varphi}(z|S_i)}{p(z)} \, dz \, dS_i.
\end{equation}

The key step is logarithmic decomposition: introducing and subtracting the aggregate posterior $q(z)$ in the numerator and denominator gives
\begin{equation}
\log \frac{q_{\varphi}(z|S_i)}{p(z)} = \log \frac{q_{\varphi}(z|S_i)}{q(z)} + \log \frac{q(z)}{p(z)}.
\end{equation}

Substituting this decomposition into the expectation and separating integrals yields
\begin{align}
\mathbb{E}_{p(S_i)}[D_\text{KL}(q_{\varphi}(z|S_i) \,\|\, p(z))] &= \int p(S_i) \int q_{\varphi}(z|S_i) \left[ \log \frac{q_{\varphi}(z|S_i)}{q(z)} + \log \frac{q(z)}{p(z)} \right] dz \, dS_i \nonumber \\
&= \underbrace{\int p(S_i) \int q_{\varphi}(z|S_i) \log \frac{q_{\varphi}(z|S_i)}{q(z)} \, dz \, dS_i}_{ \, I_q(S_i; z)} + \underbrace{\int q(z) \log \frac{q(z)}{p(z)} \, dz}_{ \, D_\text{KL}(q(z) \,\|\, p(z)) \,\ge\, 0},
\end{align}
where the second term follows from $\int p(S_i) q_{\varphi}(z|S_i) \, dS_i = q(z)$ by definition of aggregate posterior. The fundamental property of KL divergence—non-negativity $D_\text{KL}(q(z) \,\|\, p(z)) \ge 0$ with equality if and only if $q(z) = p(z)$ almost everywhere—immediately yields \textbf{inequality \eqref{eq:rate-upper-bound}}.
\end{proof}

\begin{remark}[Computational tractability and optimization implications]
\textbf{Lemma \ref{lem:rate-bound}} transforms an intractable mutual information into a tractable expectation of KL divergence. For Gaussian encoder $q_{\varphi}(z|S_i) = \mathcal{N}(\boldsymbol{m}_{\varphi}(S_i), \Sigma_{\varphi}(S_i))$ and Gaussian prior $p(z) = \mathcal{N}(0, I)$, the KL divergence admits closed form $D_\mathrm{KL}(q_{\varphi}(z|S_i) \,\|\, p(z)) = \frac{1}{2}[\mathrm{tr}(\Sigma_{\varphi}) + \boldsymbol{m}_{\varphi}^\top\boldsymbol{m}_{\varphi} - \log\det(\Sigma_{\varphi}) - d]$ where $d$ is message dimensionality, enabling efficient gradient computation. Minimizing this upper bound during training explicitly penalizes the encoder for producing messages that deviate from the prior distribution, thereby enforcing compression: the encoder learns to allocate its limited information capacity only to features critical for the downstream prediction task, discarding sender-specific details irrelevant to the receiver.
\end{remark}

\subsubsection{Deriving a tractable upper bound for the distortion}
\label{sec:distortion-bound}

Having constructed a tractable upper bound for the communication rate, we now address the complementary challenge: bounding the distortion term $-I(z; O_j' \mid S_j)$ in the IB objective. The distortion quantifies the negative information the message provides about the receiver's future observations—minimizing distortion corresponds to maximizing predictive utility. We establish an upper bound amenable to optimization through a three-step derivation: entropy decomposition, identification of the optimization-relevant component, and variational approximation via a parametric decoder.

\begin{proposition}[Entropy decomposition of conditional mutual information]
\label{prop:entropy-decomposition}
The conditional mutual information between message $z$ and future observation $O_j'$ given receiver state $S_j$ admits the entropy decomposition
\begin{equation}
\label{eq:cmi-decomposition}
I(z; O_j' \mid S_j) = H(O_j' \mid S_j) - H(O_j' \mid z, S_j),
\end{equation}
where $H(O_j' \mid S_j)$ represents the receiver's baseline uncertainty about its future using only its current state, and $H(O_j' \mid z, S_j)$ represents the residual uncertainty after incorporating the message $z$.
\end{proposition}

\begin{proof}
This is a standard identity from information theory. By definition, conditional mutual information is $I(z; O_j' \mid S_j) = H(z \mid S_j) - H(z \mid O_j', S_j)$. Equivalently, using the symmetric form, $I(z; O_j' \mid S_j) = H(O_j' \mid S_j) - H(O_j' \mid z, S_j)$. To verify: expanding using conditional entropy definitions, $H(O_j' \mid S_j) - H(O_j' \mid z, S_j) = [H(O_j', S_j) - H(S_j)] - [H(O_j', z, S_j) - H(z, S_j)] = H(O_j', S_j) - H(S_j) - H(O_j', z, S_j) + H(z, S_j)$. Using the chain rule $H(O_j', z, S_j) = H(O_j', S_j) + H(z \mid O_j', S_j)$ and $H(z, S_j) = H(S_j) + H(z \mid S_j)$, we obtain $H(O_j', S_j) - H(S_j) - [H(O_j', S_j) + H(z \mid O_j', S_j)] + [H(S_j) + H(z \mid S_j)] = H(z \mid S_j) - H(z \mid O_j', S_j) = I(z; O_j' \mid S_j)$, confirming the identity.
\end{proof}

\begin{remark}[Optimization-relevant component]
The first term $H(O_j' \mid S_j)$ in \textbf{Eq.~\eqref{eq:cmi-decomposition}} represents the baseline unpredictability of the receiver's future independent of the communication  for the message $z$. Since this term is parameter-independent, it acts as an additive constant when optimizing $\mathcal{L}_{\text{IB}}$. Therefore, minimizing the distortion $-I(z; O_j' \mid S_j) = -H(O_j' \mid S_j) + H(O_j' \mid z, S_j)$ with respect to model parameters is equivalent to minimizing the conditional entropy $H(O_j' \mid z, S_j)$. Intuitively, minimizing distortion is equivalent to minimizing the receiver's residual uncertainty about its future after processing the message.
\end{remark}

The conditional entropy $H(O_j' \mid z, S_j) = -\mathbb{E}_{p(z, S_j)}[\int p(O_j' \mid z, S_j) \log p(O_j' \mid z, S_j) \, dO_j']$ remains intractable because it requires the knowledge of the true conditional distribution $p(O_j' \mid z, S_j)$ governing how the receiver's future observations depend on both the message $z$ and its current state. This distribution encodes complex environmental dynamics and is generally unknown. We resolve this through variational approximation, introducing a parametric decoder network $p_{\vartheta}(O_j' \mid z, S_j)$ that learns to predict future observations from the exchanged messages and receiver states.

\begin{lemma}[Variational entropy bound via decoder approximation]
\label{lem:entropy-bound}
Let $p_{\vartheta}(O_j' \mid z, S_j)$ be a parametric decoder approximating the true conditional distribution. The conditional entropy satisfies the upper bound
\begin{equation}
\label{eq:entropy-upper-bound}
H(O_j' \mid z, S_j) \le -\mathbb{E}_{p(z, S_j, O_j')} [\log p_{\vartheta}(O_j' \mid z, S_j)],
\end{equation}
where the right-hand side is the expected negative log-likelihood (reconstruction error) under the decoder, a quantity tractable for gradient-based optimization.
\end{lemma}

\begin{proof}
The proof exploits KL divergence non-negativity between the true and approximate conditional distributions. For any fixed $(z, S_j)$ pair, the KL divergence from true to approximate distribution satisfies
\begin{equation}
D_\text{KL}(p(O_j' \mid z, S_j) \,\|\, p_{\vartheta}(O_j' \mid z, S_j)) = \int p(O_j' \mid z, S_j) \log \frac{p(O_j' \mid z, S_j)}{p_{\vartheta}(O_j' \mid z, S_j)} \, dO_j' \ge 0.
\end{equation}

Expanding the logarithm and separating integrals yields
\begin{equation}
\label{ineq:vib-decoder}
\underbrace{\int p(O_j' \mid z, S_j) \log p(O_j' \mid z, S_j) \, dO_j'}_{-h(p(\cdot|z,S_j))} - \int p(O_j' \mid z, S_j) \log p_{\vartheta}(O_j' \mid z, S_j) \, dO_j' \ge 0.
\end{equation}
where $h(p) := -\int p(x)\log p(x)\,dx$ denotes the Shannon entropy functional.
Rearranging \textbf{inequality \eqref{ineq:vib-decoder}} gives $h(p(\cdot|z,S_j)) \le -\int p(O_j' \mid z, S_j) \log p_{\vartheta}(O_j' \mid z, S_j) \, dO_j'$. Multiplying by $-1$ and taking expectations over $p(z, S_j)$ on both sides, we have
\begin{equation}
H(O_j' \mid z, S_j) = \mathbb{E}_{p(z,S_j)}[h(p(\cdot|z,S_j))] \le -\mathbb{E}_{p(z,S_j,O_j')}[\log p_{\vartheta}(O_j' \mid z, S_j)],
\end{equation}
establishing \textbf{inequality \eqref{eq:entropy-upper-bound}}.
\end{proof}

\begin{corollary}[Variational upper bound on distortion]
\label{cor:distortion-upper-bound}
Combining \textbf{Proposition \ref{prop:entropy-decomposition}} and \textbf{Lemma \ref{lem:entropy-bound}}, the distortion term satisfies the upper bound
\begin{equation}
-I(z; O_j' \mid S_j) = -H(O_j' \mid S_j) + H(O_j' \mid z, S_j) \le -H(O_j' \mid S_j) + \mathbb{E}_{p(z, S_j, O_j')}[-\log p_{\vartheta}(O_j' \mid z, S_j)].
\end{equation}

Since $H(O_j' \mid S_j)$ is parameter-independent, minimizing this upper bound is equivalent to minimizing the expected negative log-likelihood (reconstruction error). Consequently, minimizing reconstruction error minimizes an upper bound on the distortion, ensuring messages become maximally predictive of receiver futures during training.
\end{corollary}

\begin{proof}
From \textbf{Proposition \ref{prop:entropy-decomposition}}, $-I(z; O_j' \mid S_j) = -H(O_j' \mid S_j) + H(O_j' \mid z, S_j)$. Applying \textbf{Lemma \ref{lem:entropy-bound}}'s upper bound on the conditional entropy gives the stated inequality. Since $H(O_j' \mid S_j)$ is constant with respect to model parameters, minimizing the bound reduces to minimizing $\mathbb{E}[-\log p_{\vartheta}(O_j' \mid z, S_j)]$—the negative log-likelihood loss ubiquitous in supervised learning.
\end{proof}

\subsubsection{Synthesis: Variational information bottleneck objective}
\label{sec:vib-synthesis}

Having derived tractable bounds for both rate and distortion, we now synthesize these components into a unified training objective. The preceding derivations resolved the fundamental computational obstacles in the IB principle through complementary variational approximations: an upper bound for the intractable rate term (\textbf{Lemma \ref{lem:rate-bound}}) and an upper bound for the intractable distortion term (\textbf{Corollary \ref{cor:distortion-upper-bound}}). We now demonstrate how these bounds combine to yield the variational information bottleneck (VIB) loss function—a fully tractable surrogate that preserves the essential feature of the original IB objective while enabling a tractable gradient-based optimization.

\begin{theorem}[Variational information bottleneck bound]
\label{thm:vib-bound}
The intractable Information Bottleneck objective $\mathcal{L}_{\text{IB}} = -I(z; O_j' \mid S_j) + \beta I(S_i; z)$ admits the tractable upper bound
\begin{equation}
\label{eq:vib-bound}
\mathcal{L}_{\mathrm{IB}} \le -H(O_j' \mid S_j) + \mathbb{E}_{p(S_i, S_j, O_j')} \left[ \mathbb{E}_{q_{\varphi}(z|S_i)}\big[-\log p_{\vartheta}(O_j' \mid z, S_j)\big] + \beta \, D_\mathrm{KL}\big(q_{\varphi}(z|S_i) \,\|\, p(z)\big) \right].
\end{equation}
Since $H(O_j' \mid S_j)$ is parameter-independent, minimizing this bound is equivalent to minimizing the Variational Information Bottleneck loss
\begin{equation}
\label{eq:vib-loss}
\mathcal{L}_{\mathrm{VIB}}(\varphi, \vartheta) := \mathbb{E}_{p(S_i, S_j, O_j')} \left[ \mathbb{E}_{q_{\varphi}(z|S_i)}\big[-\log p_{\vartheta}(O_j' \mid z, S_j)\big] + \beta \, D_\mathrm{KL}\big(q_{\varphi}(z|S_i) \,\|\, p(z)\big) \right].
\end{equation}
Minimizing $\mathcal{L}_{\mathrm{VIB}}$ with respect to encoder parameters $\varphi$ and decoder parameters $\vartheta$ minimizes an upper bound on the original IB objective, providing principled approximate optimization.
\end{theorem}

\begin{proof}
We establish the bound through systematic substitution of the variational bounds derived for each term. From \textbf{Proposition~\ref{prop:entropy-decomposition}}, the distortion term decomposes as
$$-I(z; O_j' \mid S_j) = -H(O_j' \mid S_j) + H(O_j' \mid z, S_j).$$
Thus, the original IB objective can be written as
$$\mathcal{L}_{\text{IB}} = -H(O_j' \mid S_j) + H(O_j' \mid z, S_j) + \beta I(S_i; z).$$

For the distortion component $H(O_j' \mid z, S_j)$, applying \textbf{Lemma \ref{lem:entropy-bound}}'s variational upper bound yields
$$H(O_j' \mid z, S_j) \le \mathbb{E}_{q(z), p(S_j,O_j')}[-\log p_{\vartheta}(O_j'|z,S_j)],$$
where $q(z) = \int q_{\varphi}(z|S_i)p(S_i)dS_i$ is the encoder-induced marginal.
Expressing $q(z)$ via its definition:
\begin{align}
\mathbb{E}_{q(z), p(S_j,O_j')}[-\log p_{\vartheta}(O_j'|z,S_j)]
&= \mathbb{E}_{p(S_j,O_j')}\int q(z)[-\log p_{\vartheta}(O_j'|z,S_j)]dz \nonumber\\
&= \mathbb{E}_{p(S_j,O_j')}\int p(S_i)q_{\varphi}(z|S_i)[-\log p_{\vartheta}(O_j'|z,S_j)]dz\,dS_i \nonumber\\
&= \mathbb{E}_{p(S_i,S_j,O_j')}\mathbb{E}_{q_{\varphi}(z|S_i)}[-\log p_{\vartheta}(O_j'|z,S_j)].
\end{align}

For the rate component, \textbf{Lemma \ref{lem:rate-bound}} provides the upper bound
$$I(S_i; z) \le \mathbb{E}_{p(S_i)}[D_\mathrm{KL}(q_{\varphi}(z|S_i) \,\|\, p(z))].$$

Combining these bounds:
\begin{align}
\mathcal{L}_{\text{IB}} &= -H(O_j' \mid S_j) + H(O_j' \mid z, S_j) + \beta I(S_i; z) \nonumber\\
&\le -H(O_j' \mid S_j) + \mathbb{E}_{p(S_i,S_j,O_j')}\mathbb{E}_{q_{\varphi}(z|S_i)}[-\log p_{\vartheta}(O_j'|z,S_j)] + \beta \mathbb{E}_{p(S_i)}[D_\mathrm{KL}(q_{\varphi}(z|S_i) \,\|\, p(z))] \nonumber\\
&= -H(O_j' \mid S_j) + \mathbb{E}_{p(S_i,S_j,O_j')}\left[\mathbb{E}_{q_{\varphi}(z|S_i)}[-\log p_{\vartheta}(O_j'|z,S_j)] + \beta \, D_\mathrm{KL}(q_{\varphi}(z|S_i) \,\|\, p(z))\right],
\end{align}
establishing \textbf{inequality \eqref{eq:vib-bound}}. The definition of $\mathcal{L}_{\mathrm{VIB}}$ in \textbf{Eq.~\eqref{eq:vib-loss}} follows by dropping the parameter-independent constant $-H(O_j' \mid S_j)$.
\end{proof}

\begin{corollary}[Practical VIB implementation]
\label{cor:vib-implementation}
For a single training sample $(S_i, S_j, O_j')$, the VIB loss is simplified
\begin{equation}
\label{eq:vib-loss-sample}
\mathcal{L}_{\mathrm{VIB}}(\varphi,\vartheta; S_i, S_j, O_j') = \mathbb{E}_{q_{\varphi}(z|S_i)}\big[-\log p_{\vartheta}(O_j' \mid z, S_j)\big] + \beta \, D_\mathrm{KL}\big(q_{\varphi}(z|S_i) \,\|\, p(z)\big).
\end{equation}

The first term (reconstruction loss) is estimated via the reparameterization trick using a single Monte Carlo sample: parameterizing $q_{\varphi}(z|S_i) = \mathcal{N}(\boldsymbol{m}_{\varphi}(S_i), \Sigma_{\varphi}(S_i))$ and sampling $z = \boldsymbol{m}_{\varphi}(S_i) + \Sigma_{\varphi}^{1/2}(S_i) \epsilon$, where $\epsilon \sim \mathcal{N}(0, I)$, yielding differentiable gradients. The second term (KL regularizer) evaluates in closed form for Gaussian encoder and prior. Stochastic gradient descent over mini-batches provides scalable optimization.
\end{corollary}

\begin{proof}
The reparameterization trick parameterizes the stochastic encoder using deterministic neural networks $\boldsymbol{m}_{\varphi}$ and $\Sigma_{\varphi}$, outputting mean and covariance, then externalizes randomness through $\epsilon \sim \mathcal{N}(0, I)$. This allows backpropagation through $z$ despite stochasticity. For Gaussian distributions, the KL term admits closed form $D_\text{KL}(q_{\varphi}(z|S_i) \,\|\, \mathcal{N}(0,I)) = \frac{1}{2}[\text{tr}(\Sigma_{\varphi}(S_i)) + \|\boldsymbol{m}_{\varphi}(S_i)\|^2 - \log\det(\Sigma_{\varphi}(S_i)) - d]$ where $d = \dim(z)$, derived from standard Gaussian KL formulas.
\end{proof}

\paragraph{Interpretation and trade-off control.}
The VIB objective achieves an elegant decomposition directly mirroring the rate-distortion framework from information theory: $\mathcal{L}_{\text{VIB}} = \text{Distortion} + \beta \cdot \text{Rate}$. The reconstruction term $\mathbb{E}_{q_{\varphi}(z|S_i)}[-\log p_{\vartheta}(O_j' \mid z, S_j)]$ corresponds to distortion: minimizing this term maximizes the decoder's accuracy in predicting the receiver's future observation from the message, directly implementing predictive utility maximization. The KL regularizer $D_\text{KL}(q_{\varphi}(z|S_i) \,\|\, p(z))$ corresponds to rate: minimizing this term constrains the encoder to produce compressed messages statistically indistinguishable from the prior, enforcing bandwidth efficiency and preventing information leakage about sender-specific details irrelevant to prediction. The hyperparameter $\beta > 0$ implements a Lagrange multiplier governing this trade-off: low $\beta$ values prioritize distortion minimization over rate reduction, allowing the encoder to transmit more detailed, high-bandwidth messages that maximize the receiver's ability to predict future observations; high $\beta$ values prioritize rate reduction over distortion minimization, forcing the emergence of abstract, highly-compressed symbolic mechanisms that minimize bandwidth usage while potentially sacrificing some prediction accuracy. This principled parameterization enables systematic exploration of the rate-distortion frontier, revealing how communication constraints shape emergent language structure for efficient communications—a central theme in our experimental investigations.

\paragraph{Theoretical guarantees and practical benefits.}
The VIB framework provides rigorous theoretical guarantees inherited from its variational foundation. \textbf{Theorem \ref{thm:vib-bound}} ensures that optimizing the tractable surrogate $\mathcal{L}_{\text{VIB}}$ drives down an upper bound on the true IB objective, guaranteeing improvement (under perfect optimization) toward the optimal rate-distortion trade-off. The tightness of this bound improves as the variational approximations become more accurate: when the encoder's aggregate posterior $q(z)$ approaches the prior $p(z)$ and the decoder $p_{\vartheta}(O_j' \mid z, S_j)$ approaches the true conditional distribution, the inequalities in \textbf{Lemmas \ref{lem:rate-bound}} and \textbf{\ref{lem:entropy-bound}} approach equality, yielding $\mathcal{L}_{\text{IB}} \to -H(O_j' \mid S_j) + \mathcal{L}_{\text{VIB}}$. Since the constant term $-H(O_j' \mid S_j)$ is parameter-independent, optimizing $\mathcal{L}_{\text{VIB}}$ is equivalent to optimizing $\mathcal{L}_{\text{IB}}$ regardless of bound tightness.. Practically, this framework enables end-to-end learning of communication mechanisms through standard deep learning infrastructure—gradient-based optimization, mini-batch training, and GPU acceleration—without requiring explicit symbolic grounding, hand-crafted message spaces, or task-specific communication engineering. The learned communication mechanisms emerge purely from the objective of collaborative prediction under bandwidth constraints, embodying the social predictive coding principle: agents learn to exchange information that is maximally novel and decision-relevant from the receiver's perspective, the essence of effective and efficient communication.

\newpage

%% file: sections/SI/sup_figure.tex
\section{Supplementary Results and Analyses}
\begin{figure}[h!]
\centering
\includegraphics[width=0.86\linewidth]{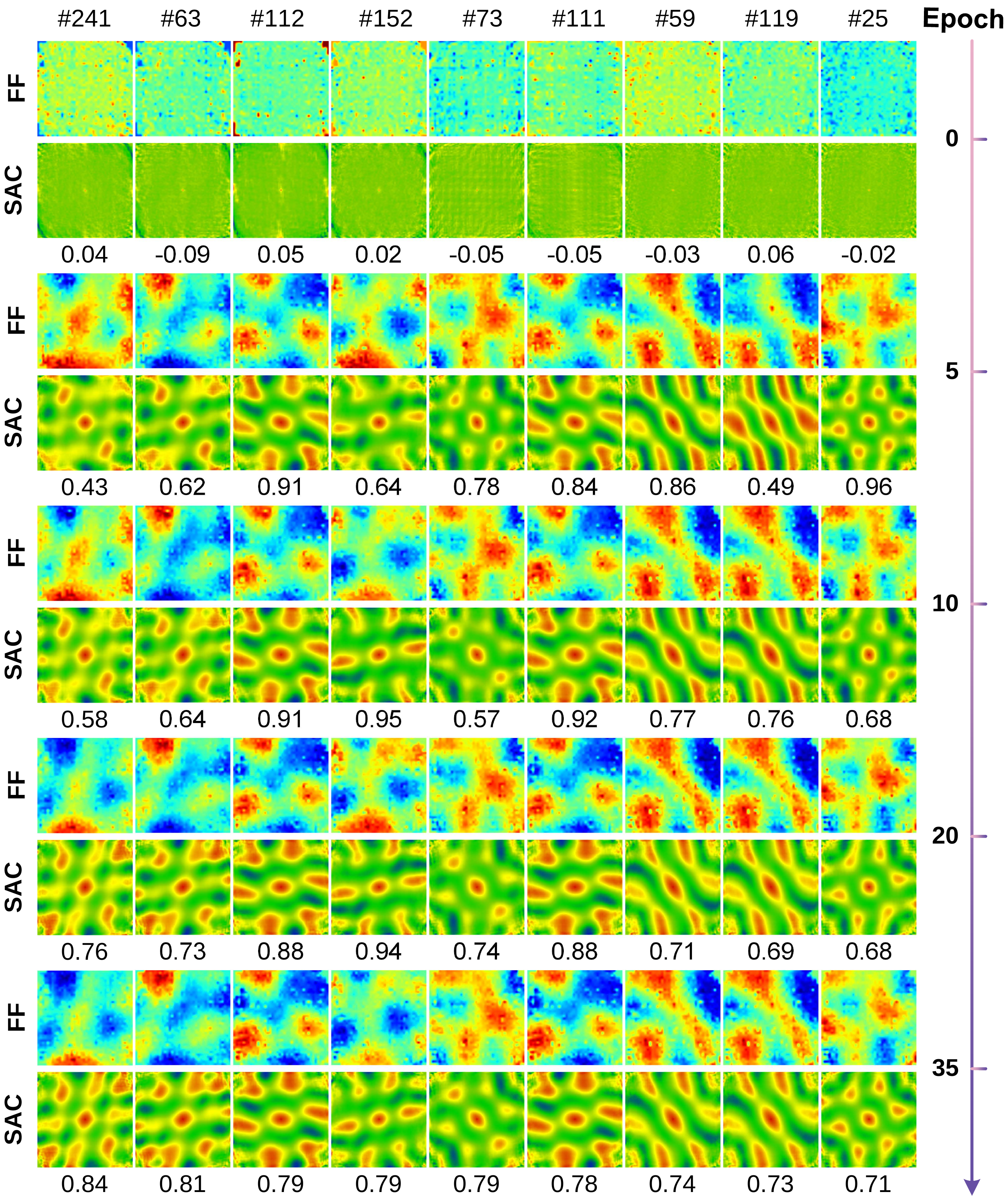}
\caption{\textbf{Emergence of grid-like neural representations through self-supervised training.} 
Spatial activity patterns of nine representative artificial neurons from the path integrator's 
bottleneck layer across training epochs (0, 5, 10, 20, 35). Top row of each pair shows firing 
fields (FF); bottom row shows spatial autocorrelograms (SACs). Gridness scores ($G_{60}$, 
displayed below each SAC) increase from near-zero or negative values at initialization to 
highly positive values after training (up to 0.96), indicating emergence of hexagonal symmetry 
characteristic of biological grid cells. Unit numbers (\#241, \#63, etc.) indicate bottleneck 
layer indices. This progressive transformation demonstrates how predictive learning drives 
self-organization of structured spatial representations.}
\label{fig:S1}
\end{figure}

\begin{figure}[h]
\centering
\includegraphics[width=\linewidth]{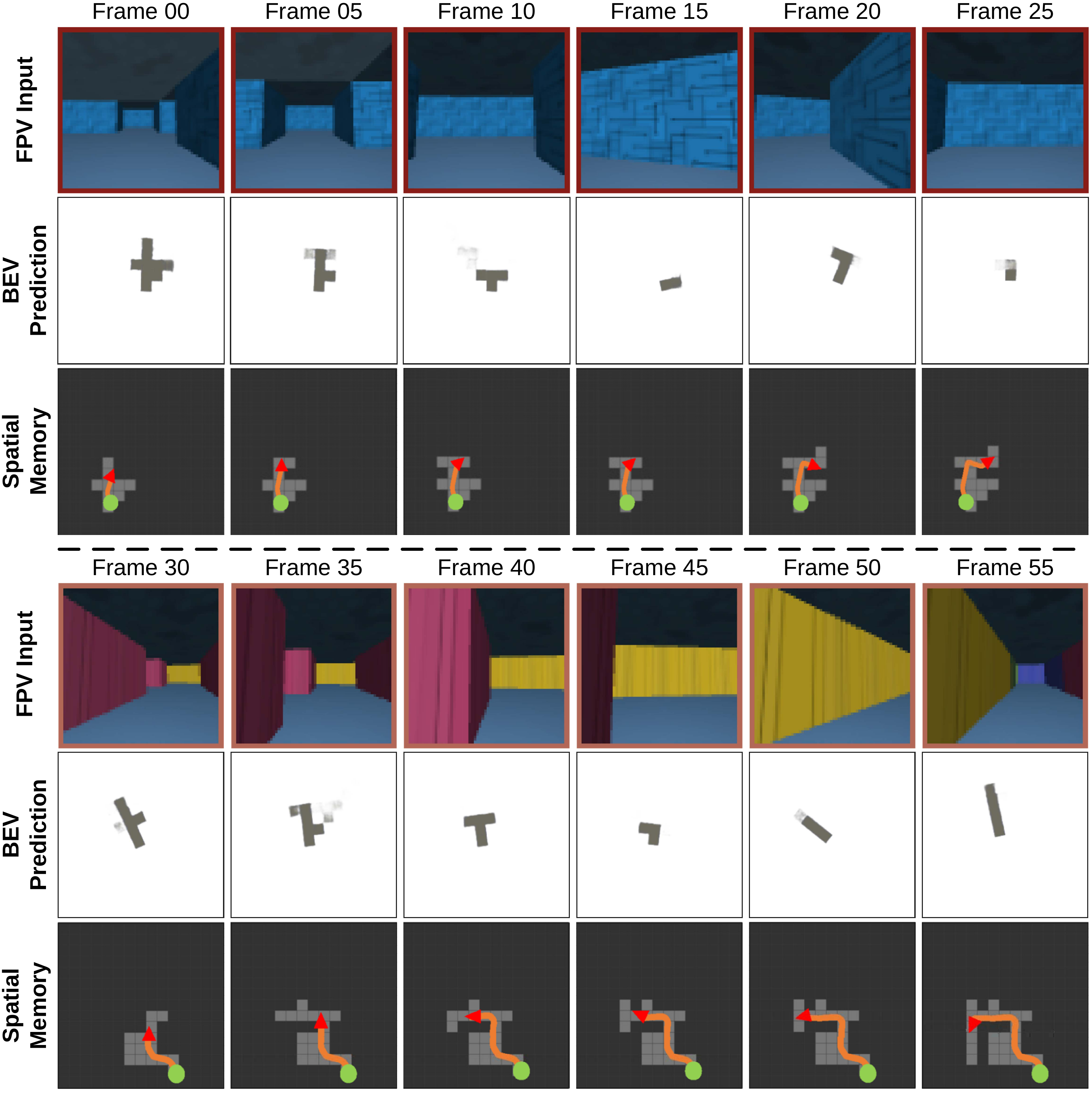}
\caption{\textbf{Visualization of the spatial memory construction pipeline.} The agent's first-person visual (FPV) input (top row) is converted into an instantaneous bird's-eye view (BEV) prediction of the local surroundings (middle row). These predictions are sequentially integrated into a persistent and growing allocentric spatial memory map, which also tracks the agent's trajectory and current pose (bottom row).}
\label{fig:figS-sm}
\end{figure}

\begin{figure}[h]
\centering
\includegraphics[width=\linewidth]{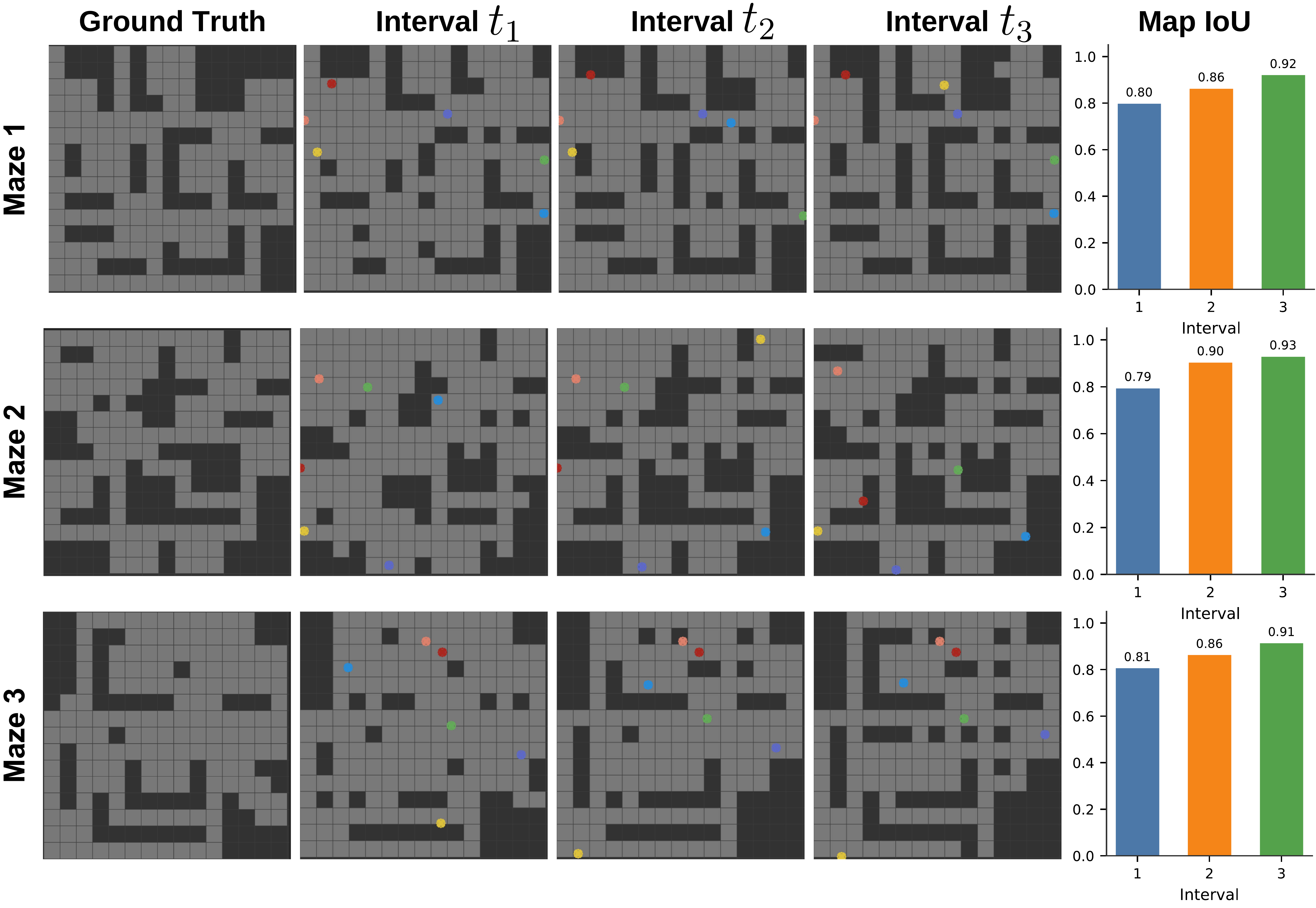}
\caption{\textbf{Visualization of spatial memory map construction over increasing time intervals.} This figure illustrates the progressive construction of the shared bird's-eye-view (BEV) map by the agent team across three different randomly generated maze environments (Maze 1-3). The leftmost column shows the ground truth layout for each maze. The following three columns visualize the state of the collaboratively built map at sequential time intervals ($t_1=1, t_2=2, t_3=3$). The bar charts on the right quantify the accuracy of the reconstructed map via the intersection over union (IoU) metric against the ground truth. The results demonstrate that as the exploration interval increases, the map's completeness and accuracy improve, reflected in the rising IoU scores. Notably, the agents are capable of achieving a high degree of accuracy in reconstructing the maze layout solely through exploration, validating the effectiveness of the spatial memory construction framework.}
\label{fig:S2}
\end{figure}

\begin{figure}[h]
\centering
\includegraphics[width=1\linewidth]{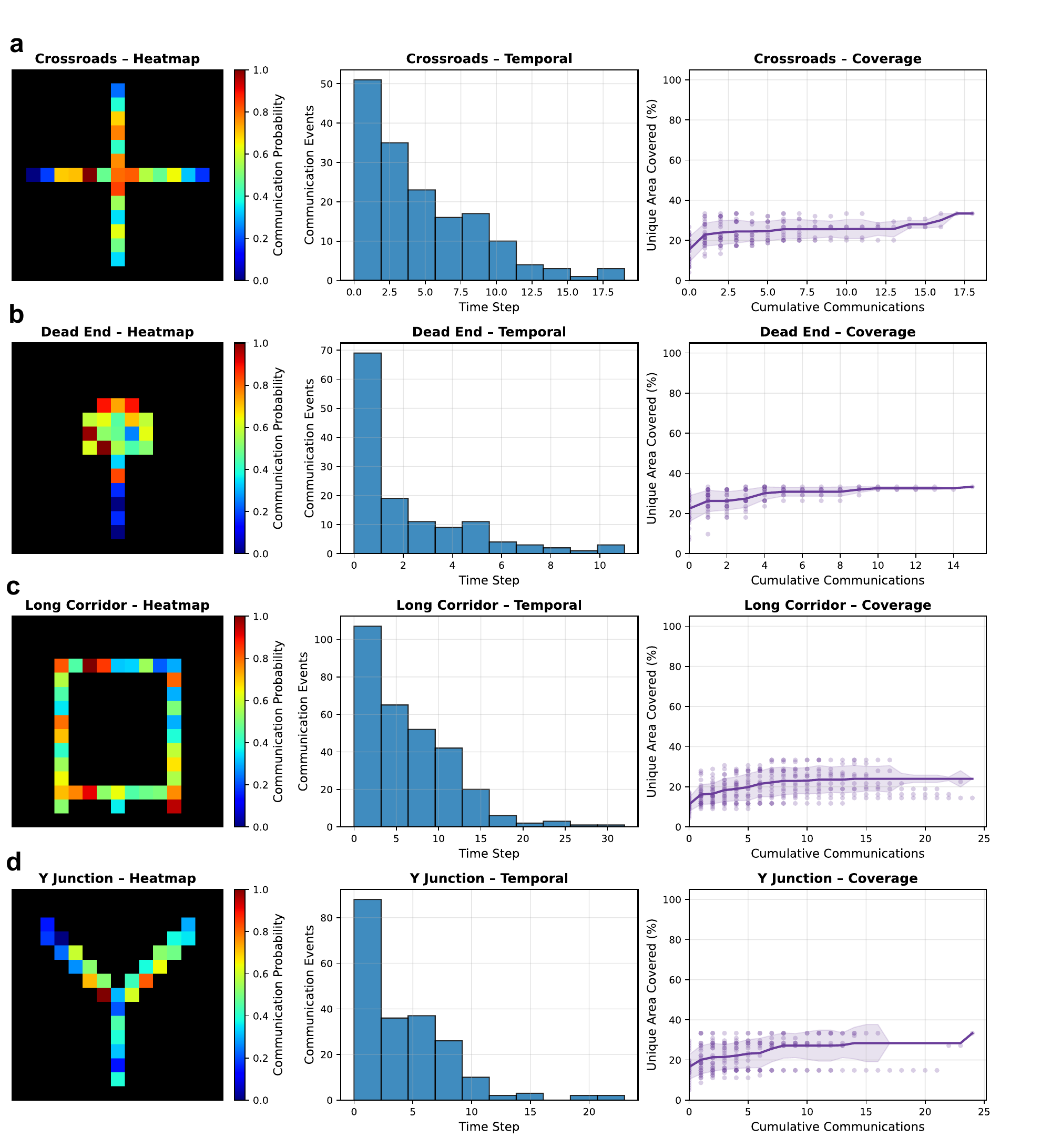}
\caption{\textbf{Strategic communication emerges across diverse maze topologies.} This figure analyzes agents' emergent communication strategies across four canonical maze structures: (a) Crossroads, (b) Dead End, (c) Long Corridor, and (d) Y Junction. For each topology, we visualize message spatial distribution (Heatmap), timing (Temporal), and the relationship between communication and exploration (Coverage). The heatmaps reveal a consistent pattern of ``strategic triggering'': agents communicate at points of high predictive uncertainty for their partners. For instance, in the Crossroads (a) and Y Junction (d), communication peaks at the central intersection, an ambiguous location where information helps coordinate exploration. Conversely, in the Dead End (b), agents communicate from deep within the trap, acting as an efficient ``prediction error'' signal to inform teammates the path is not fruitful. These consistent patterns demonstrate that the social predictive objective drives agents to learn an implicit model of their partners' beliefs, sharing information when and where it best resolves uncertainty.}
\label{fig:S3}
\end{figure}

\begin{figure}[h]
\centering
\includegraphics[width=\linewidth]{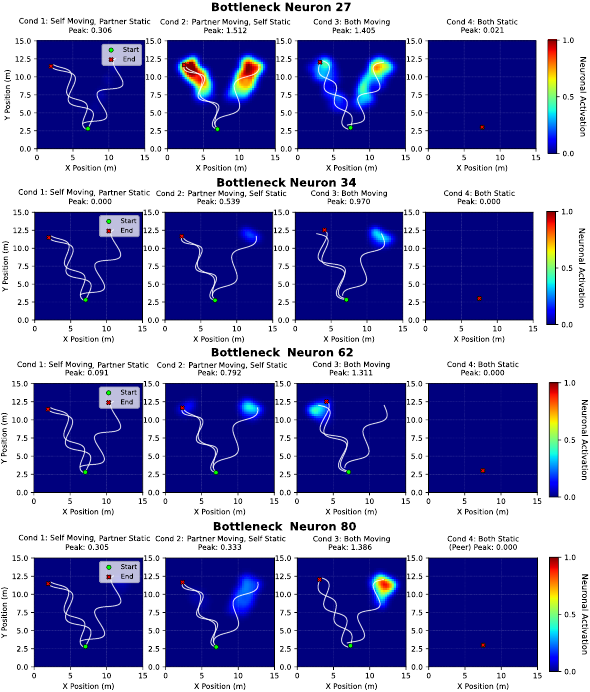}
\caption{\textbf{Gallery of emergent social place cells (SPCs).} 
Representative artificial neurons from the social processing network's bottleneck 
layer that selectively encode the partner's location. Each row displays a single 
artificial neuron's spatial firing field across four experimental conditions: 
(1) Self moving, partner static; (2) Partner moving, self static; (3) Both agents 
moving; (4) Both static. Neuron numbers (e.g., \#27, \#34, \#62, \#80) indicate 
bottleneck layer indices. White lines show agent trajectories; red dots mark 
initial positions; heatmaps represent normalized firing rates (0-1 scale). Peak 
values indicate maximum activation for each condition. These SPCs exhibit strong, 
localized firing fields when the partner moves through a specific area (Condition 2), 
but remain largely silent in response to the agent's own movement (Condition 1). 
This selective tuning to another's position is the defining feature of social 
place cells.}
\label{fig:S4}
\end{figure}

\begin{figure}[h]
\centering
\includegraphics[width=\linewidth]{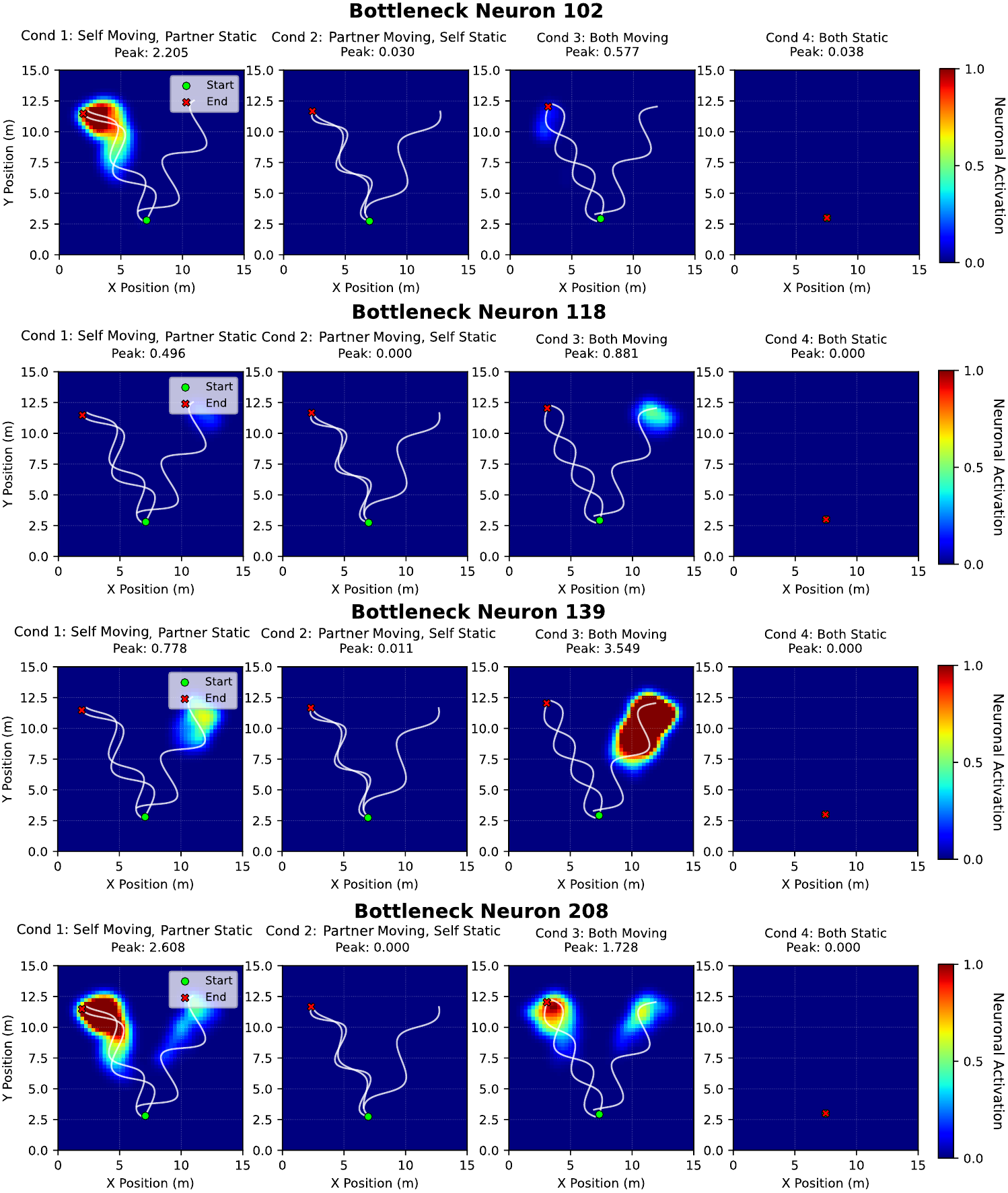}
\caption{\textbf{Gallery of emergent place cells (PCs).} 
Representative artificial neurons from the bottleneck layer that function as 
classical place cells, selectively encoding the agent's own location. Each row 
shows a single artificial neuron's activity across the four experimental conditions 
(layout as in Fig. S4). Neuron numbers (e.g., \#102, \#118, \#139, \#208) indicate 
bottleneck layer indices. These units display strong, stable firing fields when 
the agent itself traverses a specific location (Condition 1), but show negligible 
activation in response to the partner's movement (Condition 2). This demonstrates 
clear encoding of self-position, providing a stable allocentric representation 
for the agent.}
\label{fig:S5}
\end{figure}

\begin{figure}[h]
\centering
\includegraphics[width=\linewidth]{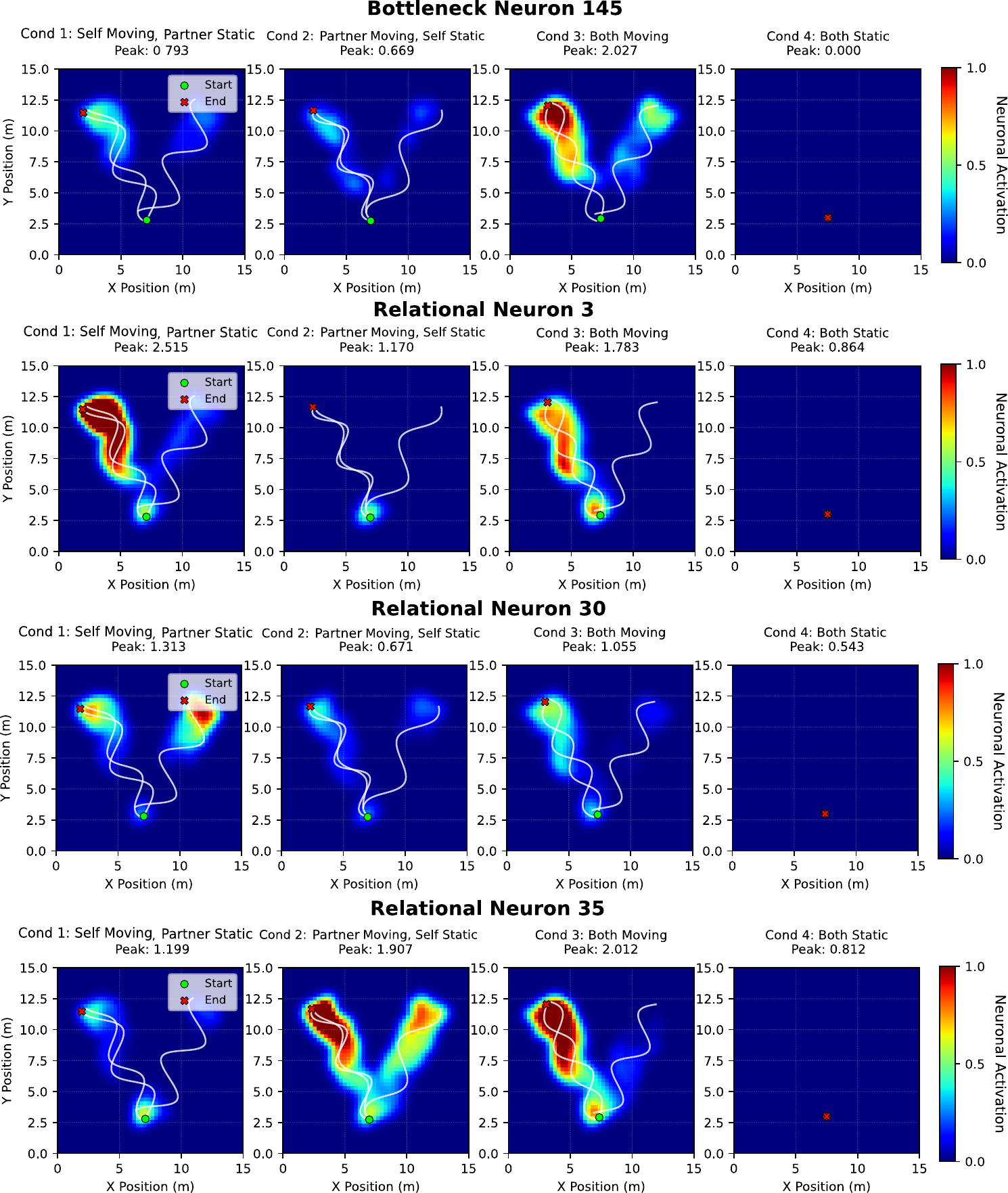}
\caption{\textbf{Gallery of special SPCs with mixed selectivity.} 
Representative artificial neurons with conjunctive encoding of both self and 
partner information. Top row shows a bottleneck neuron (\#145); remaining rows 
show relational head neurons (\#3, \#30, \#35), illustrating that mixed selectivity 
emerges across multiple network layers. Each row displays activity across the 
four experimental conditions (layout as in Fig. S4). Unlike pure PCs or SPCs, 
these artificial neurons show significant activation in response to both the 
agent's own movement (Condition 1) and the partner's movement (Condition 2). 
Firing is often maximal when both agents are moving (Condition 3), indicating 
that these cells encode a higher-order relational variable between the agents 
rather than a simple location. The distinction between ``Bottleneck'' and ``Relational'' 
neurons reflects their position in different processing stages of the network 
architecture.}
\label{fig:S6}
\end{figure}

\begin{figure}[h]
\centering
\includegraphics[width=\linewidth]{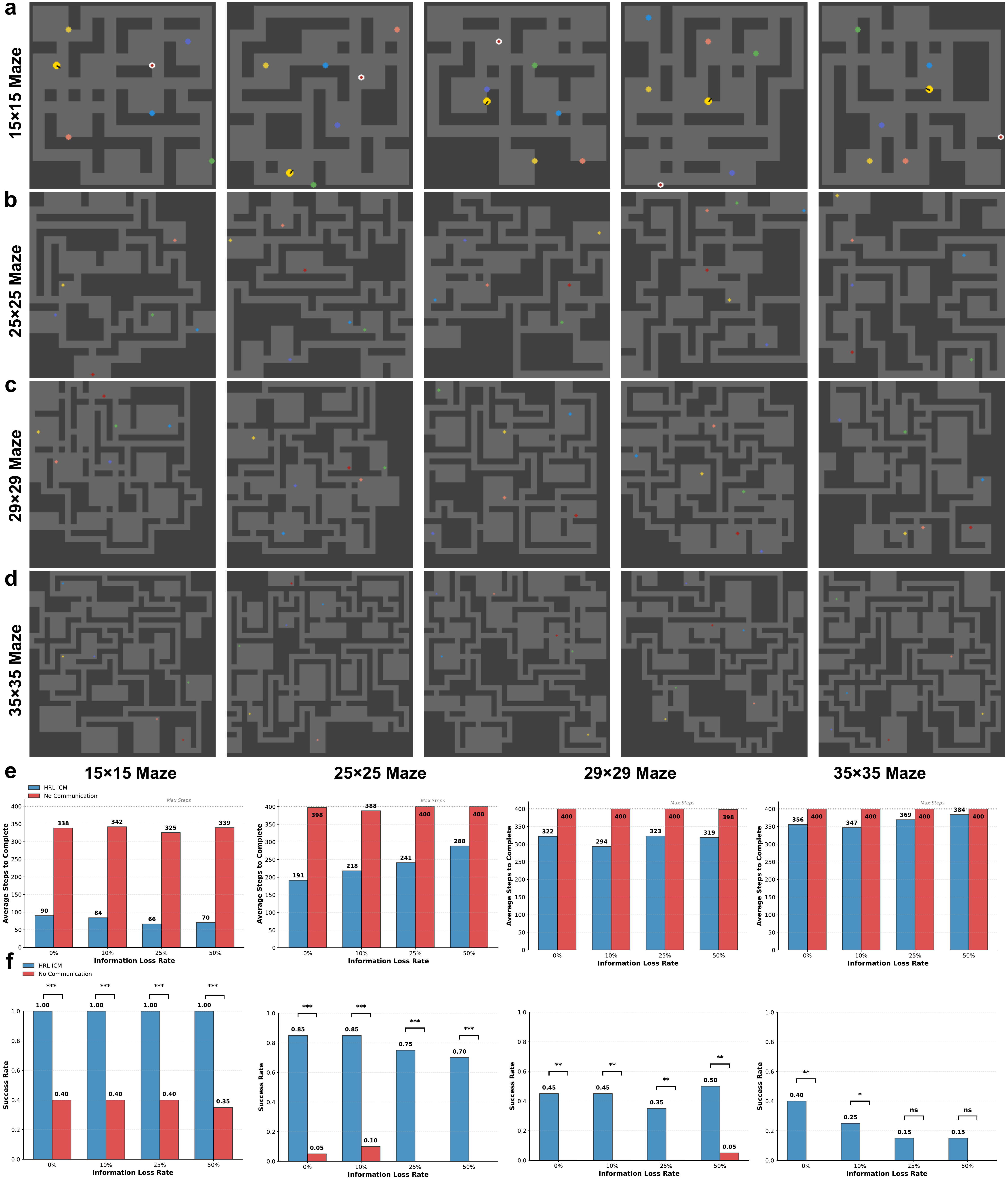}
    \caption{\textbf{Robustness of the learned communication mechanism to noise and scaling environmental complexity.} 
    This figure evaluates the HRL-ICM framework's robustness to communication noise across a suite of procedurally generated mazes of increasing scale and complexity, shown in panels \textbf{(a-d)}. 
    Performance is compared against a ``No Communication'' baseline by measuring the average steps to completion \textbf{(e)} and the task success rate \textbf{(f)} under varying information loss rates (0-50\%).
    The results demonstrate that HRL-ICM (blue) consistently solves tasks with higher efficiency and a significantly greater success rate. 
    Crucially, its performance degrades gracefully as noise increases, whereas the baseline consistently fails. This robust advantage persists even in larger, more complex environments, highlighting the scalability of the learned communication mechanism. 
    Statistical significance is denoted as: \textbf{*} $p < 0.05$, \textbf{**} $p < 0.01$, \textbf{***} $p < 0.001$.
    }
\label{fig:S-noise}
\end{figure}

%% file: sections/SI/sup_movie.tex
\section{Supplementary Videos}

Video 1: \textbf{Collaborative Mapping and Exploration}. Here, we demonstrate the complete predictive coding framework for shared spatial memory in multi-agent maze navigation tasks. We compare baseline full broadcast communication against our HRL-ICM approach under limited communication budgets (20 rounds). The results show 44.7\% faster task completion through adaptive communication strategies that avoid redundant information exchange, while baseline methods degrade after exhausting communication quotas early in exploration.

\vspace{8pt}

\noindent Video 2: \textbf{Visualization of Emergent Social Place Cells}. Here, we visualize the spontaneous emergence of social place cells that encode partner agent locations through predictive coding. We examine four scenarios with different movement patterns: self-movement with stationary peer, stationary self with moving peer, and synchronized movement. The analysis reveals distinct functional neuron types including pure place cells, pure social place cells, and special social place cells, combining both spatial references.

\vspace{8pt}

\noindent Video 3: \textbf{Causal Intervention on Emergent Communication}. Here, we demonstrate the functional significance of emergent communication through causal intervention experiments in collaborative search tasks. By injecting fake target information into the communication channel, we test whether communicated signals actually influence partner behavior. The intervention causes agents to navigate toward false targets, providing causal evidence that emergent communication mechanisms directly drive decision-making rather than being epiphenomenal.

\vspace{8pt}

\noindent Video 4: \textbf{Full Framework - Collaborative Mapping and Exploration}. Here, we present the complete demonstration of our predictive coding framework pipeline from egocentric vision to bird's-eye-view map reconstruction. We show three increasingly complex tasks involving multi-agent coordination: single-agent BEV reconstruction, two-agent collaboration in multi-room environments, and three-agent coordination in large mazes. The framework achieves emergent division of labor through bandwidth-efficient communication (4-128 bits/step).

\vspace{8pt}

\noindent Video 5: \textbf{Spontaneous Formation of Grid-Cell-like Representations}. Here, we visualize the developmental trajectory of hexagonal grid cell representations emerging through self-supervised motion prediction without explicit supervision. The video shows evolution from random neural activity (Epoch 0) to organized hexagonal tessellations with 60-degree rotational symmetry (Epochs 20-150). Grid cell emergence correlates with dramatically improved trajectory prediction accuracy, demonstrating that predictive uncertainty minimization principles sufficiently explain mammalian entorhinal cortex spatial coding.